%% file: main_arxiv_V4.tex
\documentclass{article} %
\usepackage{iclr2024_conference,times}

\usepackage{hyperref}
\usepackage{url}

\input{tex/macro}

\input{tex/notation}
\let\cite\citep

\title{\vspace{-1cm}Differentially Private Synthetic Data \\via Foundation Model APIs 1: Images\vspace{-0.4cm}}

\author{%
  Zinan Lin \\
  Microsoft Research\\
  \texttt{\small zinanlin@microsoft.com} \\
  \And
  Sivakanth Gopi \\
  Microsoft Research\\
  \texttt{\small sigopi@microsoft.com} \\
  \And
  Janardhan Kulkarni \\
  Microsoft Research\\
  \texttt{\small jakul@microsoft.com} \\
  \And
  Harsha Nori \\
  Microsoft Research\\
  \texttt{\small hanori@microsoft.com} \\
  \And
  Sergey Yekhanin \\
  Microsoft Research\\
  \texttt{\small yekhanin@microsoft.com}
}

\iclrfinalcopy %
\begin{document}

\maketitle

\vspace{-0.8cm}
\input{tex/abstract}

\blfootnote{$^\dagger$ Main updates in arXiv V4: Fix typos in \cref{alg:main,alg:main_full,alg:main_privacy}. }

\vspace{-0.55cm}
\input{tex/intro}

\input{tex/background}

\input{tex/problem}

\input{tex/algorithm}

\input{tex/experiment}

\input{tex/discussion}

\input{tex/ethics}

\input{tex/ack}

\bibliography{reference}
\bibliographystyle{iclr2024_conference}

\clearpage
\appendix
\input{tex/app_related_work}

\FloatBarrier
\input{tex/app_wasserstein}
\FloatBarrier
\input{tex/app_evolutionary}
\FloatBarrier
\input{tex/app_alg}

\FloatBarrier
\input{tex/theory} %
\FloatBarrier
\input{tex/app_theory}
\FloatBarrier
\input{tex/app_intrinsic_dimension}
\FloatBarrier
\input{tex/app_dpclustering}
\FloatBarrier
\input{tex/app_conditional_unconditional}
\FloatBarrier
\input{tex/app_cifar10}
\FloatBarrier
\input{tex/app_camelyon}
\FloatBarrier
\input{tex/app_stable_diffusion}
\FloatBarrier
\input{tex/app_cinic}
\FloatBarrier
\input{tex/app_ablation}

\FloatBarrier
\input{tex/app_more_samples}
\FloatBarrier
\input{tex/app_computation_resource}

\end{document}

%% file: tex/macro.tex
\usepackage{placeins}

\usepackage{booktabs}

\usepackage{lipsum}  
\usepackage{wrapfig}
\usepackage{nicefrac}
\usepackage{hyperref}
\hypersetup{
    colorlinks=true,
    linkcolor=blue,
    filecolor=magenta,      
    urlcolor=cyan
}
\usepackage{mathtools}
\usepackage{amsmath,amsfonts}
\usepackage{amssymb}
\usepackage[ruled,vlined]{algorithm2e}
\usepackage{algorithmic}
\usepackage{makecell}
\usepackage{diagbox}
\usepackage{multirow}
\usepackage{textgreek}

\usepackage{listings}

\usepackage{caption}
\usepackage{subcaption}

\usepackage{amsthm}
\newtheorem{theorem}{Theorem}

\newtheorem{claim}{Claim}

\usepackage[nameinlink,capitalize]{cleveref}

\crefname{section}{\S}{\S}
\Crefname{section}{\S}{\S}
\crefname{appendix}{App.}{Apps.}
\Crefname{appendix}{App.}{Apps.}
\crefname{theorem}{Thm.}{Thms.}
\Crefname{theorem}{Thm.}{Thms.}
\crefname{proposition}{Prop.}{Props.}
\Crefname{proposition}{Prop.}{Props.}

\crefname{algorithm}{Alg.}{Algs.}
\Crefname{algorithm}{Alg.}{Algs.}
\crefname{assumption}{Asm.}{Asms.}
\Crefname{assumption}{Asm.}{Asms.}
\crefname{mechanism}{Mech.}{Mechs.}
\Crefname{mechanism}{Mech.}{Mechs.}

\crefformat{footnote}{#2\footnotemark[#1]#3}
\crefmultiformat{footnote}{#2\footnotemark[#1]#3}%
{\textsuperscript{,}#2\footnotemark[#1]#3}{\textsuperscript{,}#2\footnotemark[#1]#3}{\textsuperscript{,}#2\footnotemark[#1]#3}

\makeatletter
\newcommand\footnoteref[1]{\protected@xdef\@thefnmark{\ref{#1}}\@footnotemark}
\makeatother

\newcommand{\highlightblock}[1]{\begin{center}\vspace{-0.1cm}\emph{#1}\vspace{-0.1cm}\end{center}}

\newcommand{\myparatightestn}[1]{ \noindent\textbf{{#1}}}

\newcommand{\myparaemphtightestn}[1]{\noindent\emph{{#1}}}

\newcounter{packednmbr}
\newenvironment{packedenumerate}{\begin{list}{\thepackednmbr.}{\usecounter{packednmbr}\setlength{\itemsep}{0.5pt}\addtolength{\labelwidth}{-4pt}\setlength{\leftmargin}{\labelwidth}\setlength{\listparindent}{\parindent}\setlength{\parsep}{1pt}\setlength{\topsep}{0pt}}}{\end{list}}
\newenvironment{packeditemize}{\begin{list}{$\bullet$}{\setlength{\itemsep}{0.5pt}\addtolength{\labelwidth}{-4pt}\setlength{\leftmargin}{\labelwidth}\addtolength{\leftmargin}{-2pt}\setlength{\listparindent}{\parindent}\setlength{\parsep}{1pt}\setlength{\topsep}{0pt}}}{\end{list}}

\NewDocumentCommand{\codeword}{v}{%
\texttt{\textcolor{blue}{#1}}%
}

\newcommand{\revision}[1]{#1}

\newcommand{\size}[2]{{\fontsize{#1}{0}\selectfont#2}}

 \newcommand\blfootnote[1]{%
  \begingroup
  \renewcommand\thefootnote{}\footnote{#1}%
  \addtocounter{footnote}{-1}%
  \endgroup
}

%% file: tex/notation.tex
\newcommand{\DPWA}{DPWA}
\newcommand{\DPWAfull}{DP Wasserstein Approximation}
\newcommand{\probname}{DP Synthetic Data via APIs}
\newcommand{\probnamewithunderline}{\underline{D}\underline{P} \underline{S}ynthetic \underline{D}ata via \underline{A}PIs}
\newcommand{\probnameshort}{DPSDA}
\newcommand{\algname}{Private Evolution}
\newcommand{\algnameshort}{PE}
\newcommand{\dpvotingname}{DP Nearest Neighbors Histogram}

\newcommand{\codeurl}{\url{https://github.com/microsoft/DPSDA}}

\newcommand{\cifar}{CIFAR10}
\newcommand{\imagenet}{ImageNet}
\newcommand{\camelyon}{Camelyon17}
\newcommand{\catcookie}{Cat Cookie}
\newcommand{\catruby}{Cat Doudou}

\newcommand{\numiterations}{T}
\newcommand{\numgensamples}{N_\syn}
\newcommand{\numprisamples}{N_\priv}
\newcommand{\privatesampleset}{S_\priv}
\newcommand{\privatesampleclassset}{C}
\newcommand{\generatedsampleset}{S_\syn}
\newcommand{\noisemultiplier}{\sigma}
\newcommand{\threshold}{H}
\newcommand{\alg}{\calM}

\newcommand{\randomsampleapiname}{\size{9}{\textsf{RANDOM\_API}}}
\newcommand{\randomsampleapi}[1]{\randomsampleapiname{}\bra{#1}}

\newcommand{\samplevariationapiname}{\size{9}{\textsf{VARIATION\_API}}}
\newcommand{\samplevariationapi}[1]{\samplevariationapiname{}\bra{#1}}

\newcommand{\embeddingnetworkname}{\Phi}
\newcommand{\embeddingnetwork}[1]{\embeddingnetworkname\bra{#1}}

\newcommand{\packingdegree}{k}

\newcommand{\packing}{lookahead}
\newcommand{\Packing}{Lookahead}

\newcommand{\variationdegree}{v}

\newcommand{\dpvotingfunctionname}{\size{9}{\textsf{DP\_NN\_HISTOGRAM}}}
\newcommand{\dpvotingfunction}[1]{\dpvotingfunctionname{}\bra{#1}}

\newcommand{\distancefunctionname}{d}
\newcommand{\distancefunction}[1]{\distancefunctionname{}\bra{#1}}

\newcommand{\normaldistribution}[2]{\calN\bra{#1,#2}}

\newcommand{\priv}{\mathrm{priv}}
\newcommand{\syn}{\mathrm{syn}}
\newcommand{\supp}{\mathrm{supp}}

\newcommand{\probnotation}{\mathbb{P}}
\newcommand{\probof}[1]{\probnotation\bra{#1}}

\newcommand{\real}{\mathbb{R}}

\newcommand{\bra}[1]{\left( #1 \right)}
\newcommand{\brb}[1]{\left[ #1 \right]}

\newcommand{\brc}[1]{\left\{ #1 \right\}}
\newcommand{\brd}[1]{\left| #1 \right|}
\newcommand{\brn}[1]{\left\lVert #1 \right\rVert}
\newcommand{\norm}[1]{\brn{#1}}

\newcommand{\brnt}[1]{\left\lVert #1 \right\rVert_2}

\newcommand{\calD}{\mathcal{D}}

\newcommand{\calM}{\mathcal{M}}
\newcommand{\calN}{\mathcal{N}}

\newcommand{\calP}{\mathcal{P}}

\newcommand{\tw}{\tilde{w}}
\newcommand{\eps}{\varepsilon}

\newcommand{\E}{\mathbb{E}}

\newcommand{\inpro}[2]{\left\langle #1, #2 \right\rangle}
\newcommand{\intr}{\mathrm{intrinsic}}
\newcommand{\polylog}{\mathrm{polylog}}

%% file: tex/abstract.tex
\begin{abstract}
\vspace{-0.2cm}
   Generating \emph{differentially private (DP) synthetic data} that closely resembles the original private data %
   is a scalable way to mitigate privacy concerns in the current data-driven world.
   In contrast to current practices that train customized models for this task, 
   we aim to %
   \emph{generate \probname{} 
 (\probnameshort{})}, where we treat foundation models as blackboxes and only utilize their inference APIs.
   Such API-based, training-free approaches are easier to deploy as exemplified by the recent surge in the number of API-based apps. 
   These approaches can also leverage the power of large foundation models which are %
   only accessible %
   via their inference APIs. %
   However, this comes with greater challenges due to strictly more restrictive model access and the %
   need to protect privacy from the API provider.
   
   In this paper, we present a new framework called \emph{\algname{} (\algnameshort{})} to solve this problem and show its initial promise on synthetic images.
   Surprisingly, \algnameshort{} can match or even outperform state-of-the-art (SOTA) methods \emph{without any model training.}
   For example, on \cifar{} (with \imagenet{} as the public data), we achieve FID$\leq$7.9 with privacy cost $\epsilon$ = $0.67$, significantly improving the previous SOTA from $\epsilon$ = $32$. 
   We further demonstrate the promise of applying \algnameshort{} on large foundation models such as Stable Diffusion to tackle challenging private datasets with a small number of high-resolution images. 
   The code and data are released at \codeurl{}.
\end{abstract}

%% file: tex/intro.tex
\vspace{-0.4cm}
\section{Introduction}
\label{sec:intro}
\vspace{-0.2cm}

\begin{wrapfigure}[12]{R}{0.42\linewidth}
	\centering
        \vspace{-1.7cm}
        \includegraphics[width=1\linewidth]{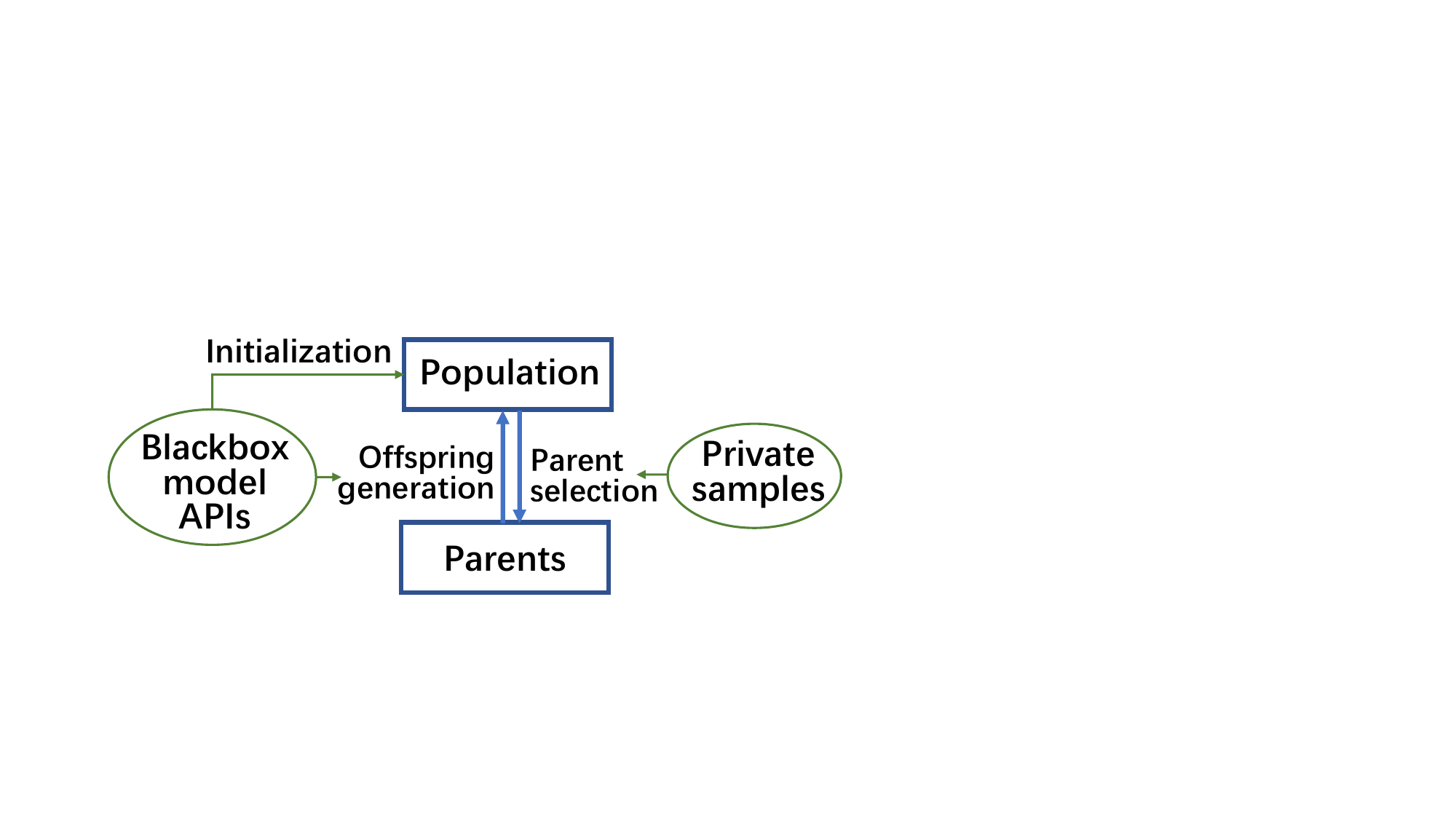}
	\vspace{-0.5cm}
	\caption{We consider the problem of generating DP synthetic data with API access to pre-trained models without any model training. This is in contrast to prior work which assumes full access to pre-trained models and requires training.}
	\label{fig:problem}
\end{wrapfigure}

While data-driven approaches have been successful, privacy is a major concern. For example, statistical queries of a dataset may leak sensitive information about individual users \cite{dwork2014algorithmic}. Entire training samples can be reconstructed from deep learning models \cite{haim2022reconstructing,fredrikson2015model,carlini2021membership,carlini2023extracting,carlini2023quantifying,CarliniLEKS19,CarliniTWJHLRBSEOR21,ChoquetteChooTCP21,TramerSJLJHC22,wang2023decodingtrust}. %
\emph{Differential privacy (DP)} is the gold
standard in quantifying and mitigating these concerns \cite{dwork2006calibrating}. DP algorithms ensure that information about individual samples in the original data cannot be inferred with high confidence from algorithm outputs. 
\emph{Differentially private synthetic data} is the holy grail of DP research~\cite{hu2023sok,jordon2019pate,lin2020using,beaulieu2019privacy,dockhorn2022differentially,yin2022practical,yu2021differentially,he2022exploring,li2021large,ghalebikesabi2023differentially,yue2022synthetic,harder2022differentially,harder2021dp,Savage2023,lin2022data,tang2023privacy}. The goal is to generate a synthetic dataset that is statistically similar to the original data while ensuring DP. The benefits are: 
(1) Thanks to the post-processing property of DP \cite{dwork2014algorithmic}, we can use any existing \emph{non-private} algorithm (e.g., training machine learning (ML) models) as-is on the synthetic data without incurring additional privacy loss. This is more scalable than redesigning and reimplementing every algorithm for DP.
(2) Synthetic data can be shared freely with other parties without violating privacy. This is useful in situations when sharing data is necessary, such as when organizations (e.g., hospitals) want to release datasets to support open research initiatives \cite{beaulieu2019privacy,lin2020using}. (3) Since synthetic data is DP, developers can look at the data directly, which makes algorithm development and debugging a lot easier.

At the same time, with the recent advancement of %
large foundation models, API-based solutions are gaining tremendous popularity, exemplified by the surge of GPT4-based applications. 
In contrast to the traditional paradigm that trains/fine-tunes customized ML models for each application, API-based solutions treat ML models as blackboxes and only utilize APIs\footnote{\label{footnote:openai_api}See \url{https://platform.openai.com/docs/introduction} for examples of APIs. For example, a text completion API can complete a text prompt using a foundation model such as GPT4. An image variation API can produce variations of a given image using a foundation model such as DALLE2.} that provide the input/output functions of the models. 
In fact, many foundation models including GPT4, Bard, and DALLE2 only provide API access %
without releasing model weights or code.
Key reasons for the success of API-based solutions are that APIs offer a clean abstraction of ML and are readily available and scalable. 
Therefore, implementing and deploying these API-based algorithms is easier and faster even for developers without ML expertise. %
Such an approach can also leverage %
powerful foundation models that are only accessible through APIs.
Unfortunately, SOTA DP synthetic data algorithms today are still in the old paradigm \cite{ghalebikesabi2023differentially,li2021large}: they need a customized training process for each dataset, whose implementation requires significant ML engineering efforts (\cref{sec:motivation}).

\begin{figure}[!t]
    \centering
    \vspace{-1.1cm}
     \begin{subfigure}[b]{.44\textwidth}
       \centering      \includegraphics[width=0.9\linewidth]{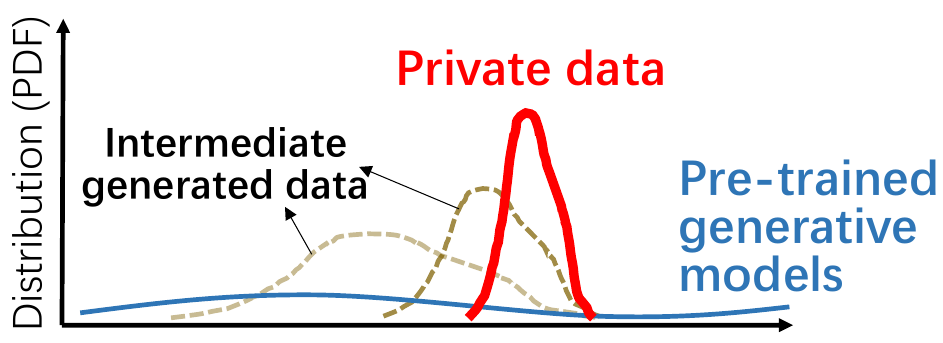}
     \end{subfigure}
     ~~~
     \begin{subfigure}[b]{.41\textwidth}
       	\centering      \includegraphics[width=\linewidth]{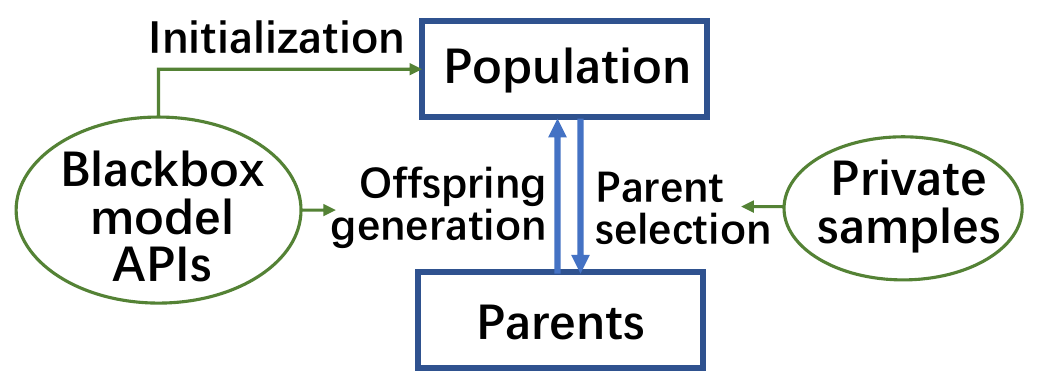}
     \end{subfigure}
     
    \vspace{-0.2cm}
     \caption{\algname{} (\algnameshort{}) framework for DP synthetic data. 
     \textbf{Left: Intuition of \algnameshort{}.} Though private data and pre-trained generative models have very different distributions, the support of the former is likely to be covered by the support of the latter. We gradually shift the distribution of generated data toward private data through \algnameshort{}.
     \textbf{Right: Algorithm of \algnameshort{}.} We maintain a sample set (population), and iteratively select the most similar ones to the private samples (parents) and mutate them to generate the next population (offspring). The initial population and offspring are generated with foundation model APIs. Parent selection is done in DP using private samples.}
    \label{fig:alg}
    \label{fig:dist}
    \vspace{-0.2cm}
\end{figure}
Motivated from these observations, we ask the following ambitious question (\cref{fig:problem}):
\vspace{-0.1cm}
\highlightblock{Can we generate DP synthetic data using blackbox APIs of foundation models?}
\vspace{-0.1cm}

\begin{figure}[t]
    \centering
    \begin{minipage}{0.45\linewidth}
        \centering
        \includegraphics[width=0.88\linewidth]{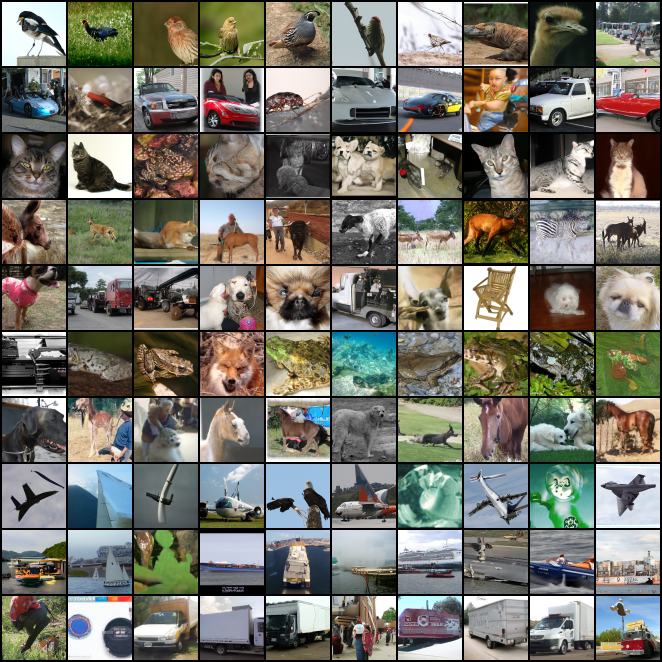}
        \vspace{-0.2cm}
        \caption{Generated samples on \cifar{} with $\bra{0.67,10^{-5}}$-DP. Each row corresponds to one class. FID=7.87. See \cref{app:cifar} for real and generated images side-by-side.}
        \label{fig:cifar_gen_samples}
        \vspace{-0.4cm}
    \end{minipage}
    \hfill
    \begin{minipage}{0.5\linewidth}
	\includegraphics[width=0.9\linewidth]{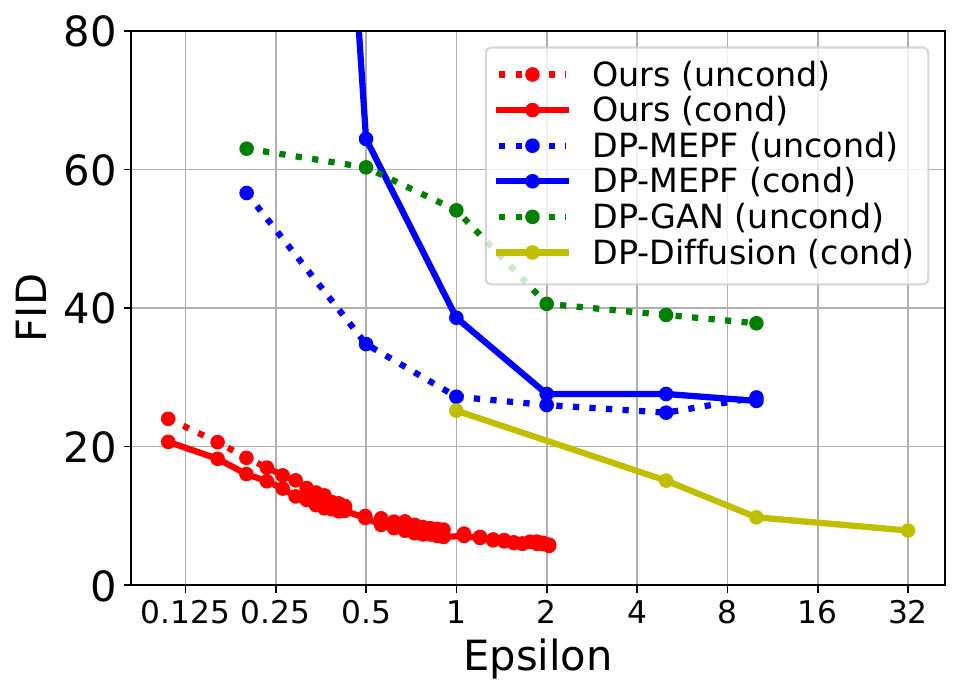}
        \caption{FID \cite{heusel2017gans} (lower is better) v.s. privacy cost $\epsilon$ on \cifar{} ($\delta=10^{-5}$). (Un)cond means (un)conditional generation. Ours achieves the best privacy-quality trade-off compared to DP-MEPF \cite{harder2022differentially}, DP-GAN, DP-Diffusion \cite{ghalebikesabi2023differentially}.}
        \label{fig:cifar_fid_epsilon}
        \vspace{-0.4cm}
    \end{minipage}
\end{figure}

We treat API providers %
as {\em untrusted} entities so we also want to protect user privacy from them, i.e., the API queries we make during generation should also be DP.
If successful, we can potentially democratize the deployment of DP synthetic data in the industry similar to how API-based solutions have facilitated other applications.
This is a challenging task, however, as we do not have access to model weights and gradients by assumption. In this paper, we conduct the first exploration of the potential and limits of this vision on DP synthetic \emph{images}. 
Surprisingly, we show that not only is such a vision realizable, but that it also has the potential to match or improve SOTA training-based DP synthetic image algorithms despite more restrictive model access. Our contributions are:

    \vspace{-0.15cm}
    \myparatightestn{(1) New problem (\cref{sec:problem}).} %
    We highlight the importance of  \emph{\probnamewithunderline{} (\probnameshort)}. Such algorithms are easy to implement and deploy and can leverage the foundation models behind APIs.
    
    \vspace{-0.15cm}
    \myparatightestn{(2) New framework (\cref{sec:alg}).} We propose %
    \emph{\algname{} (\algnameshort{})} algorithm for achieving our goal (\cref{fig:alg}). 
    We consider using 2 popular APIs: random generation %
    and sample variation (i.e., generating a sample similar to the given one).\cref{footnote:dalle,footnote:stablediffusion}
    The key idea is to iteratively use private samples to vote for the most similar samples generated from the blackbox model and ask the blackbox models to generate more of those similar samples. 
    We theoretically prove that the distribution of the generated samples from \algnameshort{} will converge to the private distribution under some modeling assumptions (\cref{sec:theory}). \algnameshort{} only requires (existing) APIs of the foundation models, and does not need any model training.
    
    \vspace{-0.15cm}
    \myparatightestn{(3) Experimental results (\cref{sec:exp}).}\footnote{\revision{In the experiments of this paper, we only experimented with APIs from local models where the user has full control of the model weights and runs it in a controlled environment.}} Some key results are: (a) Surprisingly, without any training, \algnameshort{} can still outperform SOTA training-based DP image generation approaches on some datasets (\cref{fig:cifar_fid_epsilon,fig:cifar_gen_samples}). 
    For example, to obtain FID$\leq7.9$ on \cifar{} dataset, \algnameshort{} (with blackbox access to an \imagenet{}-pre-trained model) only needs $\epsilon=0.67$. In contrast, DP fine-tuning of an \imagenet{}-pre-trained model (prior SOTA) requires $\epsilon=32$~\cite{ghalebikesabi2023differentially}.\\ %
    (b) We show that \algnameshort{} works even when there is significant distribution shift between private and public data. We create a DP synthetic version (with $\eps=7.58$) of Camelyon17, a medical dataset for classification of breast cancer metastases, using the same \imagenet{}-pre-trained model. A downstream classifier trained on our DP synthetic data achieves a classification accuracy of 79.56\% (prior SOTA based on DP fine-tuning is 91.1\% with $\epsilon=10$ \cite{ghalebikesabi2023differentially}).\\ %
    (c) We set up new challenging benchmarks that the DP synthetic image literature has not studied before. We show that with powerful foundation models such as Stable Diffusion \cite{rombach2022high}, \algnameshort{} can work with high-resolution (512x512) image datasets with a small size (100 images), which are common in practice but challenging for current DP synthetic image algorithms.

%% file: tex/background.tex
\vspace{-0.2cm}
\section{Background and Related Work}
\label{sec:background}
\vspace{-0.3cm}

\myparatightestn{Differential Privacy (DP).} 
We say a mechanism $\alg$ is $(\epsilon,\delta)$-DP if for any two neighboring datasets $\calD$ and $\calD'$ which differ in a single entry (i.e., $\calD'$ has one extra entry compared to $\calD$ or vice versa) and for any set $S$ of outputs of $\alg$, we have 
$\probof{\alg\bra{\calD}\in S}\leq e^{\epsilon}\probof{\alg\bra{\calD'}\in S} + \delta$. Intuitively, this means that any single sample cannot influence the mechanism output too much. 

\vspace{-0.1cm}
\myparatightestn{DP synthetic data.} Given a private dataset $\calD$, the goal is to generate a DP synthetic dataset $\alg\bra{\calD}$ which is statistically similar to $\calD$. %
One %
method is to \emph{train generative models from scratch on private data} \cite{lin2020using,beaulieu2019privacy,dockhorn2022differentially} with DP-SGD \cite{abadi2016deep}, a DP variant of stochastic gradient descent. 
Later studies show that \emph{pre-training generative models on public data} before fine-tuning them on private data with DP-SGD \cite{yin2022practical,yu2021differentially,he2022exploring,li2021large,ghalebikesabi2023differentially,yue2022synthetic} gives better privacy-utility trade-offs due to knowledge transfer from public data \cite{yin2022practical,ganesh2023public}, smaller gradient spaces \cite{li2022does}, or better initialization \cite{ganesh2023public}. 
This approach achieves SOTA results on several data modalities such as text and images.
In particular, DP-Diffusion \cite{ghalebikesabi2023differentially} achieves SOTA results on DP synthetic images by pre-training diffusion models %
\cite{sohl2015deep,ho2020denoising} 
on public datasets and fine-tuning them on the private dataset.
Some other methods do not depend on DP-SGD \cite{jordon2019pate,harder2022differentially,harder2021dp,vinaroz2022hermite,cao2021don}. For example, DP-MEPF \cite{harder2022differentially} %
trains generative models to produce synthetic data that matches the (privatized) statistics of the private features. 

Note that all these methods obtain generative models \emph{whose weights are DP}, which can then be used to draw DP synthetic data. It is stronger than our goal which only requires DP synthetic data \cite{lin2021privacy}. In this paper, we do not do any model training and only produce DP synthetic data.

%% file: tex/problem.tex
\vspace{-0.3cm}
\section{\probname{} (\probnameshort)}
\label{sec:problem}
\vspace{-0.3cm}

\subsection{Motivation}
\label{sec:motivation}
\vspace{-0.2cm}

As discussed in \cref{sec:background}, SOTA DP synthetic data algorithms require training or fine-tuning generative models with DP-SGD.
There are some obstacles to deploying them in practice.

\myparaemphtightestn{(1) Significant engineering effort.}
Deploying normal ML training pipelines is hard;
deploying \emph{DP} training pipelines is even harder because most ML infrastructure is not built around this use case. %
Recently, there has been significant progress in making DP training more efficient~\cite{li2021large,he2022exploring} and easy to use (Opacus \cite{opacus} and Tensorflow Privacy). 
However, incorporating them in new codebases and new models is highly non-trivial. %
For example, Opacus requires us to implement our own per-sample gradient calculator %
for new layers. 
Common layers and loss functions that depend on multiple samples (e.g., batch normalization) are often not supported.%

\vspace{-0.1cm}
\myparaemphtightestn{(2) Inapplicability of API-only models.} It may be appealing to take advantage of the powerful foundation models in DP synthetic data generation. However, due to the high commercial value of foundation models, many companies choose to only release inference APIs of the models but not the weights or code. Examples include popular models such as DALLE 2 \cite{ramesh2022hierarchical} and GPT 3/4 \cite{brown2020language,openai2023gpt4} from OpenAI and Bard from Google. In such cases, existing training-based approaches are not applicable.\footnote{Some companies also provide model fine-tuning APIs, e.g., \url{https://platform.openai.com/docs/guides/fine-tuning}. However, they do not support DP fine-tuning and do not provide gradients. Also, uploading sensitive data to these APIs controlled by other companies can lead to privacy violations.} 

In contrast, DP synthetic data approaches that only require model inference APIs could potentially be deployed more easily,
as they do not require ML or DP expertise to conduct modifications inside the model and require minimal modifications when switching to a different model (as long as they support the same APIs).  
In addition, such an approach %
is compatible with \revision{both the models running locally and the models behind APIs.}

\vspace{-0.3cm}
\subsection{Problem Formulation}
\label{sec:problem_formulation}
\vspace{-0.2cm}

We now give a formal statement of \probnameshort{}. We first define a core primitive for DP synthetic data.

\vspace{-0.1cm}
\myparatightestn{\DPWAfull{} (\DPWA).} %
Given a private dataset  $\privatesampleset=\brc{x_i: i\in [\numprisamples]}$ with $\numprisamples$ samples (e.g., images), a distance function $d(\cdot,\cdot)$ between samples and some $p\ge 1$, the goal is to design an $(\epsilon,\delta)$-DP algorithm $\alg$ that outputs a synthetic dataset $\generatedsampleset=\brc{x_i':i\in [\numgensamples]}$
with $\numgensamples$ samples (as a multiset) whose %
distance to $\privatesampleset$, $W_p(\privatesampleset, \generatedsampleset)$,
is minimized. Here $W_p$ is the Wasserstein $p$-distance w.r.t. the distance function $d(\cdot,\cdot)$ (see \cref{app:wasserstein} for the definition).

\vspace{-0.1cm}
\myparatightestn{\probnameshort{}.} We want to solve \DPWA{} where $\alg$ is given blackbox access to foundation models trained on public data via APIs.\cref{footnote:openai_api}  API queries should also be $(\epsilon,\delta)$-DP as API providers cannot be trusted.

In some applications, besides the raw samples $x_i$, we may also care about some auxiliary information such as class labels of images. In such cases, we may write  $\privatesampleset=\brc{(x_i,y_i): i\in [\numprisamples]}$ (and $\generatedsampleset=\brc{(x_i',y_i'): i\in [\numgensamples])}$) where $y_i$ (and $y_i'$) is the auxiliary information of $i$-th sample.

When the distance function $d(\cdot,\cdot)$ is $\ell_2$ (in the sample space), \DPWA{} is closely related to DP Clustering~\cite{ghazi2020differentially,su2016differentially,balcan2017differentially} and DP Heatmaps~\cite{ghazi2022differentially}. But a direct application of 
them %
does not work in our setting; see \cref{app:dpclustering} for more discussions. %

\vspace{-0.3cm}
\subsection{Scope of This Work}
\label{sec:scope}
\vspace{-0.3cm}

\myparatightestn{Data type.} %
\revision{In this paper, we instantiate the above framework on \emph{images}.} %
We consider both unconditional (i.e., no $y_i$) and conditional generation (e.g., $y_i$ can be image categories such as cats/dogs). 

\vspace{-0.1cm}
\myparatightestn{APIs.} In our algorithm design and experiments,  we use 2 APIs, both of which are either directly provided in the APIs of popular models (e.g., DALLE 2,\footnote{\label{footnote:dalle}See \url{https://platform.openai.com/docs/guides/images/usage}.} Stable Diffusion\footnote{\label{footnote:stablediffusion}See \url{https://huggingface.co/docs/diffusers/api/pipelines/stable_diffusion/overview}.}) or can be easily implemented by adapting current APIs (e.g., using appropriate text prompts in GPT APIs\cref{footnote:openai_api}): 

    \vspace{-0.1cm}
    (1) $\randomsampleapi{n}$ that randomly generates $n$ samples. Some APIs %
    accept condition information such as text prompts in text-to-image generation,\cref{footnote:dalle,footnote:stablediffusion} which we omit in the notation for simplicity.
    
    \vspace{-0.1cm}
    (2) $\samplevariationapi{S}$ that generates variations for each  sample in $S$. For images, it means to generate similar images to the given one, e.g., with similar colors or objects.%
    \cref{footnote:dalle,footnote:stablediffusion} Some APIs also support setting the variation degree: $\samplevariationapi{S, \variationdegree}$, where larger $\variationdegree$ indicates more variation.\footnote{If this is not implemented, we can simply compose the \samplevariationapiname{} $\variationdegree$ times to achieve it.} %

%% file: tex/algorithm.tex
\vspace{-0.4cm}
\section{\algname{} (\algnameshort{})}
\label{sec:alg}
\vspace{-0.3cm}

Foundation models have a broad and general model of our world from their extensive training data. Therefore, we expect that foundation models can generate samples close to private data with non-negligible probability. The challenge is that by naively calling the APIs, the probability of drawing %
such samples
is quite low.  We need a way to \emph{guide} the generation towards private samples.

\vspace{-0.1cm}
Inspired by \emph{evolutionary algorithms (EA)}  \cite{davis1987genetic} (\cref{app_evolutionary}), we propose \emph{\algname{} (\algnameshort{})} framework for generating DP synthetic data via APIs. See \cref{fig:alg} for the intuition behind \algnameshort{}.
The complete algorithm is in \cref{alg:main}.   %
Below, we discuss the components %
in detail.

\label{sec:implementation}

\setlength{\textfloatsep}{0pt} %
\setlength{\floatsep}{0pt} 
\vspace{-0.4cm}
\input{tex/alg_main_unconditional}
\vspace{-0.8cm}
\input{tex/alg_voting}
\vspace{-0.4cm}

\myparatightestn{Initial population (\cref{line:initial}).} We use \randomsampleapiname{} to generate the initial population. If there is public information about the private samples (e.g., they are dog images), we can use this information as prompts to the API to seed a better initialization.

\myparatightestn{Fitness function (\cref{line:dp_voting} or \cref{alg:voting})}. 
We need to evaluate how useful each sample in the population is for modeling the private distribution. Our idea is that, if a sample in the population is surrounded by many private samples, then we should give it a high score. 
To implement this, we define the fitness function of a sample $x$ as \emph{the number of private samples whose nearest neighbor in the population is $x$}. A higher fitness value means that more private samples are closest to it. More details are below:

\vspace{-0.1cm}
\myparaemphtightestn{(1) Distance function (\cref{line:distance}, \cref{alg:voting}).} To define ``nearest neighbor'', we need a distance function that measures the %
similarity of two samples. A naive way is to use $\ell_2$ distance $\distancefunction{x,z}=\brnt{x-z}$, where $x$ is from the private dataset and $z$ is from the population. However, it is well-known that $\ell_2$ distance on pixel space is not a good metric for images. For example, 
a small shift of an object can
result in a high $\ell_2$ distance. We therefore compute the $\ell_2$ distance in the embedding space:
\begin{align}
    \distancefunction{x,z}=\brnt{\embeddingnetwork{x}-\embeddingnetwork{z}} \label{eq:distance_no_packing}
\end{align}
where $\embeddingnetworkname$ is a network for extracting image embeddings such as inception embedding 
\cite{szegedy2016rethinking} or CLIP embedding \cite{radford2021learning}. 

\vspace{-0.1cm}
\myparaemphtightestn{(2) \Packing{}.} 
The above approach gives high scores to the good samples in the \emph{current} population. However, as we will see later, these good samples will be modified through \samplevariationapiname{} for the next population. Therefore, it is better to ``look ahead'' to compute the distance based on the modified samples as if they are kept in the population. We modify \cref{eq:distance_no_packing} to compute the distance between the embedding of $x$ and the \emph{mean embedding of $k$ variations of $z$}:
$\distancefunction{x,z}=\brnt{\embeddingnetwork{x}-\frac{1}{\packingdegree{}}\sum_{i=1}^k \embeddingnetwork{z^i}}$, %
where $\packingdegree$ is called \emph{lookahead degree}, and $z^1,\dots,z^k$ are variations of $z$ obtained via $\samplevariationapiname{}$.

\vspace{-0.1cm}
\myparaemphtightestn{(3) Noise for DP (\cref{line:add_noise}, \cref{alg:voting}).}
Because this step utilizes private samples, we need to add noise to ensure DP. We add i.i.d. Gaussian noise from $\normaldistribution{0}{\noisemultiplier}$. %
The privacy analysis is presented in \cref{sec:privacy_analysis}.

\vspace{-0.1cm}
\myparaemphtightestn{(4) Thresholding (\cref{line:threshold}, \cref{alg:voting}).} When the number of generated samples is large, the majority of the histogram will be DP noise added above. To make the signal-noise ratio larger, we set a threshold $\threshold$ to each bin of the histogram. Similar ideas have been used in DP set union \cite{gopi2020differentially}.

\vspace{-0.1cm}
In summary, we called the above fitness function \emph{\dpvotingname{}}. Note that it is not the traditional ``histogram'' on a continuous space that requires binning. Instead, it is a histogram built on the generated samples: the value of $i$-th bin means the (privatized) number of \revision{private samples} whose nearest neighbor among the generated ones is the $i$-th sample. See \cref{app:related_work} for related work.

\myparatightestn{Parent selection (\cref{line:histogram_sampling}).}  We sample from the population according to the \dpvotingname{} so that a sample with more private samples around is more likely to be selected. %

\myparatightestn{Offspring generation (\cref{line:offspring}).}  We use \samplevariationapiname{} to get variants of the parents as offsprings.

Please see \cref{sec:theory} for the convergence analysis of \algnameshort{}.

\vspace{-0.3cm}
\subsection{Conditional Generation}
\label{sec:conditional}
\vspace{-0.3cm}

The above procedure is for unconditional generation. To support conditional generation, i.e., each generated sample is associated with a label such as an image class (e.g., cats v.s. dogs),  we take a simple approach: \emph{we repeat the above process for each class of samples in the private dataset separately}. See \cref{alg:main_full} for the full algorithm.

\vspace{-0.3cm}
\subsection{Generating Unlimited Number of Samples}
\label{sec:algorithm_unlimited}
\vspace{-0.3cm}
Our algorithm in \cref{alg:main} is preset with a fixed number of generated samples $\numgensamples$. 
What if users want more samples afterward?
In prior training-based methods \cite{ghalebikesabi2023differentially}, this is easy to achieve: one can draw an arbitrary number of samples from the trained generative models without additional privacy cost.
In this section, we want to highlight that \algnameshort{} can also do that, again \emph{with API access only}. We can simply generate an unlimited number of samples by \revision{calling variation API multiple times, each with the generated dataset as input: $[\samplevariationapi{\generatedsampleset},\ldots,\samplevariationapi{\generatedsampleset}]$}. In \ref{sec:exp_unlimited_number_of_samples}, we will see that this simple algorithm is sufficient to provide more useful samples for downstream applications.

\vspace{-0.3cm}
\subsection{Privacy Analysis}
\label{sec:privacy_analysis}
\vspace{-0.3cm}
Unlike the analysis of DP-SGD which requires complicated DP composition theorems (e.g., \citet{gopi2021numerical,mironov2017renyi}) due to subsampling, \algnameshort{} does not have subsampling steps and therefore the privacy analysis is rather straightforward.
The DP guarantee of  the \textbf{unconditional version of \algnameshort{} (\cref{alg:main})} can be reasoned as follows:
\vspace{-0.15cm}
\begin{packeditemize}
    \item \textbf{Step 1: The sensitivity of \dpvotingname{} (\cref{line:hist_init,line:hist_for,line:distance,line:hist_assign} in \cref{alg:voting}).} %
    Each private sample only contributes one vote. If we add or remove one sample, the resulting histogram will change by 1 in the $\ell_2$ norm. Therefore, the sensitivity is 1.
    \item \textbf{Step 2: Regarding each \algnameshort{} iteration as a Gaussian mechanism.} \cref{line:add_noise} adds i.i.d. Gaussian noise with standard deviation $\noisemultiplier$ to each bin. This is a standard Gaussian mechanism \cite{dwork2014algorithmic} with noise multiplier $\noisemultiplier$. %
    \item \textbf{Step 3: Regarding the entire \algnameshort{} algorithm as $\numiterations$ \revision{adaptive} compositions of Gaussian mechanisms}, as \algnameshort{} is simply applying \cref{alg:voting}  $\numiterations$ times sequentially.
    \item \textbf{Step 4: Regarding the entire \algnameshort{} algorithm as one Gaussian mechanism with noise multiplier $\noisemultiplier/\sqrt{\numiterations}$.} It is a standard result from \citet{dong2022gaussian} (see Corollary 3.3 therein).
    \item \textbf{Step 5: Computing DP parameters $\epsilon$ and $\delta$.} Since the problem is simply computing $\epsilon$ and $\delta$ for a standard Gaussian mechanism, we use the formula from \citet{balle2018improving} directly.
\end{packeditemize}

\vspace{-0.15cm}
For the \textbf{conditional version of \algnameshort{} (\cref{alg:main_full})}, since it does the unconditional version for each class separately (discussed in \cref{sec:conditional}), adding or removing one sample will only influence the results of one class. For that class, the impact due to the added/removed sample is also bounded, as seen in the privacy analysis above. Therefore, \cref{alg:main_full} is also DP. In fact, we can show that \emph{the privacy guarantee of \cref{alg:main_full} is the same as \cref{alg:main}, and it protects the labels of samples in the same level as DP-SGD \cite{ghalebikesabi2023differentially}}. Please refer to \cref{app:alg} for more details.

This privacy analysis implies that releasing all the (intermediate) generated sets $S_1,\dots,S_T$ also satisfies the same DP guarantees. Therefore \algnameshort{} provides the same privacy even from the API provider.

%% file: tex/alg_main_unconditional.tex
\begin{algorithm}[thpb]
    \DontPrintSemicolon
    \LinesNumbered
	\BlankLine
	\SetKwInOut{Input}{Input}
	\SetKwInOut{Output}{Output}
	\caption{\algname{} (\algnameshort{})}
    \label{alg:main}
	\Input{
 Private samples: $\privatesampleset=\brc{x_i}_{i=1}^{\numprisamples}$\\
 Number of iterations: $\numiterations$\\
 Number of generated samples: $\numgensamples$\\
 Noise multiplier for \dpvotingname{}: $\noisemultiplier$\\
 Threshold for \dpvotingname{}: $\threshold$
	}
        \Output{
        Synthetic data: $\generatedsampleset$
        }
	\BlankLine
            $S_0 \leftarrow \randomsampleapi{\numgensamples}$ \; \label{line:initial}
    	\For{$t \leftarrow 1, \ldots, \numiterations$}
    	{
                $histogram_t \leftarrow \dpvotingfunction{\privatesampleset,S_{t-1}, \noisemultiplier,\threshold}$ \label{line:dp_voting} \tcp*{See \cref{alg:voting}} 
                $\calP_t \leftarrow histogram_t/\mathrm{sum}(histogram_t)$ \tcp*{$\calP_t$ is a distribution on $S_t$}
                $S_{t}'\leftarrow $ draw $N_\syn$ samples with replacement from $\calP_t$ \label{line:histogram_sampling} \tcp*{$S_t'$ is a multiset}
                $S_{t}\leftarrow\samplevariationapi{S_t'}$ \label{line:offspring}
            }
	\Return{ $S_T$}
\end{algorithm}

%% file: tex/alg_voting.tex
\begin{algorithm}[thpb]
    \DontPrintSemicolon
    \LinesNumbered
	\BlankLine
	\SetKwInOut{Input}{Input}
	\SetKwInOut{Output}{Output}
	\caption{\dpvotingname{} (\dpvotingfunctionname{})}
    \label{alg:voting}
	\Input{Private samples: $S_{\priv}$\\
 Generated samples: $S=\brc{z_i}_{i=1}^{n}$\\
 Noise multiplier: $\noisemultiplier$\\
 Threshold: $\threshold$\\
 Distance function: $\distancefunction{\cdot,\cdot}$
	}
 \Output{DP nearest neighbors histogram on $S$}
	\BlankLine
        $histogram\leftarrow[0,\ldots,0]$\label{line:hist_init}\; 
        \For{$x_{\priv} \in S_{\priv}$\label{line:hist_for}}{ 
            $i=\arg\min_{j\in\brb{n}} \distancefunction{x_{\priv},z_j}$ \label{line:distance}\;
            $histogram[i] \leftarrow histogram[i] + 1$ \label{line:hist_assign}
        }
        $histogram \leftarrow histogram + \normaldistribution{0}{\noisemultiplier I_n}$ \label{line:add_noise}\tcp*{Add noise to ensure DP}
        $histogram \leftarrow \max\bra{histogram -\threshold, 0}$ \tcp*{`max', `-' are element-wise} \label{line:threshold}
	\Return{ $histogram$}
\end{algorithm}

%% file: tex/experiment.tex
\vspace{-0.4cm}
\section{Experiments}
\label{sec:exp}
\vspace{-0.4cm}

In \cref{sec:exp_sota}, we compare \algnameshort{} with SOTA training-based methods on standard benchmarks to understand its promise and limitations. 
In \cref{sec:exp_highreso}, we present proof-of-concept experiments to show 
how \algnameshort{} can utilize the power of large foundation models.
We did (limited) hyper-parameter tunings in the above experiments; following prior DP synthetic data work \cite{yu2021differentially,ghalebikesabi2023differentially}, we ignore the privacy cost of hyper-parameter tuning. However, as we will see in the ablation studies (\cref{sec:exp_ablation,app:ablation}), \algnameshort{} stably outperforms SOTA across a wide range of hyper-parameters, and the results can be further improved with better hyper-parameters than what we used. Detailed hyper-parameter settings and more results such as \emph{generated samples} and \emph{their nearest images in the private dataset} are in \cref{app:cifar,app:camelyon,app:stable_diffusion}.

\vspace{-0.3cm}
\subsection{Comparisons to State-of-the-Art}
\label{sec:exp_sota}
\vspace{-0.3cm}

\myparatightestn{Public information.} We use standard benchmarks \cite{ghalebikesabi2023differentially} which treat \imagenet{} \cite{deng2009imagenet} as public data. For fair comparisons, we only use \imagenet{} as public information in \algnameshort{}: \emph{(1) Pre-trained model.} Unlike the SOTA \cite{ghalebikesabi2023differentially} which trains customized diffusion models, we simply use public  \imagenet{} pre-trained diffusion models (pure image models without text prompts) \cite{nichol2021improved}. %
\revision{\randomsampleapiname{} and \samplevariationapiname{} are implemented using the same pre-trained model (see \cref{app:cifar}).}
\emph{(2) Embedding (\cref{eq:distance_no_packing}).} We use   \imagenet{} inception embedding \cite{szegedy2016rethinking}.
\algnameshort{} is not sensitive to embedding choice though and we get good results even with CLIP embeddings (\cref{fig:cifar_ablation_clip} in \cref{app:ablation}).

\myparatightestn{Baselines.} We compare with DP-Diffusion \cite{ghalebikesabi2023differentially}, DP-MEPF \cite{harder2022differentially}, and DP-GAN \cite{harder2022differentially,goodfellow2020generative}.
DP-Diffusion \cite{ghalebikesabi2023differentially} is the current SOTA that achieves the best results on these benchmarks.  Baseline results are taken from their paper.

\begin{wrapfigure}[18]{R}{0.45\linewidth}
    \centering
    \vspace{-0.9cm}
    \includegraphics[width=0.9\linewidth]{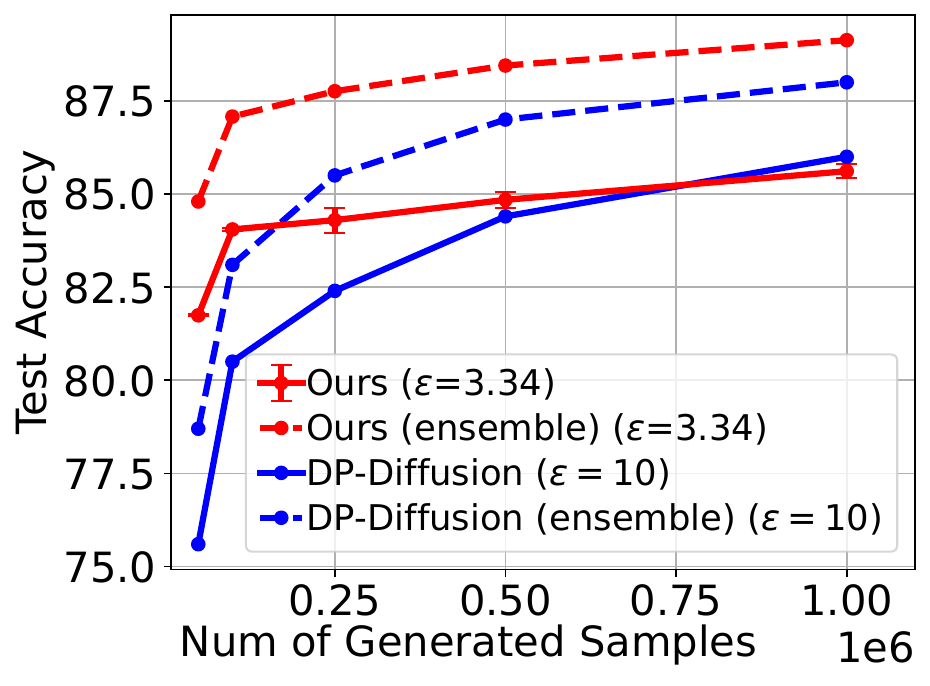}
    \vspace{-0.3cm}
    \caption{Downstream classification accuracy (higher is better) on \cifar{} ($\delta=10^{-5}$). The baseline results are taken from \citet{ghalebikesabi2023differentially}. Two "ensemble" lines are from ensembles of 5 classifiers. The other two lines show the average accuracy of 5 independently trained classifiers with error bars. Our \algnameshort{} achieves better accuracy across almost all settings with smaller privacy costs.}
    \label{fig:cifar_acc_numsamples_more_runs}
\end{wrapfigure}

\myparatightestn{Outline.} 
We test \algnameshort{} on private datasets that are either similar to or differ a lot from \imagenet{} in \cref{sec:exp_cifar,sec:exp_camelyon}.  We demonstrate that \algnameshort{} can generate an unlimited number of useful samples in \cref{sec:exp_unlimited_number_of_samples}. We show that \algnameshort{} is computationally cheaper than train-based methods in \cref{sec:exp_computation}. %

\vspace{-0.2cm}
\subsubsection{Moderate Distribution Shift (\imagenet{} $\rightarrow$ \cifar{})}
\label{sec:exp_cifar}
\vspace{-0.2cm}

We treat \cifar{} \cite{krizhevsky2009learning} as private data. 
Given that both \imagenet{} and \cifar{} are natural images, it is a relatively easy task for \algnameshort{} (and also for the baselines). \cref{fig:cifar_fid_epsilon,fig:cifar_acc_numsamples_more_runs,fig:cifar_gen_samples} show the results. Surprisingly, despite the fact that we consider strictly more restrictive model access and do not need training, \algnameshort{} still outperforms the SOTA training-based methods. Details are below.

\myparatightestn{Sample quality.} \cref{fig:cifar_fid_epsilon} shows the trade-off between privacy cost and FID, a popular metric for image quality \cite{heusel2017gans}. For either conditional or unconditional generation, \algnameshort{} outperforms the baselines significantly. For example, to reach FID$\leq 7.9$, \citeauthor{ghalebikesabi2023differentially} requires $\epsilon=32$, \citeauthor{harder2022differentially} cannot achieve it even with infinity $\epsilon$, whereas our \algnameshort{} only needs $\epsilon=0.67$.

\vspace{-0.1cm}
\myparatightestn{Downstream classification accuracy.} 
We train a downstream WRN-40-4 classifier \cite{zagoruyko2016wide} from scratch on 50000 generated samples and test the accuracy on \cifar{} test set. This simulates how users would use synthetic data, and a higher accuracy means better utility. 
\cref{fig:cifar_acc_numsamples_more_runs} shows the results (focus on the left-most points with num of generated samples = 50000 for now).
\citet{harder2022differentially} achieves 51\% accuracy with $\epsilon=10$ (not shown). Compared with the SOTA \cite{ghalebikesabi2023differentially}, \algnameshort{} achieves better accuracy (+6.1\%) with less privacy cost. Further with an ensemble of 5 classifiers trained on the same data, \algnameshort{} is able to reach an accuracy of 84.8\%. 
\revision{For reference, the SOTA DP classifier pre-trained on ImageNet (without DP) and fine-tuned on CIFAR10 (with DP) \cite{de2022unlocking} achieves 94.8\% and 95.4\% accuracies with epsilon=1 and 2 respectively. It is not surprising that DP classifiers outperform \algnameshort{} (and other DP synthetic data approaches) on classification tasks, as DP classifiers are targeted at and optimized for a single task whereas DP synthetic data is general-purpose.} 

The above results suggest that when private and public images are similar, \algnameshort{} is a promising framework given its better privacy-utility trade-off and the API-only requirement.

\vspace{-0.2cm}
\subsubsection{Large Distribution Shift (\imagenet{} $\rightarrow$ \camelyon{})}
\label{sec:exp_camelyon}
\vspace{-0.3cm}

\begin{figure}[t]
    \centering
    \vspace{-1.1cm}
    \includegraphics[width=0.45\linewidth]{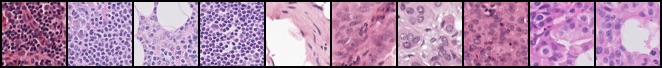}
    \put(-100,20){Real}
    ~
    \includegraphics[width=0.45\linewidth]{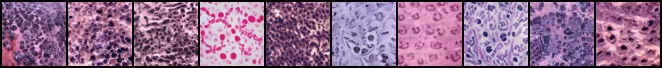}
    \put(-180,20){Generated ($\bra{9.92,3\cdot 10^{-6}}$-DP, FID=10.66)}
    \caption{Real and generated images from \camelyon{}. More in \cref{app:camelyon}.}
    \vspace{0.2cm}
    \label{fig:camelyon_gen_few}
\end{figure}

Next, we consider a hard task for \algnameshort{}, where the private dataset is very different from \imagenet{}. %
We use \camelyon{} dataset \cite{bandi2018detection,koh2021wilds} as private data which contains 302436 images of histological lymph node sections with labels on whether it has cancer (real images in \cref{fig:camelyon_gen_few,fig:camelyon_real}). %
Despite the large distribution shift, training-based methods can update the model weights to adapt to the private distribution (given enough samples). However, \algnameshort{} can only draw samples from APIs as is. %

We find that even in this challenging situation, \algnameshort{} can still achieve non-trivial results. Following  \citet{ghalebikesabi2023differentially}, we train a WRN-40-4 classifier from scratch on 302436 generated samples  and compute the test accuracy. We achieve %
\revision{80.33\%}
accuracy with %
\revision{$(10.00,3\cdot 10^{-6})$-DP}. Prior SOTA \cite{ghalebikesabi2023differentially} is 91.1\% with $(10,3\cdot 10^{-6})$-DP. Random guess is 50\%. \cref{fig:camelyon_gen_few} (more in \cref{fig:camelyon_gen_samples}) shows that generated images from \algnameshort{} are very similar to \camelyon{} despite that the pre-trained model is on \imagenet{}. \cref{fig:camelyon_gen_samples_iteration} further shows how the generated images are gradually moved towards \camelyon{} across iterations. 

\revision{\myparatightestn{Why \algnameshort{} works under large distribution shifts.} 
Even though the diffusion model is trained on natural images, the support of the generated distribution spans the entire sample space. \algnameshort{} is effective in guiding the model to generate samples from the region that is low-density in the original pre-trained distribution but high-density in the private distribution. See \cref{app:camelyon} for more explanations.}

\revision{\myparatightestn{Limitation.}} These results demonstrate the effectiveness of \algnameshort{}. But when public models that are similar to private data are not available and when there is enough private data, the traditional training-based methods are still more promising at this point if the privacy-utility trade-off is the only goal. However, given the benefit of API-only assumption and the non-trivial results that \algnameshort{} already got, it is worth further exploiting the potential of \algnameshort{} in future work. Indeed, we expect these results can be improved with further refinement of \algnameshort{} (\cref{app:ablation}).

\vspace{-0.2cm}
\subsubsection{Generating Unlimited Number of Samples}
\label{sec:exp_unlimited_number_of_samples}
\vspace{-0.3cm}

We use the approach in \cref{sec:algorithm_unlimited} to generate more synthetic samples from \cref{sec:exp_cifar} and train classifiers on them. The results are in \cref{fig:cifar_acc_numsamples_more_runs}. Similar as \citet{ghalebikesabi2023differentially}, the classifier accuracy improves as more generated samples are used. 
With an ensemble of 5 classifiers, we reach 89.13\% accuracy with 1M samples.
\emph{This suggests that \algnameshort{} has the same capability as training-based methods in generating an unlimited number of useful samples.} 

We also see two interesting phenomena: \emph{(1) The gap between \algnameshort{} and DP-Diffusion diminishes as more samples are used.} We hypothesize that it is due to the limited improvement space: As shown in  \citet{ghalebikesabi2023differentially}, even using an \imagenet{} pre-trained classifier, the best accuracy DP-Diffusion achieves is %
close to the best points in \cref{fig:cifar_acc_numsamples_more_runs}. 
\emph{(2) The benefit of \algnameshort{} is more evident over ensembling, especially when having more generated samples.} We hypothesize it is due to different ways of generating more samples. In \citet{ghalebikesabi2023differentially}, the newly generated samples are from the same distribution as the first 50000 samples. In contrast, the newly generated samples in \algnameshort{} are from a different distribution (see \cref{sec:algorithm_unlimited}), which could be more diverse and therefore are more beneficial for ensembling approaches.

    \begin{wrapfigure}[12]{R}{0.45\linewidth}
        \centering
        \vspace{-0.8cm}
        \includegraphics[width=0.8\linewidth]{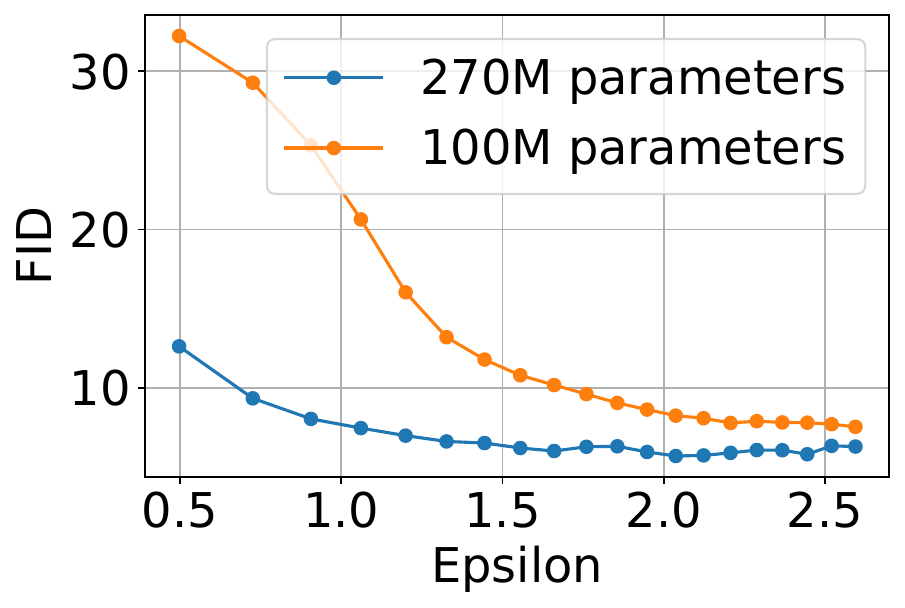}
        \vspace{-0.3cm}
        \captionof{figure}{Ablation studies on the pre-trained model. Both are public diffusion models trained on \imagenet{} \cite{nichol2021improved}. The 270M network is conditional, whereas the 100M network is unconditional.}
        \label{fig:cifar_ablation_pretrained_network}
    \end{wrapfigure}

\vspace{-0.1cm}
\subsection{More Challenging Benchmarks with Large Foundation Models}
\label{sec:exp_highreso}
\vspace{-0.3cm}
We demonstrate the feasibility of applying \algnameshort{} on large foundation models with Stable Diffusion \cite{rombach2022high}.

\myparatightestn{Data.} Ideally we want to experiment with a dataset that has no overlap with Stable Diffusion's training data.\footnote{The training set of Stable Diffusion is public. However, it is hard to check if a public image or its variants (e.g., cropped, scaled) have been used to produce images in it. Therefore, we resort to our own private data.} We take the safest approach: we construct two datasets with photos of the author's two cats that have never been posted online. Each dataset has 100 512x512 images. Such high-resolution datasets with a small number of samples represent a common need in practice (e.g., in health care), %
but are challenging for DP synthetic data: to the best of our knowledge, no prior training-based methods have reported results on datasets with a similar resolution or number of samples.
The dataset is released at \codeurl{} as a new benchmark. %
See \cref{app:stable_diffusion} for all images.

\myparatightestn{API implementation.} We use off-the-shelf Stable Diffusion APIs (see \cref{app:stable_diffusion}).

\myparatightestn{Results.} We run \algnameshort{} for these two datasets with the same hyperparameters. \cref{fig:cat_gen} show examples of generated images for each of the cat datasets. We can see that \algname{} correctly captures the key characteristics of these two cats. See  \cref{app:stable_diffusion} for all generated images.

\begin{figure}[t]
    \centering
    \vspace{-1.1cm}
    \includegraphics[width=0.09\linewidth]{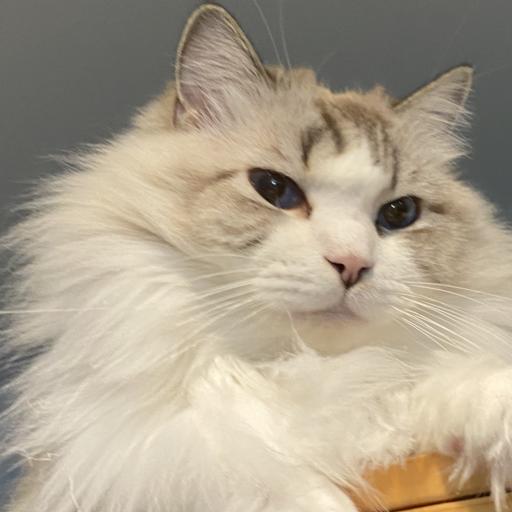}
    \put(-25,40){Real}
    ~
    \includegraphics[width=0.36\linewidth]{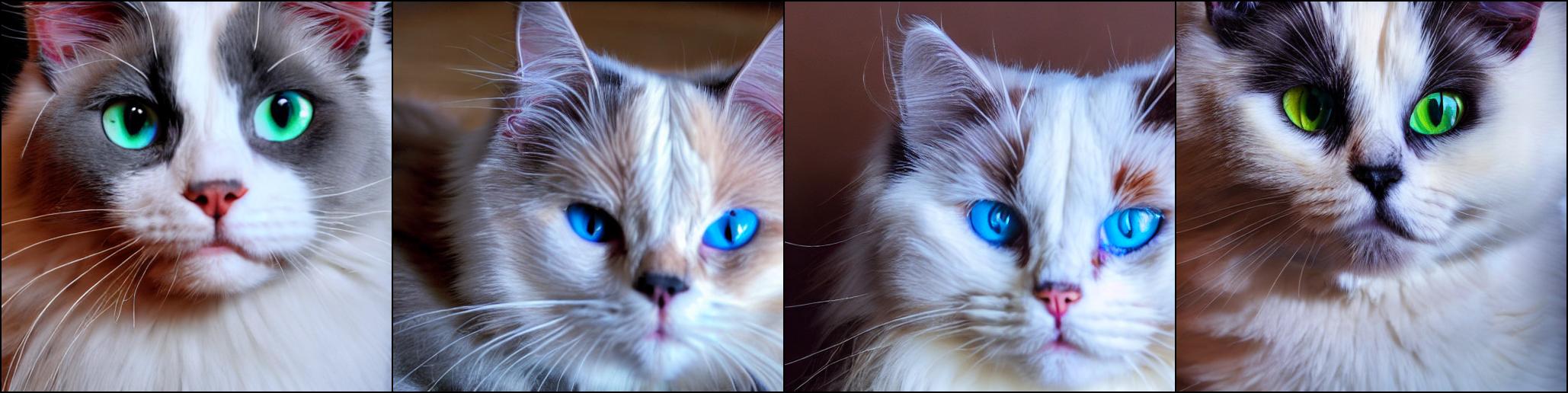}
    \put(-130,40){Generated ($(6.62,10^{-3})$-DP)}
    ~
    \includegraphics[width=0.09\linewidth]{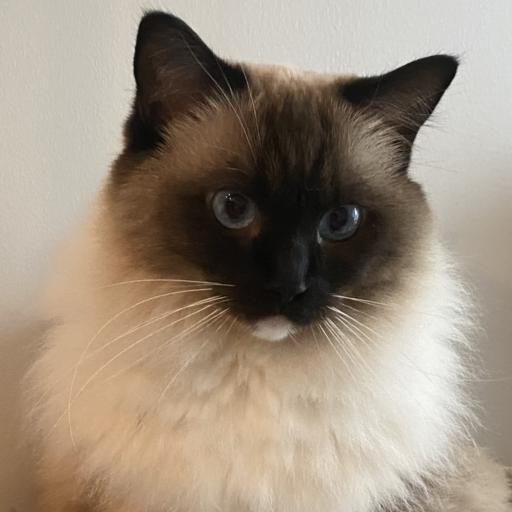}
    \put(-25,40){Real}
    ~
    \includegraphics[width=0.36\linewidth]{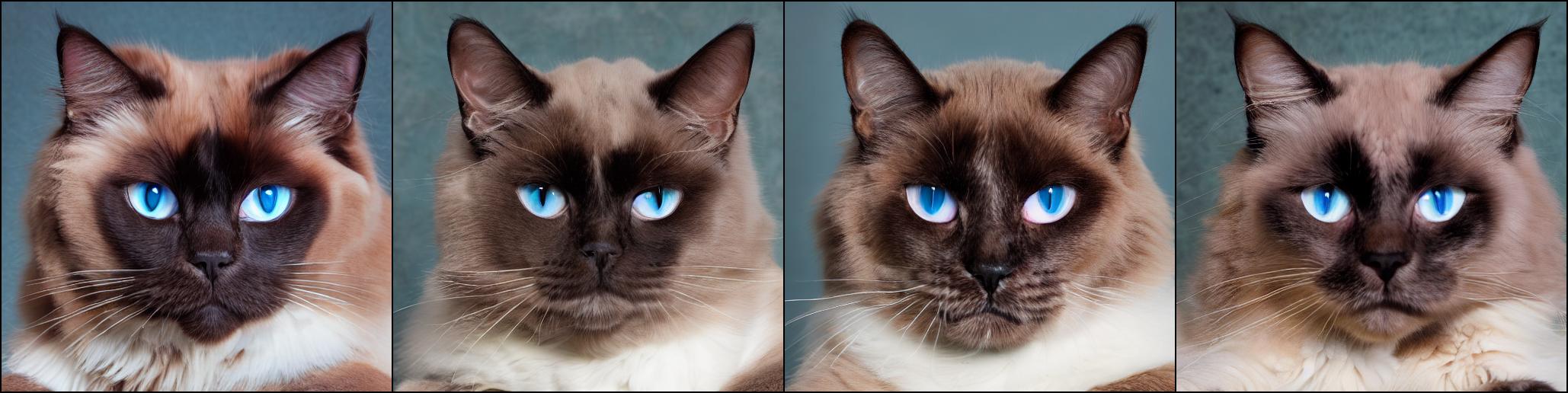}
    \put(-130,40){Generated ($(6.62,10^{-3})$-DP)}
    ~
    \vspace{-0.2cm}
    \caption{Real \& generated images from \catcookie{} (left) and \catruby{} (right). More in \cref{app:stable_diffusion}.}
    \vspace{0.2cm}
    \label{fig:cat_gen}
\end{figure}

\vspace{-0.3cm}
\subsection{Ablation Studies}
\label{sec:exp_ablation}
\vspace{-0.3cm}

\myparatightestn{Pre-trained network.}
\cref{fig:cifar_ablation_pretrained_network} shows the results with two different \imagenet{} pre-trained networks: one is larger (270M) with \imagenet{} class labels as input; the other one is smaller (100M) without label input (see \cref{app:conditional_unconditional} for implementation details). 
In all experiments in \cref{sec:exp_sota}, we used the 270M network. Two takeaways are: (1) The 270M network (trained on the same dataset) improves the results. This is expected as larger and more powerful models can learn public distributions better. %
 This suggests the potential of \algnameshort{} with future foundation models with growing capabilities.
(2) Even with a relatively weak model (100M), \algnameshort{} can still obtain good results that beat the baselines (though with a slower convergence speed), suggesting the effectiveness of \algnameshort{}.

More ablation studies on the \packing{} degree $\packingdegree$, the number of generated samples $\numgensamples$, the threshold $\threshold$, and the embedding network are in \cref{app:ablation}, where we see that \algnameshort{} obtains good results across a wide range of hyper-parameters, and the main results %
can be improved with better hyper-parameters.

%% file: tex/discussion.tex
\vspace{-0.4cm}
\section{Limitations and Future Work}
\label{sec:discussion}
\vspace{-0.4cm}

\myparatightestn{Algorithms.} (1) %
We did not take the number of API calls into account when optimizing \algnameshort{}.
One future work is to optimize the number of API calls along with privacy-utility tradeoffs.
(2) We considered two APIs: \randomsampleapiname{} and \samplevariationapiname{}. It is interesting to consider \algnameshort{} variants that leverage the large set of APIs.\cref{footnote:dalle,footnote:stablediffusion}
(3) When the distributions of private data and foundation models are too different, \algnameshort{} achieved non-trivial classification results, but was still worse than SOTA (\cref{sec:exp_camelyon}). It is interesting to understand the limits of \algnameshort{} %
and explore potential improvements.
(4) \algnameshort{} requires an embedding network (\cref{eq:distance_no_packing}) that projects the samples into a space for measuring the similarity between samples. While for images there are plenty of open-source embedding networks to choose, %
it may not be the case 
for other modalities. 
(5) Recent papers show the phenomenon of Model Autophagy Disorder (MAD), where repeatedly \emph{training} the next generative model using synthetic data from the previous one can result in degraded sample quality \cite{alemohammad2023self}. While \algnameshort{} also repeatedly uses synthetic data to create new synthetic data in the main loop (\cref{alg:main}), it is different in two aspects: (a) Instead of purely relying on the synthetic data, \algnameshort{} utilizes the signals from private data to guide the generation; (b) \algnameshort{} does repeated \emph{inference} instead of repeated \emph{training}. It would be interesting to study the MAD effect in the context of \algnameshort{}. 
(6) Solving \probnameshort{} in the Local/Shuffle DP model and in federated learning settings.

\vspace{-0.1cm}
\myparatightestn{Applications.} (1) New privacy-preserving vision applications that were previously challenging but are now possible due to \algnameshort{}'s capability of generating %
high-resolution DP synthetic images with small dataset sizes.
(2) The use of 
\algnameshort{} in other data modalities beyond images such as texts, tabular data, and time series data.
(3) Besides DP, there are other parallel/orthogonal privacy concerns, notions, and metrics \cite{issa2019operational,lin2022distributional,lin2023summary,dataarticle}. It is interesting to study if \algnameshort{} can be used to generate privacy-preserving synthetic data with respect to these privacy notations.

%% file: tex/ethics.tex
\section{\revision{Ethics Statement}}
\revision{
\algnameshort{} uses the APIs of \emph{pre-trained} models. The DP guarantee of \algnameshort{} is rigorous for the data used in the \algnameshort{} algorithm (i.e., $\privatesampleset$). That being said, \algnameshort{} does not address the privacy of \emph{pre-training} data of foundation models, which is a different goal. \algnameshort{} has no control over the pre-training data---any privacy breaches in the pre-training data are attributable to the data holder (e.g., leaking the data publicly) or the foundation model developer (e.g., using data without permission). However, as a \algnameshort{} user, it is advisable to ensure no overlap between the pre-training data and $\privatesampleset$ for liability reasons.
Depending on whether the APIs are from \emph{blackbox models}, which can only be accessed through APIs  (e.g., DALLE3), or \emph{local models}, whose weights and architectures are accessible by the users (e.g., Stable Diffusion), it has different implications.
\begin{packeditemize}
    \item \emph{Using APIs from blackbox models.}  
    Since most blackbox models do not reveal their training dataset, it is safer to only consider $\privatesampleset$ that was never been shared or posted online. %
    For instance, a hospital who wants to share a DP synthetic version of its proprietary medical records can safely run \algnameshort{} if it has never released these medical records to any other party, making it impossible for those records to be in the pre-training data of any foundation model. %
    \item \emph{Using APIs from local models.} For local models, we have full control over the model weights and architectures. We can pre-train the models on data that surely has no overlap with the private data. In all experiments of the paper, we use local models including Improved Diffusion \cite{nichol2021improved} and Stable Diffusion \cite{rombach2022high}. We directly take the pre-trained models from prior work, and we make sure that the private data and the pre-training data have no overlap.
\end{packeditemize}
}

%% file: tex/ack.tex
\section*{Acknowledgement}
The authors would like to thank the anonymous reviewers for their valuable feedback and suggestions. 
The authors would also like to thank Sepideh Mahabadi for the insightful discussions and ideas, and Sahra Ghalebikesabi for the tremendous help in providing the experimental details of DP-Diffusion \cite{ghalebikesabi2023differentially}. 
The authors would like to extend their heartfelt appreciation to Cat Cookie and Cat Doudou for generously sharing their adorable faces in the new dataset, as well as to Wenyu Wang for collecting and pre-processing the photos.

$^\dagger$ This paper is the full version of our previous workshop paper \cite{lin2023differentially}.

%% file: tex/app_related_work.tex
\section{More Related Work}
\label{app:related_work}

\vspace{-0.1cm}
\myparatightestn{DP data selection.}
One key component in \algnameshort{} is to use private samples to select similar generated \emph{samples} (\cref{alg:voting}). Prior work studied similar problems in different applications.

In the federated learning setting, \citet{hou2023privately} select public \emph{datasets} that are similar to the clients' data, and then pre-train the model on the selected datasets before federated fine-tuning. \citet{hong2022outsourcing} first cluster public data, and then use clients' data to select the closest cluster centers (using a similar histogram approach), so that models trained on the selected \emph{clusters} can have better performance for the clients. In contrast to these works, \algnameshort{} selects generated samples at a \emph{sample-level}, so that we can improve the generated data in a more fine-grained manner.

Similar to ours, \citet{yu2023selective} also conduct \emph{sample-level} data selection. They select public \emph{samples} that are similar to the private data for pre-training the model, before DP fine-tuning the model on the private data \cite{yu2023selective}.  Their data selection rule does not provide guarantees on the distance between the selected samples and the private samples, whereas our selection provides distribution convergence guarantees (\cref{sec:theory}).

Nevertheless, given that all these methods deal with DP data selection, they can be interchangeably used in each application. 
It would be interesting to study such extensions in future work.

Furthermore, we apply such data selection \emph{iteratively} on the generated data. Together with the use of foundation model APIs, we can generate DP synthetic data, which is not studied in prior work.

%% file: tex/app_wasserstein.tex
\section{Definition of Wasserstein Distance}
\label{app:wasserstein}

Wasserstein distance is a widely used metric in designing \cite{arjovsky2017wasserstein} and evaluating \cite{heusel2017gans} generative models.
Given probability distributions $\mu,\nu$ on a metric space, the Wasserstein distance w.r.t. to a distance function $d(\cdot,\cdot)$ is defined as $W_p(\mu,\nu)=\inf_\gamma \brb{\E_{(x,y)\sim\gamma} d(x,y)^p}^{1/p}$ where the infimum is over all couplings $\gamma$ of $\mu,\nu$. Also given discrete point sets $S,T$, we use $W_p(S,T)$ to denote the $W_p$-distance between uniform distributions on $S$ and $T$.

%% file: tex/app_evolutionary.tex
\section{A Brief Introduction to Evolutionary Algorithms}
\label{app_evolutionary}

Evolutionary algorithms \cite{davis1987genetic}  are inspired by biological evolution, and the goal is to produce samples that maximize an \emph{objective value}. 
It starts with an \emph{initial population} (i.e., a set of samples), which is then iteratively updated. In each iteration, it selects \emph{parents} (i.e., a subset of samples) from the population according to the \emph{fitness function} which describes how useful they are in achieving better \emph{objective values}. After that, it generates \emph{offsprings} (i.e., new samples) by modifying the parents, hoping to get samples with better objective values, and puts them in the population. By doing so, the population will be guided towards better objective values.

We cannot directly apply existing EA algorithms to our problem. Firstly, our objective is to produce a set of samples that are \emph{jointly} optimal (i.e., closer to the private distribution, \cref{sec:problem_formulation}), instead of optimizing an objective calculated from \emph{individual} samples in typical EA problems. In addition, the differential privacy requirement and the restrictive model API access are unique to our problem. These differences require us to redesign all components of EA.

%% file: tex/app_alg.tex
\section{More Details on \algname{}}
\label{app:alg}

\cref{alg:main_full} shows the full algorithm that supports both unconditional data and conditional data.
For simplicity, we assume that the private dataset $\privatesampleset$ is balanced, i.e., it has an equal number of samples in each label class. Otherwise, we can first estimate the counts using the Laplace mechanism and use these counts to generate synthetic data of appropriate size in each label class.

\input{tex/alg_main}

\myparatightestn{Privacy analysis.} 
For ease of understanding the privacy guarantee of \cref{alg:main_full}, we can consider a modified version of \cref{alg:main_full} as in \cref{alg:main_privacy}, where we switch the order of the two for loops over $t$ and $c$. Apparently, this modified algorithm gives the same outcome as \cref{alg:main_full}. The only lines that touch private data are \cref{line:privacy_hist_for,line:privacy_hist}. 
The input to these lines is the entire private dataset, and the output is a histogram with size $\numgensamples$. Same as step 1 in the analysis of \cref{alg:main} (\cref{sec:privacy_analysis}), each private sample only contributes one vote in the histogram. If we add or remove one sample, the resulting histogram will change by 1 in the $\ell_2$ norm. Therefore, the sensitivity of these lines is 1. The following privacy analysis follows exactly the same as steps 2-5 in \cref{sec:privacy_analysis}, and therefore, \textbf{the privacy guarantee of \cref{alg:main_full} is the same as \cref{alg:main}.} 
It is important to emphasize that even though \cref{alg:main_full}
utilizes the labels of the samples directly, the above analysis means that \textbf{\cref{alg:main_full} provides privacy protection to the label assignment of the samples (i.e., the labels each sample has) in the same way as the DP-SGD-fine-tuneing-based algorithms \cite{ghalebikesabi2023differentially}}.

\input{tex/alg_main_privacy}

%% file: tex/alg_main.tex
\begin{algorithm}[thpb]
    \DontPrintSemicolon
    \LinesNumbered
	\BlankLine
	\SetKwInOut{Input}{Input}
	\caption{\algname{} (\algnameshort{}) for both labeled and unlabeled data.}
    \label{alg:main_full}
	\Input{The set of private classes: $\privatesampleclassset$ ~~~ 
 ($\privatesampleclassset=\brc{0}$ if for unconditional generation)\\
 Private samples: $\privatesampleset=\brc{(x_i, y_i)}_{i=1}^{\numprisamples}$, where $x_i$ is a sample and $y_i\in \privatesampleclassset$ is its label\\
 Number of iterations: $\numiterations$\\
 Number of generated samples: $\numgensamples$~~~(assuming $\numgensamples~\textrm{mod}~\brd{\privatesampleclassset}=0$)\\
 Noise multiplier for \dpvotingname{}: $\noisemultiplier$\\
 Threshold for \dpvotingname{}: $\threshold$
	}
	\BlankLine
        $\generatedsampleset \leftarrow \emptyset$\;
        \For{$c \in \privatesampleclassset$}{
            $private\_samples \leftarrow \brc{x_i| (x_i, y_i)\in \privatesampleset \text{  and  } y_i=c}$ \;
            $S_0 \leftarrow \randomsampleapi{\numgensamples / \brd{\privatesampleclassset}}$ \; 
    	\For{$t \leftarrow 1, \ldots, \numiterations$}
    	{
                $histogram_t \leftarrow \dpvotingfunction{private\_samples,S_{t-1}, \noisemultiplier,\threshold}$  \tcp*{See \cref{alg:voting}} 
                $\calP_t \leftarrow histogram_t/\mathrm{sum}(histogram_t)$ \tcp*{$\calP_t$ is a distribution on $S_t$}
                $S_{t}'\leftarrow $ draw $N_\syn/|C|$ samples with replacement from $\calP_t$  \tcp*{$S_t'$ is a multiset}
                $S_{t}\leftarrow\samplevariationapi{S_t'}$ 
            }
            $\generatedsampleset\leftarrow \generatedsampleset \cup \brc{(x,c)|x\in S_T}$
        }
	\Return{ $\generatedsampleset$}
\end{algorithm}

%% file: tex/alg_main_privacy.tex
\begin{algorithm}[thpb]
    \DontPrintSemicolon
    \LinesNumbered
	\BlankLine
	\SetKwInOut{Input}{Input}
	\caption{\algname{} (\algnameshort{}) for both labeled and unlabeled data. (Modified from \cref{alg:main_full}) for the ease of privacy analysis.)}
    \label{alg:main_privacy}
	\Input{The set of private classes: $\privatesampleclassset$ ~~~ 
 ($\privatesampleclassset=\brc{0}$ if for unconditional generation)\\
 Private samples: $\privatesampleset=\brc{(x_i, y_i)}_{i=1}^{\numprisamples}$, where $x_i$ is a sample and $y_i\in \privatesampleclassset$ is its label\\
 Number of iterations: $\numiterations$\\
 Number of generated samples: $\numgensamples$~~~(assuming $\numgensamples~\textrm{mod}~\brd{\privatesampleclassset}=0$)\\
 Noise multiplier for \dpvotingname{}: $\noisemultiplier$\\
 Threshold for \dpvotingname{}: $\threshold$
	}
	\BlankLine
        $\generatedsampleset \leftarrow \emptyset$\;
        $S_0^c \leftarrow \randomsampleapi{\numgensamples / \brd{\privatesampleclassset}}$ for each $c \in \privatesampleclassset$ \;
        $private\_samples^c \leftarrow \brc{x_i| (x_i, y_i)\in \privatesampleset \text{  and  } y_i=c}$ for each $c \in \privatesampleclassset$\;
        \For{$t \leftarrow 1, \ldots, \numiterations$}{
            \For{$c \in \privatesampleclassset$ \label{line:privacy_hist_for}}
    	{
                $histogram_t^c \leftarrow \dpvotingfunction{private\_samples^c,S_{t-1}^c, \noisemultiplier,\threshold}$  \label{line:privacy_hist}\tcp*{See \cref{alg:voting}} 
            }
            \For{$c \in \privatesampleclassset$}
    	{
                $\calP_t \leftarrow histogram_t^c/\mathrm{sum}(histogram_t^c)$ \tcp*{$\calP_t$ is a distribution on $S_t^c$}
                $S_{t}'\leftarrow $ draw $N_\syn/|C|$ samples with replacement from $\calP_t$  \tcp*{$S_t'$ is a multiset}
                $S_{t}^c\leftarrow\samplevariationapi{S_t'}$ 
            }
        }
        $\generatedsampleset\leftarrow \generatedsampleset \cup \brc{(x,c)|x\in S_T^c, c\in \privatesampleclassset}$\;
	\Return{ $\generatedsampleset$}
\end{algorithm}

%% file: tex/theory.tex
\section{Theoretical Evidence for Convergence of PE}
\label{sec:theory}
In this section, we will give some intuition for why \algnameshort{} can solve \DPWA{}. %

\myparatightestn{Convergence of Non-Private Evolution.}
We %
first analyze \cref{alg:main} when no noise is added to the histograms (i.e., we set $\sigma=0$ and $H=0$ in \cref{line:dp_voting}). We %
show that in this case, the evolution algorithm does %
converge to the private distribution in $O(d)$ iterations where $d$ is the dimension of the embedding space. 
Under some reasonable modeling assumptions (see \cref{app:theory}), we prove the following theorem. Here $D$ is the diameter of $\privatesampleset$, $L \approx$ number of variations of each point in $S_t'$ that are added to $S_{t+1}$ in \cref{line:offspring} of \cref{alg:main}.
\begin{theorem}
	\label{thm:nonprivate_analysis}
	Assume that $\log L \ll d$.\footnote{If $\log L \gg d\log (D/\eta)$, i.e., if we generate an exponential number of points then by a simple epsilon-net argument we can prove that the algorithm will converge in a single step.} With probability $\ge 1-\tau$, the non-private evolution algorithm (\cref{alg:main} with $\sigma=H=0$) outputs $\generatedsampleset$ with Wasserstein distance $W_p(\privatesampleset,\generatedsampleset)\le \eta$ after $T$ iterations\footnote{Number of samples produced using $\samplevariationapiname$ per iteration is $\le L\cdot r\cdot N_\priv = L\log(D/\eta)N_\priv$.} $\forall p\in[1,\infty]$ whenever 
	\begin{align}
		\label{eq:nonprivate_T}
		T\gg \frac{d \log(D/\eta)}{\log L}+\log(N_\priv/\tau).
	\end{align}
\end{theorem}
This theorem is nearly tight. In each iteration of \algnameshort{}, we get the voting information which is about $\tilde{O}(N_\priv\log(L))$ bits. To converge to $S_\priv$, we need at least $\tilde{\Omega}(N_\priv d)$ bits of information. Therefore we do require at least $\tilde{\Omega}(d/\log L)$ iterations to converge.
Here is some intuition for how the proof of \cref{thm:nonprivate_analysis} works. Fix some private point $x\in \privatesampleset$. Let $z^*\in S_t$ be its closest point in $S_t$. In $S_{t+1}$, we generate variations of $z^*$ using the $\samplevariationapi{z^*}$. We then prove that if $\norm{x-z^*}\ge \eta$, then one of the variations will get closer to $x$ than $z^*$ by a factor of $(1-(\log L)/d)$ with constant probability. Repeating this for $T$ iterations as in \cref{eq:nonprivate_T}, will bring some point in $S_T$ $\eta$-close to $x$.

\myparatightestn{Convergence of \algname{}.}
To get some intuition on the working of \algnameshort{} in the presence of noise, we make a simplifying assumption. We will assume that there are $B$ identical copies of each private point in $\privatesampleset$. We call $B$ as multiplicity.
Note that for any DP algorithm to converge to $\privatesampleset$, we need that the private data is well clustered with some minimum cluster size. Any cluster with too few points cannot be represented in $\generatedsampleset$, because that would violate DP. And when there is a cluster of $B$ private points and the generated points in $S_t$ are still far from this cluster, then it is likely that all the $B$ private points will have a common closest point in $S_t$\footnote{In fact, it is easy to formally prove that there will be a common approximate nearest neighbor.}, i.e., they all vote for the same point in $S_t$ as a single entity. Therefore multiplicity is a reasonable modeling assumption to make to understand the working of \algnameshort{}. Note that, actually it is very easy to find $S_\priv$ exactly using DP Set Union~\cite{gopi2020differentially} with the multiplicity assumption. The point of \cref{thm:private_analysis} is to give intuition about why \algnameshort{} works in practice; it is proved in \cref{app:private_analysis_proof} under the same assumptions as in \cref{thm:nonprivate_analysis}. 
\begin{theorem}
	\label{thm:private_analysis}
	Let $0\le \eps\le \log(1/2\delta)$. Suppose each point in $\privatesampleset$ has multiplicity $B$. Then, with high probability ($\ge 1-\tau$), \algname{} (\cref{alg:main}) with $\sigma\gg \sqrt{T\log(1/\delta)}/\eps$ and $H\gg \sigma\sqrt{\log (TLN_\priv/\tau)}$, when run for $T$ iterations, satisfies $(\eps,\delta)$-DP and outputs $\generatedsampleset$ such that $W_p(\privatesampleset,\generatedsampleset)\le \eta,$ $\forall p\in[1,\infty],$ whenever $T$ satisfies \cref{eq:nonprivate_T} and multiplicity $B \gg H$. %
\end{theorem}
Ignoring polylogarithmic factors in $d,L,N_\priv,\log(D/\eta),\tau$, we need $T\gg  d$ and $B\gg \sqrt{d\log(1/\delta)}/\eps$ for Theorem \ref{thm:private_analysis} to hold. Thus, we should expect that \algnameshort{} will discover every cluster of private data of size $\gg \sqrt{d\log(1/\delta)}/\eps$ in $O(d)$ iterations. We now compare this to previous work on DP clustering. \citet{ghazi2020differentially} gives an algorithm for densest ball, where they show an $(\eps,\delta)$-DP algorithm which (approximately) finds any ball of radius $r$ which has at least $\gg \sqrt{d\log(1/\delta)}/\eps$ private points. Thus intuitively, we see that \algnameshort{} compares favorably to SOTA DP clustering algorithms (though we don't have rigorous proof of this fact). If this can be formalized, then \algnameshort{} gives a very different algorithm for densest ball, which in turn can be used to solve DP clustering. Moreover \algnameshort{} is very amenable to parallel and distributed implementations. We therefore think this is an interesting theory problem for future work.

\myparatightestn{Why \algname{} works well in practice.}
\label{sec:intrinsic_dimension}
We have seen that in the worst case, \algnameshort{} takes $\Omega(d)$ iterations to converge. In our experiments with \cifar{} and \camelyon{}, where $d=2048$ is the embedding dimension, we see that \algnameshort{} actually converges in only about 20 iterations which is much smaller than $d$. We offer one plausible explanation for this via \emph{intrinsic dimension}. Suppose the (embeddings of) realistic images lie on a low dimensional manifold $M$ inside $\mathbb{R}^d$ of dimension $d_{\intr}\ll d$ (see experimental results in \cref{app:intrinsic}). 
Given an image $z$, $\samplevariationapi{z}$ will create variations of $z$ which are also realistic images. Therefore the embeddings of these variations will also lie in the same manifold $M$, and \algnameshort{} is searching for the private points only inside the manifold $M$ without ever going outside it. Therefore the $d$ that matters for convergence is actually $d_{\intr}\ll d.$ In this case, we expect that \algnameshort{} converges in $O(d_{\intr})$ iterations and discovers clusters of private points of size at least $\sqrt{d_{\intr}\log(1/\delta)}/\eps$.

%% file: tex/app_theory.tex
\section{Proofs of \algnameshort{} Convergence Theorems}
\label{app:theory}

We will slightly modify the algorithm as necessary to make it convenient for our analysis. %
We will make the following modeling assumptions:
\begin{itemize}
    \item The private dataset $\privatesampleset$ is contained in an $\ell_2$ ball of diameter $D$ and $\randomsampleapiname$ will also produce initial samples in the same ball of diameter $D$. This is a reasonable assumption in practice,  as images always have bounded pixel values: for the original images in UINT8 data type, each pixel is in the range of $[0, 255]$; in diffusion models, they are usually normalized to $[-1, 1]$ (i.e., 0 corresponds to -1 and 255 corresponds to 1), and all generated images are guaranteed to be in this range.
    \item The distance function used in \cref{alg:voting} is just the $\ell_2$ norm, $d(x,z)=\norm{x-z}_2$. 
    \item The distribution of points we output is $S_\syn=\calP_T$ (i.e., we output a distribution of points). 
    \item $S_{t}'\supset \supp(\calP_t)=\supp(histogram_t)$.\footnote{This may not be true in the original algorithm due to sampling, we need to modify it so that $S_t'\supset \supp(\calP_t).$} %
    \item $S_{t+1}=S_t' \cup \bigcup_{z\in S_t'}\samplevariationapi{z}$ where $\samplevariationapi{z}$ samples $L$ samples each from Gaussian distributions $\calN(z,\sigma_i^2 I)$ for $\sigma_i=\frac{D\sqrt{\log L}}{2^i d}$ where $1\le i\le r$ and $r=\log(D/\eta)$ where $\eta>0$ is the final Wasserstein distance.
\end{itemize}

\subsection{Proof of \cref{thm:nonprivate_analysis}}
\label{app:nonprivate_analysis_proof}
\begin{proof}
Fix a point $x\in \privatesampleset$ and some iteration $t$. Suppose $z^*\in S_t$ is the closest point to $x$. Since $x$ will vote for $z^*$ in $histogram_t$, $z^*\in\supp(\calP_t)\subset S_t'$. Therefore $\samplevariationapi{z^*}\subset S_{t+1}$. Let $V=\samplevariationapi{z^*}$.
\begin{claim}
If $\norm{x-z^*}\ge \eta$, then with probability at least $1/2$, some point in $V$ will get noticeably closer to $x$ than $z^*$, i.e., $$\min_{z\in V}\norm{x-z}_2 \le \bra{1-\frac{\log L}{4d}}\norm{x-z^*}_2.$$  
\end{claim}
\begin{proof}
    Let $s=\norm{x-z^*}$ and let $\sigma\in \{\sigma_1,\sigma_2,\dots,\sigma_r\}$ be such that $\sigma d/\sqrt{\log L} \in [s/2,s]$. Note that such a $\sigma$ exists since $s\in [\eta,D]$. We will now prove that one of the $L$ samples $z_1,z_2,\dots,z_L \sim \calN(z^*,\sigma^2 I_d)$ will get noticeably closer to $x$ than $z^*.$ Let $z_i=z^*+\sigma w_i$ where $w_i\sim \calN(0,I_d).$    
    \begin{align*}
    \min_{i\in [L]} \norm{x-z_i}_2^2&= \norm{x-z^*}_2^2 + \min_{i\in [L]} \bra{\sigma^2 \norm{w_i}_2^2 - 2\sigma \inpro{x-z^*}{w_i}}\\
    &\le s^2 + \max_{i\in [L]}  \sigma^2 \norm{w_i}_2^2 - \max_{i\in [L]}2\sigma \inpro{x-z^*}{w_i}
    \end{align*}
    Note that $\norm{w_i}_2^2$ is a $\chi_d^2$ random variable. By using upper tail bounds for $\chi_d^2$ distribution and union bound over $i\in [L]$, we can bound $$\Pr\brb{\max_{i\in [L]} \norm{w_i}_2^2 \ge 3d/2} \le L \exp(-\Omega(d))\ll 1.$$
    The distribution of $\inpro{x-z^*}{w_i}$ is the same as $\norm{x-z^*}_2 \tw_i$ where $\tw_1,\dots,\tw_L$ are i.i.d. $\mathcal{N}(0,1)$ random variables. By using the fact that max of $L$ i.i.d. Gaussians is at least $\sqrt{\log L}$ with probability at least $3/4$ (for $L\gg 1$), we get that  $$\Pr\brb{\max_{i\in [L]} \inpro{x-z^*}{w_i} \le s\sqrt{\log L}}\le \frac{1}{4}.$$
    Combining everything we get:
    \begin{align*}
    \frac{1}{2}&\ge\Pr\brb{\min_{i\in [L]} \norm{x-z_i}_2^2\ge s^2 +(3/2)d\sigma^2-2 s \sigma \sqrt{\log L}}\\
    &\ge\Pr\brb{\min_{i\in [L]} \norm{x-z_i}_2^2\ge \max_{\lambda \in [1/2,1]} s^2 +(3/2)d\bra{\frac{\lambda s \sqrt{\log L}}{d}}^2-2 s \bra{\frac{\lambda s \sqrt{\log L}}{d}} \sqrt{\log L}}\\
    &\ge\Pr\brb{\min_{i\in [L]} \norm{x-z_i}_2^2\ge  s^2 +\bra{\frac{s^2 \log L}{d}}\max_{\lambda \in [1/2,1]}(3\lambda^2/2-2\lambda)}\\
    &\ge\Pr\brb{\min_{i\in [L]} \norm{x-z_i}_2^2\ge  s^2\bra{1-\frac{\log L}{2d}}}\\
    &\ge\Pr\brb{\min_{i\in [L]} \norm{x-z_i}_2\ge  s\bra{1-\frac{\log L}{4d}}}
    \end{align*}
    where last inequality uses the fact that $\sqrt{1-t}\le 1-\frac{t}{2}$ for $t\le 1$.
\end{proof}

Now in $T$ iterations, $\min_{z\in S_t}\norm{x-z}_2$ will shrink by a factor of $\bra{1-\frac{\log L}{4d}}$ in at least $T/4$ iterations with probability $1-\exp(-\Omega(T)) \ge 1 - \frac{\tau}{N_\priv}$ (by standard Chernoff bounds). Note that in iterations where it doesn't shrink, it doesn't grow either since $S_t'\subset S_{t+1}.$ Similarly, if $\min_{z\in S_t} \norm{x-z}_2 \le \eta$ for some iteration, it will remain so in all subsequent iterations. Therefore after $T \gg \frac{d\log(D/\eta)}{\log L}$ iterations, $\min_{z\in S_T} \norm{x-z}_2 \le \eta$ with probability at least $1-\frac{\tau}{N_\priv}$. By union bounding over all points we get that, with probability at least $1-\tau$, for every point $x\in \privatesampleset$ there is a point in $S_T$ which is $\eta$-close. This proves that $W_p(\privatesampleset,\calP_T)\le \eta
.$
\end{proof}

\subsection{Proof of \cref{thm:private_analysis}}
\label{app:private_analysis_proof}
\begin{proof}
    Since we are doing $T$ iterations of Gaussian mechanism with noise level $\sigma$, we need to set $\sigma \gg \sqrt{T\log(1/\delta)}/\eps$ to satisfy $(\eps,\delta)$-DP~\cite{dwork2014algorithmic} when $\eps \le log(1/2\delta)$. Let $x\in \privatesampleset$ be a point with multiplicity $B$. If $z^*\in S_t$ is the closest point to $x$, then it will get $B$ votes. After adding $\calN(0,\sigma^2)$ noise, if $B\gg H \gg \sigma \sqrt{\log (TLN_\priv/\tau)}$, then with probability at least $1-\tau/(4T)$, the noisy votes that $z^*$ gets is still above the threshold $H$. Therefore $z^*$ will survive in $S_{t+1}$ as well. Also since $H\gg \sigma \sqrt{\log (TLN_\priv/\tau)}$, with probability $1-\tau/(4T)$, points in $S_t$ which do not get any votes (there are $LN_\priv$ of them) will not survive even after adding noise and thresholding by $H$. Therefore, by union bounding over all $T$ iterations, with probability at least $1-\tau/2$, the algorithm behaves identically to the non-private algorithm. Therefore by an identical proof as in the non-private analysis, we can prove that after $T$ iterations $W_p(\privatesampleset,\calP_T)\le \eta$ with probability at least $1-\tau$. %
\end{proof}

%% file: tex/app_intrinsic_dimension.tex
\section{Intrinsic Dimension of Image Embeddings}
\label{app:intrinsic}

To illustrate the intrinsic dimension of image embeddings, we use the following process:
\begin{packedenumerate}
    \item We (randomly) take an image $x$ from \cifar{}.
    \item We use \samplevariationapiname{} from \cref{app:cifar} to obtain 3000 image variations of $x$: $x_1,\ldots,x_{3000}$, and their corresponding inception embeddings $g_1,...,g_{3000}\in\real^{2048}$. 3000 is chosen so that the number of variations is larger than the embedding dimension.
    \item We construct a matrix $M=[g_1-g;\ldots;g_{3000}-g]\in \real^{3000\times 2048}$, where $g$ is the mean$(g_1,\ldots,g_{3000})$.
    \item We compute the singular values of $M$: $\sigma_1\geq \sigma_2\geq \ldots \geq \sigma_{2048}$. 
    \item We compute the minimum number of singular values $n$ needed so that the explained variance ratio\footnote{See \url{https://scikit-learn.org/stable/modules/generated/sklearn.decomposition.TruncatedSVD.html}.} $\nicefrac{\sum_{i=1}^{n}\sigma_i^2}{\sum_{i=1}^{2048}\sigma^2} \geq 0.8$. Intuitively, this $n$ describes how many dimensions are needed to reconstruct the embedding changes $M$ with a small error. We use it as an estimated intrinsic dimension of the image variations.
\end{packedenumerate}

We conduct the above process with the variation degree $[98,  96,  94,  92,  90,  88,  86,  84,  82,  80,  78,  76, 74,  72,  70,  68,  66,  64,  62,  60]$ utilized in the \cifar{} experiments (see \cref{app:cifar}). We additionally add a variation degree of 100 which is the highest variation degree in the API that was used to generate the initial samples. We plot the estimated intrinsic dimension v.s. variation degree in \cref{fig:intrinsic}. The raw original singular values of $M/\sqrt{3000}$ for variation degree=60 are in \cref{fig:singular_values} (other variation degrees have similar trend). Two key observations are: 
\begin{packeditemize}
    \item As the variation degree increases, the estimated intrinsic dimension also increases. This could be because the manifold of image embeddings is likely to be non-linear, the above estimation of intrinsic dimension is only accurate when we perturb the image $x$ to a small degree so that the changes in the manifold can still be well approximated by a linear subspace. Using a larger variation degree (and thus larger changes in the embedding space) will overestimate the intrinsic dimension. %
    \item Nevertheless, we always see that the singular values decrease rapidly (\cref{fig:singular_values}) and the estimated intrinsic dimension is much smaller than the embedding size 2048 (\cref{fig:intrinsic}), which supports our hypothesis in \cref{sec:intrinsic_dimension}.
\end{packeditemize}

\begin{figure}[ht]
    \centering
    \includegraphics[width=0.5\linewidth]{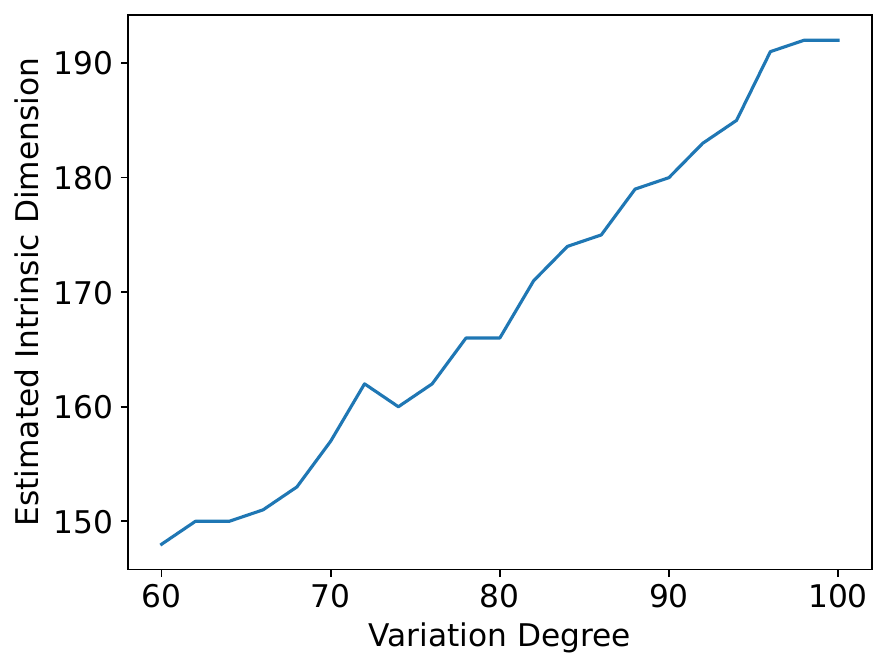}
    \caption{Estimated intrinsic dimension of inception embeddings of realistic images.}
    \label{fig:intrinsic}
\end{figure}

\begin{figure}[ht]
    \centering
    \includegraphics[width=0.5\linewidth]{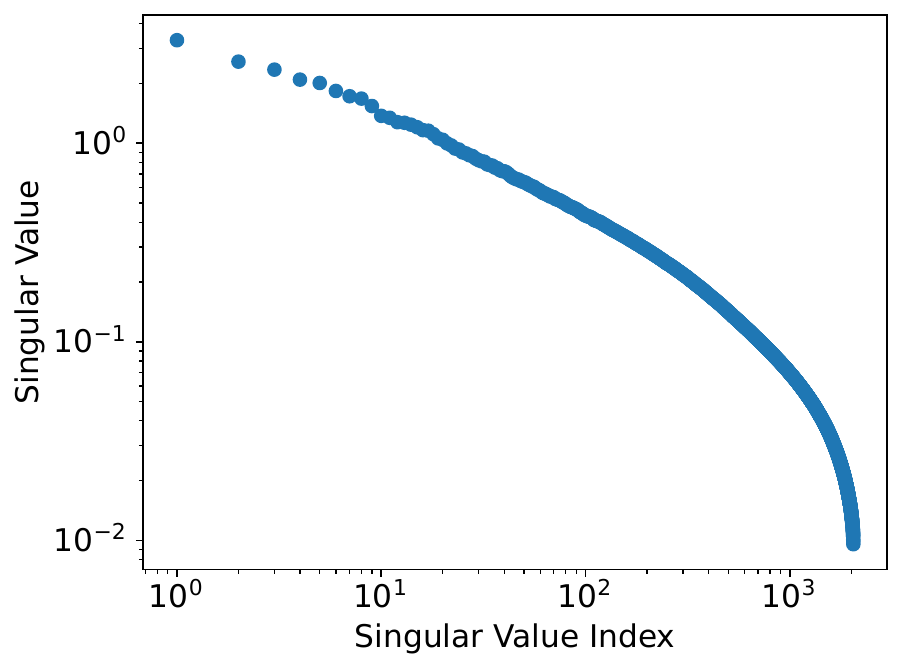}
    \caption{Singular values of inception embeddings of image variations at variation degree=60.}
    \label{fig:singular_values}
\end{figure}

%% file: tex/app_dpclustering.tex
\section{Relation of \DPWA{} to prior work}
\label{app:dpclustering}

Recall that in \DPWA{}, we want a DP algorithm to output $\generatedsampleset$ which is close to the distribution of $\privatesampleset$ in Wasserstein distance w.r.t. some distance function $d(\cdot,\cdot)$. When the distance function $d(\cdot,\cdot)$ is just $\ell_2$ distance between the samples (i.e., $\ell_2$ in the pixel space for images), then \DPWA{} is closely related to DP Clustering~\cite{ghazi2020differentially,balcan2017differentially,su2016differentially} and DP Heatmaps~\cite{ghazi2022differentially}.

In~\cite{ghazi2022differentially}, to give an algorithm for DP Heatmaps, the authors study DP sparse EMD\footnote{Earth's Mover Distance, which is the another name for Wasserstein metric $W_1$.} aggregation problem where we need to output a distribution of points which approximates the distribution of private data in EMD distance (i.e., $W_1$). They study this problem only in two dimensions and the running time of their algorithms (suitably generalized to higher dimensions) will be exponential in the dimension $d$.

The DP Clustering problem requires us to output a clustering of private data using DP. The most common clustering studied is $k$-means clustering where we should output $k$ cluster centers such that $k$-means cost is minimized, where $k$-means cost is the sum of squares of $\ell_2$-distance of each data point to its nearest cluster center. Note that in DPWA, if the number of synthetic data points $N_\syn$ is specified to be $k$, then DP $k$-means clustering and \DPWA{} with $W_2$ metric are equivalent. In~\cite{ghazi2022differentially}, a polynomial time DP Clustering algorithm with an additional $k$-means cost (over what is non-privately possible) of $k\sqrt{d\log(1/\delta)}\polylog(N_\priv,d)/\epsilon$ is given. This can be converted into an upper bound on the Wasserstein distance. But this is not a practical algorithm. The privacy-utility tradeoffs are bad due to the large hidden constants in the analysis and the authors don't provide an implementation. There is a practical DP Clustering algorithm (along with an implementation) given in~\cite{chang2023differentially} (but with no theoretical guarantees).

\subsection{Why Not Just Use DP Clustering?}
We now explain why we can't just use prior work on DP Clustering to solve \probnameshort{} say for images.

\myparatightestn{Clustering in the image space.} We can use DP $k$-means Clustering to cluster the images w.r.t. $\ell_2$ metric in the pixel space. This doesn't work because $\ell_2$ distance in the pixel space doesn't capture semantic similarity. An image which is slightly shifted in pixel space gets very far in $\ell_2$ distance. And the dimension of the images is too large for prior DP Clustering algorithms to work well. Their convergence and privacy-utility tradeoffs depend too strongly on the dimension.

\myparatightestn{Clustering in the embedding space.} We can use DP $k$-means Clustering to cluster the image embeddings w.r.t. $\ell_2$ metric in the embedding space. Note that this is the distance function we use in \algnameshort{} (\cref{eq:distance_no_packing}). Even after we find the cluster centers, it is hard to invert the embedding map (i.e., find an image whose embedding is close to a given vector in the embedding space).\footnote{Some special embeddings such as CLIP embedding do have such an inverse map called unCLIP \cite{ramesh2022hierarchical}.} Moreover the dimension of the embedding space is still too large for the above methods to be practical.

Our \algnameshort{} algorithm does much better because:
\begin{enumerate}
    \item Its distance function is $\ell_2$ in the embedding space which captures semantic similarity,
    \item It exploits the intrinsic dimension of the manifold of images in the embedding space which is much smaller than the embedding dimension (see \cref{sec:intrinsic_dimension} and \cref{app:intrinsic}) and
    \item There is no need to invert points in embedding space to the image space.
\end{enumerate}

 In an early experiment, we have tried DP clustering in the CLIP embedding space using the practical DP Clustering algorithm in~\cite{chang2023differentially}. We then inverted the cluster centers (which are in the embedding space) using unCLIP. But we found the resulting images are too noisy compared to the images we get from \algnameshort{} and the FID scores are also significantly worse than that of \algnameshort{}.

%% file: tex/app_conditional_unconditional.tex
\section{Implementation Details on Label Condition}
\label{app:conditional_unconditional}

There are two meanings of ``conditioning'' that appear in our work:
\begin{packedenumerate}
    \item Whether the pre-trained networks or APIs (e.g., \imagenet{} pre-trained diffusion models used in \cref{sec:exp_cifar,sec:exp_camelyon}) support conditional input (e.g., \imagenet{} class label).
    \item Whether the generated samples are associated with class labels from the private data.
\end{packedenumerate}

In DP fine-tuning approaches, these two usually refer to the same thing: if we want to generate class labels for generated samples, the common practice is to use a pre-trained network that supports conditional input \cite{ghalebikesabi2023differentially}. However, in \algnameshort{}, these two are completely orthogonal. 

\myparatightestn{Conditional pre-trained networks/APIs.} We first explain our implementation when the pre-trained networks or APIs support conditional inputs such as class labels or text prompts.
When generating the initial population using \randomsampleapiname{} (\cref{line:initial}), we will either randomly draw labels from all possible labels when no prior public information is available (which is what we do in \cifar{} and \camelyon{} experiments where we randomly draw from all possible \imagenet{} classes), or use the public information as condition input (e.g., the text prompt used in Stable Diffusion experiments; see \cref{app:stable_diffusion}). 
In the subsequent \samplevariationapiname{} calls (\cref{line:offspring}), for each image, we will use its associated class label or text prompt as the condition information to the API, and  
the output samples from \samplevariationapiname{} will be associated with the same class label or text prompt as the input sample. For example, if we use an image with ``peacock'' class to generate variations, all output images will be associated with ``peacock'' class for future \samplevariationapiname{} calls.
Note that throughout the above process, all the condition inputs to the pre-trained networks/APIs are public information; they have nothing to do with the private classes.

\myparatightestn{Conditional generation.} Conditional generation is achieved by \cref{alg:main_full}, where we separate the samples according to their class labels, and run the main algorithm (\cref{alg:main}) on each sample set. We can use either conditional or unconditional pre-trained networks/APIs to implement it.

\smallskip
Throughout the paper, ``(un)condition'' refers to 2, expect the caption in \cref{fig:cifar_ablation_pretrained_network} which refers to 1.

%% file: tex/app_cifar10.tex
\section{More Details and Results on \cifar{} Experiments}
\label{app:cifar}

\myparatightestn{Pre-trained model.} By default, we use the checkpoint \codeword{imagenet64_cond_270M_250K.pt} released in \cite{nichol2021improved}.\footnote{\url{https://github.com/openai/improved-diffusion}} For the ablation study of the pre-trained network, we additionally use the checkpoint \codeword{imagenet64_uncond_100M_1500K.pt}.

\myparatightestn{API implementation.} \randomsampleapiname{} follows the standard diffusion model sampling process. \samplevariationapiname{} is implemented with SDEdit \cite{meng2021sdedit}, which adds noise to input images and lets the diffusion model denoise them. We use DDIM sampler \cite{song2020denoising} and the default noise schedule to draw samples. Note that these choices are not optimal; our results can potentially be improved by using better noise schedules and the full DDPM sampling \cite{ho2020denoising} which are known to work better.
The implementation of the above APIs is straightforward without touching the core modeling part of diffusion models and is similar to the standard API implementations in Stable Diffusion (\cref{app:stable_diffusion}).

\myparatightestn{Hyperparameters.} 
We set the maximum number of iterations $\numiterations=20$, \packing{} degree $\packingdegree=8$, and number of generated samples $\numgensamples=50000$.
For \randomsampleapiname{} and \samplevariationapiname{}, we use DDIM sampler with 100 steps. 
For \samplevariationapiname{}, we use SDEEdit \cite{meng2021sdedit} by adding noise till $[98,  96,  94,  92,  90,  88,  86,  84,  82,  80,  78,  76, 74,  72,  70,  68,  66,  64,  62,  60]$ timesteps for each iteration respectively.  These timesteps can be regarded as the $\variationdegree$ parameter in \cref{sec:scope}.

For the experiments in \cref{fig:cifar_fid_epsilon}, we use noise multiplier $\noisemultiplier=t\cdot\sqrt{2}$ and threshold $\threshold=2t$ for $t\in\brc{5, 10, 20}$, and pick the the pareto frontier. \cref{fig:cifar_fid_epsilon_all} shows all the data points we got. Combining this figure with \cref{fig:cifar_fid_epsilon}, we can see that \algnameshort{} is not very sensitive to these hyper-parameters, and even with less optimal choices \algnameshort{} still outperforms the baselines.

\begin{figure}[h]
    \centering
    \includegraphics[width=0.7\linewidth]{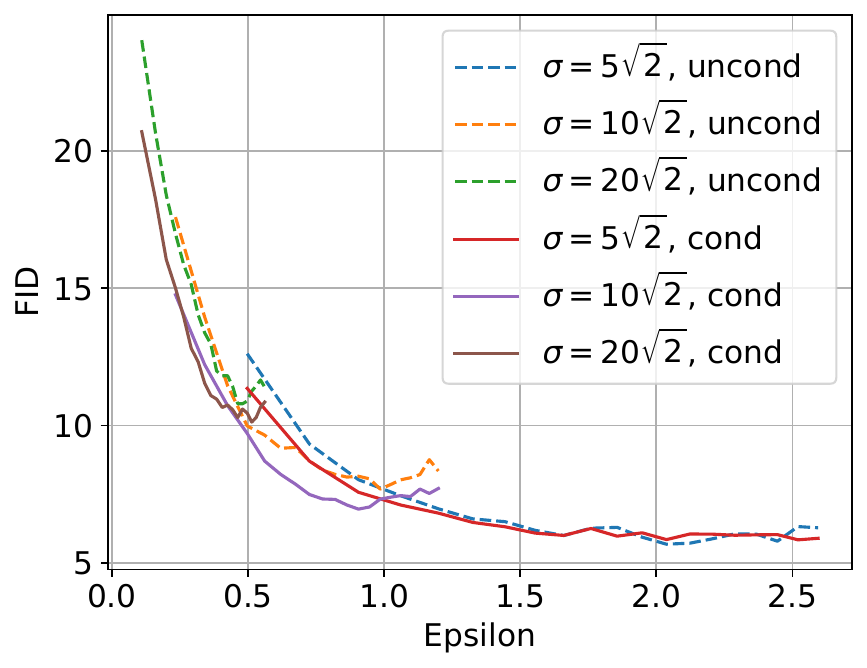}
    \caption{FID (lower is better) v.s. privacy cost $\epsilon$ ($\delta=10^{-5}$) on \cifar{} with different noise multipliers and thresholds. ``(Un)cond'' means (un)conditional generation.}
    \label{fig:cifar_fid_epsilon_all}
\end{figure}

For the experiments in \cref{fig:cifar_acc_numsamples_more_runs}, we use noise multiplier $\noisemultiplier=2\sqrt{2}$, threshold $\threshold=4$ \revision{(i.e., $t=2$)}, and $T=5$.

\revision{For the experiments in \cref{fig:cifar_gen_samples,fig:cifar_gen_samples_app}, we use noise multiplier $\noisemultiplier=10\sqrt{2}$ and threshold $\threshold=20$ (i.e., $t=10$) and the number of PE iteration is 5 (i.e., the point of ``Ours (cond)'' in \cref{fig:cifar_fid_epsilon} that has FID$\leq$ 7.9).} 

For downstream classification (\cref{fig:cifar_acc_numsamples_more_runs}), we follow \cite{ghalebikesabi2023differentially} to use 
WRN-40-4 classifier \cite{zagoruyko2016wide}. We use the official repo\footnote{\url{https://github.com/szagoruyko/wide-residual-networks/tree/master/pytorch}} without changing any hyper-parameter except adding color jitter augmentation according to \cite{ghalebikesabi2023differentially}.
The ensemble of the classifier is implemented by ensembling the logits.

\myparatightestn{FID evaluation.}
Compared to \cref{fig:cifar_fid_epsilon} in the main text, \cref{fig:cifar_fid_epsilon_app} shows the full results of two versions of \cite{harder2022differentially}.
Baseline results are taken from \cite{harder2022differentially,ghalebikesabi2023differentially}.

    \begin{figure}[h]
      \centering      \includegraphics[width=0.5\linewidth]{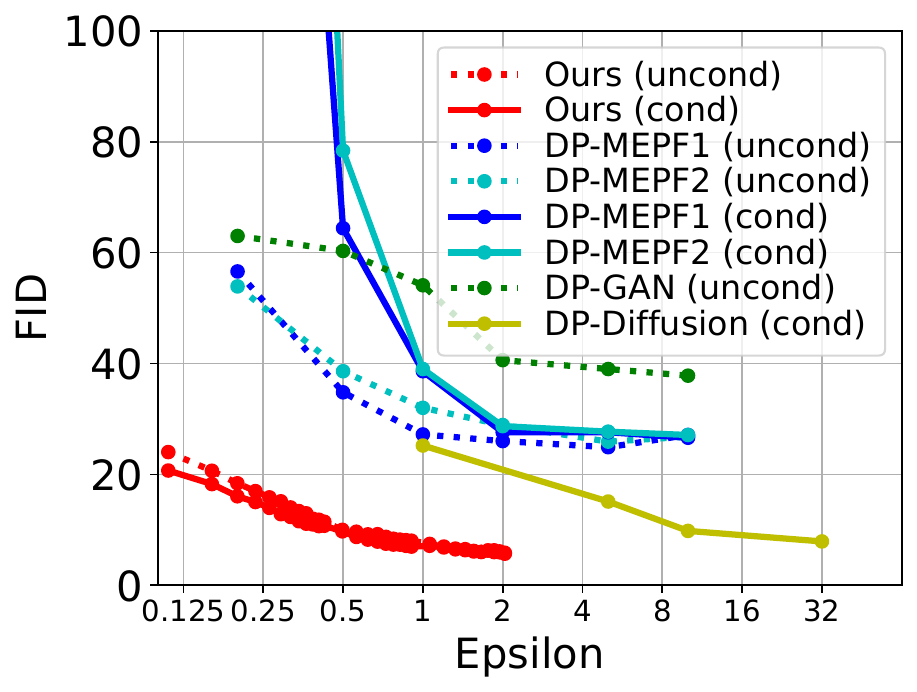}
      \vspace{-0.3cm}
        \caption{FID \cite{heusel2017gans} (lower is better) v.s. privacy cost $\epsilon$ on \cifar{} ($\delta=10^{-5}$). Baseline results are taken from \cite{harder2022differentially,ghalebikesabi2023differentially}. (Un)cond means (un)conditional generation. Ours achieves the best privacy-quality trade-off.}
        \label{fig:cifar_fid_epsilon_app}
    \end{figure}

\myparatightestn{Classification accuracy.}
In \cref{sec:exp_cifar}, we show the classification accuracy at a single privacy budget $\epsilon=3.34$. In \cref{tab:cifar_acc_epsilon_app}, we further show how the classification accuracy evolves with respect to different $\epsilon$s. These results are from the first 5 \algnameshort{} iterations.

\begin{table}[ht]
    \centering
    \begin{tabular}{c|c}
        \toprule
        $\epsilon$ & Accuracy \\\midrule
        1.36 & 72.46\%\\
        1.99 & 78.78\%\\
        2.50 & 80.83\%\\
        2.94 & 81.15\%\\
        3.34 & 81.74\%\\\bottomrule
    \end{tabular}
    \caption{Classification accuracy v.s. privacy cost $\epsilon$ on \cifar{} ($\delta=10^{-5}$).}
    \label{tab:cifar_acc_epsilon_app}
\end{table}

\myparatightestn{Generated samples.}
See \cref{fig:cifar_gen_samples_app,fig:cifar_real} for generated images and real images side-by-side. Note that the pre-trained model we use generates 64x64 images, whereas \cifar{} is 32x32. In \cref{fig:cifar_gen_samples}, %
we show the raw generated 64x64 images; in \cref{fig:cifar_gen_samples_app}, we scale them down to 32x32 for better comparison with the real images.

\begin{figure}[ht]
    \begin{minipage}{.45\textwidth}
	 	\centering
		\includegraphics[width=\linewidth]{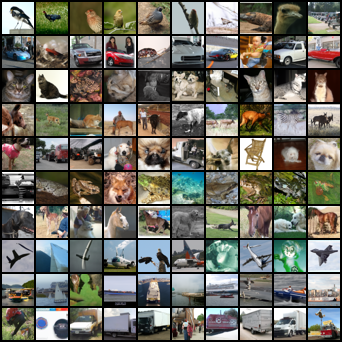}
        \caption{Generated samples on \cifar{} with $\bra{0.67,10^{-5}}$-DP. Each row corresponds to one class. FID=7.87.}
        \label{fig:cifar_gen_samples_app}
    \end{minipage}%
    \hfill
    \begin{minipage}{.45\textwidth}
      	\centering
		\includegraphics[width=\linewidth]{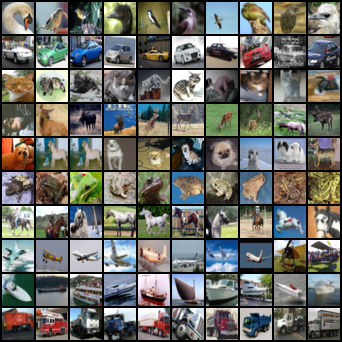}
		\caption{Real samples from \cifar{}. Each row corresponds to one class.\\}
		\label{fig:cifar_real}
    \end{minipage}
\end{figure}

\myparatightestn{Nearest samples in the private dataset}
\cref{fig:cifar_nearest_neighbor_inception,fig:cifar_nearest_neighbor_original} show generated images and their nearest neighbors in the private dataset evaluated using two distance metrics: $\ell_2$ distance in the inception embedding space and the pixel space. We can see that the generated images are different from private images. This is expected due to the DP guarantee.
\begin{figure}[h]
    \centering
    \includegraphics[width=0.55\linewidth]{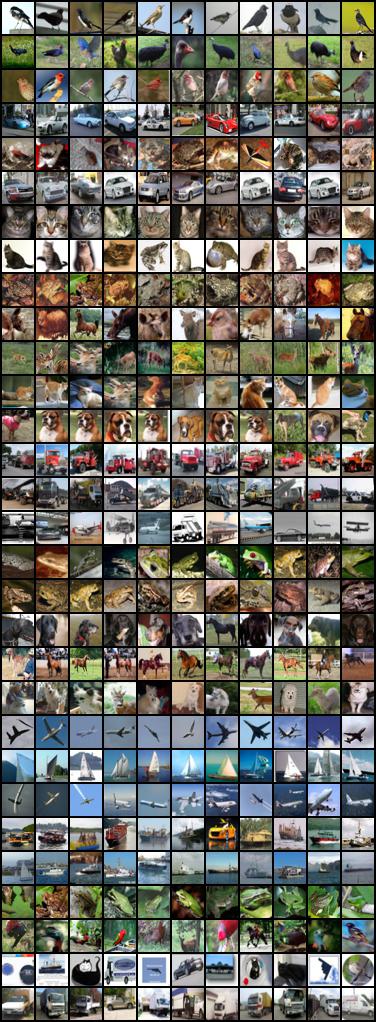}
    \caption{Nearest samples in the private dataset on \cifar{}. In each row, the first column is a generated image (from \cref{fig:cifar_gen_samples_app}), and the other columns are its nearest neighbors in the private dataset, sorted by the distance in ascending order. Every three rows correspond to generated image from one class. The distance metric is $\ell_2$ in the \textbf{inception embedding space.} }
    \label{fig:cifar_nearest_neighbor_inception}
\end{figure}

\begin{figure}[h]
    \centering
    \includegraphics[width=0.55\linewidth]{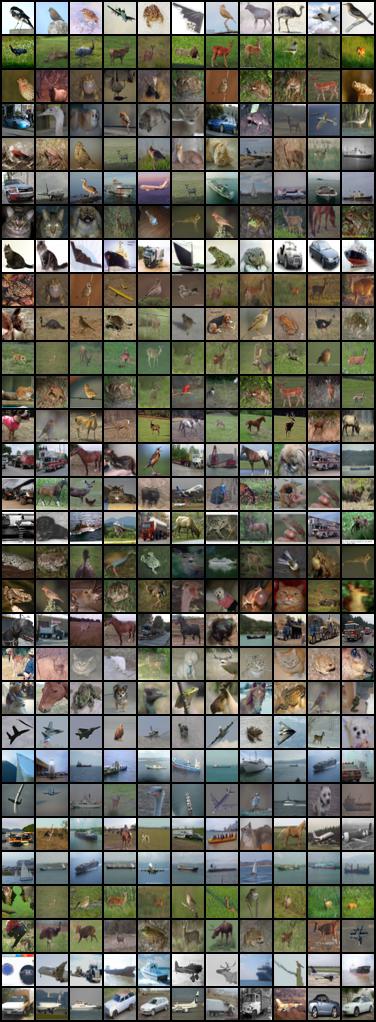}
    \caption{Nearest samples in the private dataset on \cifar{}. In each row, the first column is a generated image (from \cref{fig:cifar_gen_samples_app}), and the other columns are its nearest neighbors in the private dataset, sorted by the distance in ascending order. Every three rows correspond to generated image from one class. The distance metric is $\ell_2$ in the \textbf{pixel space.} }
    \label{fig:cifar_nearest_neighbor_original}
\end{figure}

\revision{\myparatightestn{Distributions of the distances to nearest samples.} Continuing the above experiments, we further show the distribution of the distances between (1) generated samples and their nearest real samples and (2) real samples and their nearest generated samples in \cref{fig:nn_cdf_cifar10,fig:inverse_nn_cdf_cifar10} respectively. Two key observations are: (1) During the early \algnameshort{} iterations, the distances tend to decrease. This means that \algnameshort{} is effective in pushing the generated distribution to be closer to the private distribution. (2) However, as \algnameshort{} continues, the distances stop decreasing. It is expected, as DP upper bounds the probability of reconstructing any sample in the private dataset.}

\begin{figure}[ht]
    \centering
    \begin{subfigure}[b]{0.4\linewidth}
        \centering
        \includegraphics[width=1\linewidth]{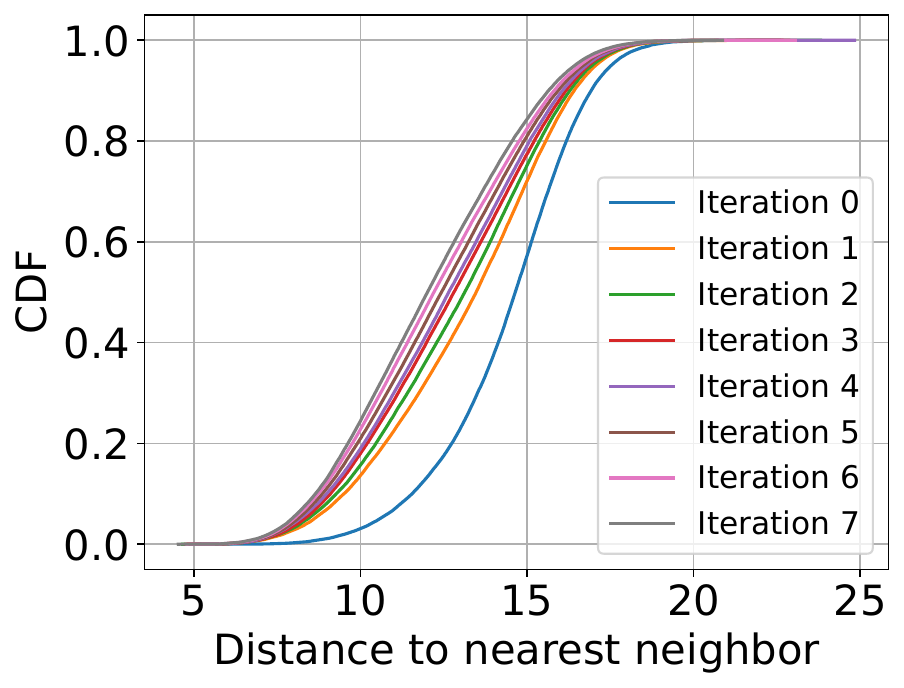}
        \caption{\revision{The distance metric is $\ell_2$ in the \textbf{inception embedding space.}}}
    \end{subfigure}
    ~~~~~
    \begin{subfigure}[b]{0.4\linewidth}
        \centering
        \includegraphics[width=1\linewidth]{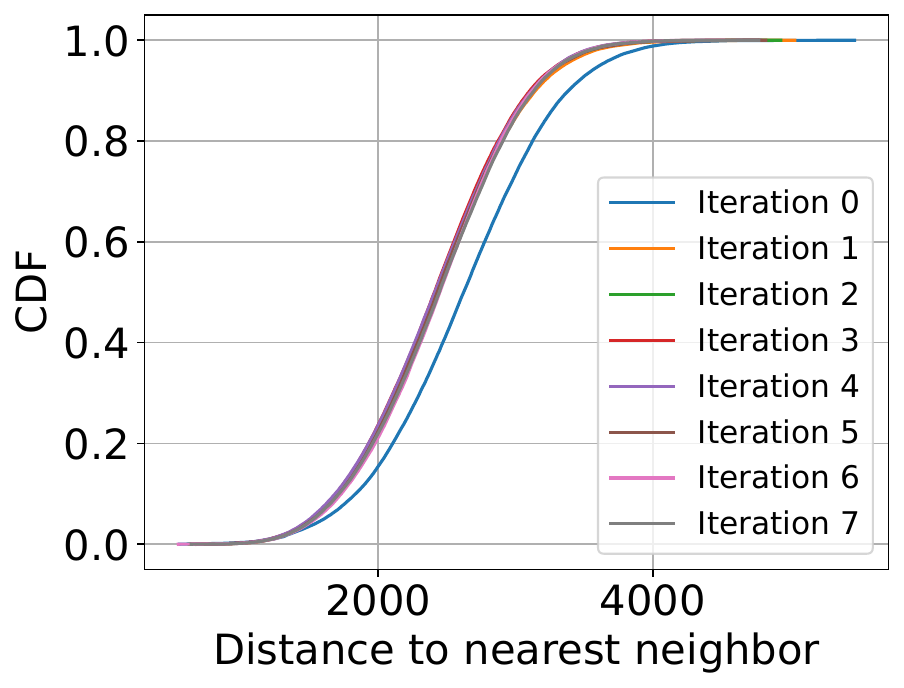}
        \caption{\revision{The distance metric is $\ell_2$ in the \textbf{pixel space.}}}
    \end{subfigure}
    \caption{\revision{CDF of the distributions \textbf{between each generated sample and its nearest private samples} on \cifar{} across different \algnameshort{} iterations. ``Iteration 0'' refers to the initial random samples from \cref{line:initial} in \cref{alg:main}.}}
    \label{fig:nn_cdf_cifar10}
\end{figure}
\begin{figure}[ht]
    \centering
    \begin{subfigure}[b]{0.4\linewidth}
        \centering
        \includegraphics[width=1\linewidth]{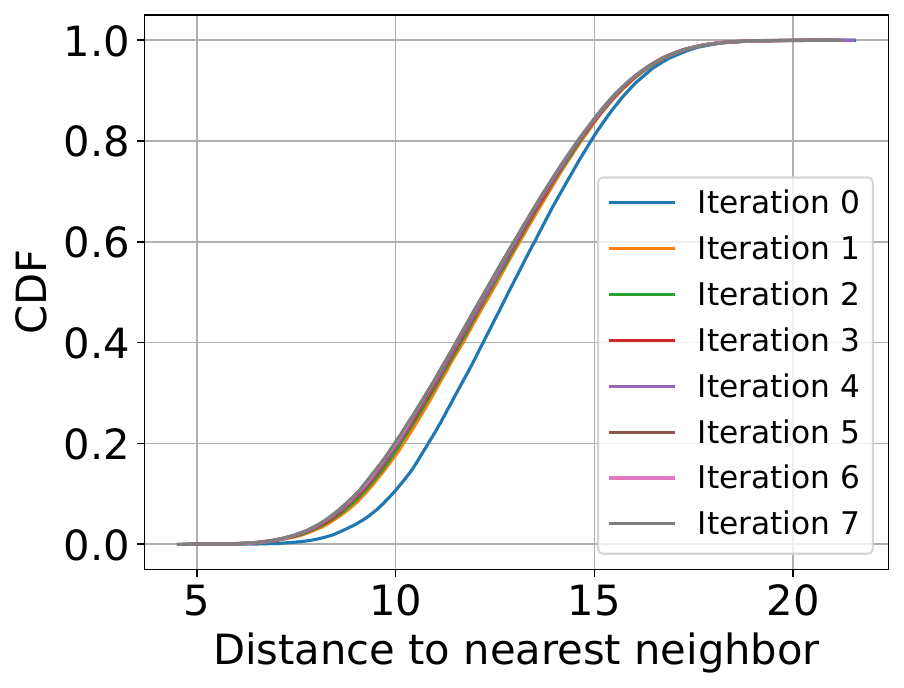}
        \caption{\revision{The distance metric is $\ell_2$ in the \textbf{inception embedding space.}}}
    \end{subfigure}
    ~~~~~
    \begin{subfigure}[b]{0.4\linewidth}
        \centering
        \includegraphics[width=1\linewidth]{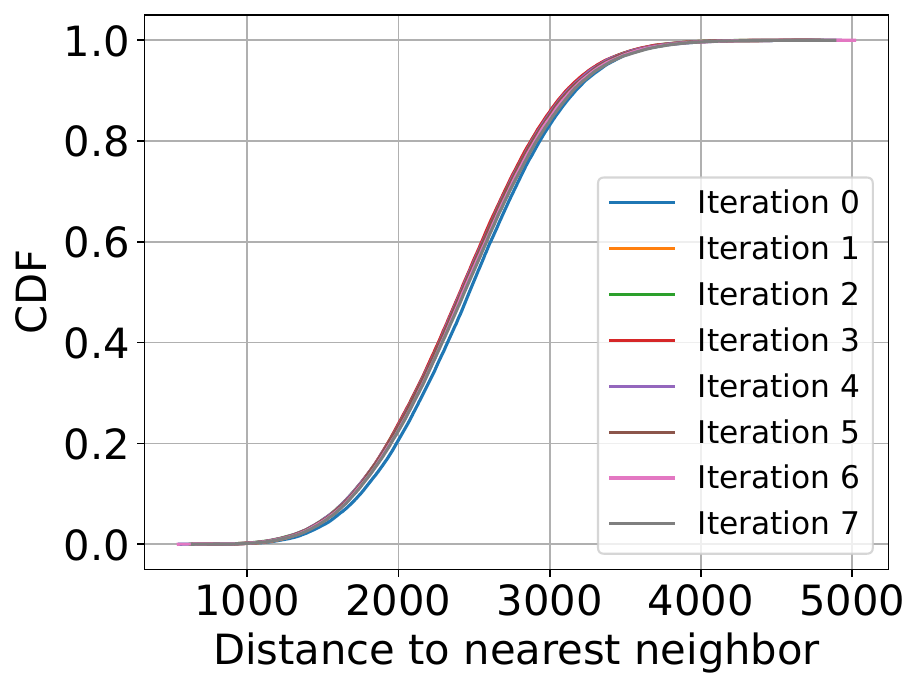}
        \caption{\revision{The distance metric is $\ell_2$ in the \textbf{pixel space.}}}
    \end{subfigure}
    \caption{\revision{CDF of the distributions \textbf{between each private sample and its nearest generated samples} on \cifar{} across different \algnameshort{} iterations. ``Iteration 0'' refers to the initial random samples from \cref{line:initial} in \cref{alg:main}.}}
    \label{fig:inverse_nn_cdf_cifar10}
\end{figure}

\revision{\myparatightestn{Samples with the highest and the lowest votes.}
\cref{fig:cifar10-top-bottom-count} shows the samples with the highest and the lowest votes in \dpvotingname{} across different \algnameshort{} iterations.
We can see that \dpvotingname{} picks samples similar to the private data as desired. This is more obvious in the first two \algnameshort{} iterations, where \dpvotingname{} assigns high votes on the samples with correct classes and puts low votes on the samples with incorrect classes. 
} 

\begin{figure}
    \centering
    \begin{subfigure}[b]{0.4\linewidth}
        \centering
        \includegraphics[width=0.45\linewidth]{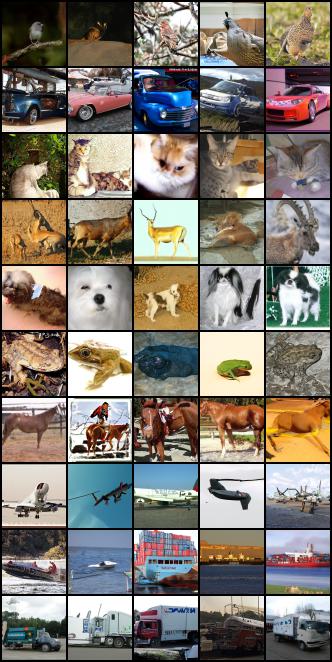}
        \includegraphics[width=0.45\linewidth]{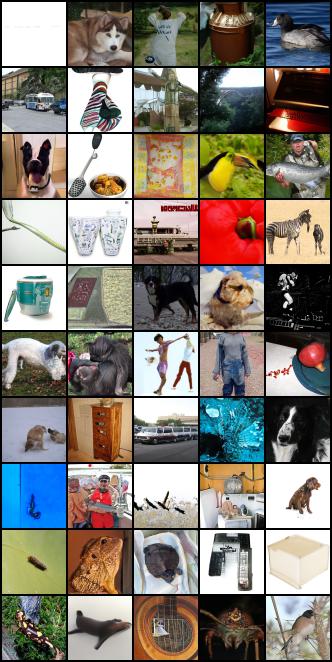}
        \caption{\revision{Iteration=0.}}
    \end{subfigure}
    \begin{subfigure}[b]{0.4\linewidth}
        \centering
        \includegraphics[width=0.45\linewidth]{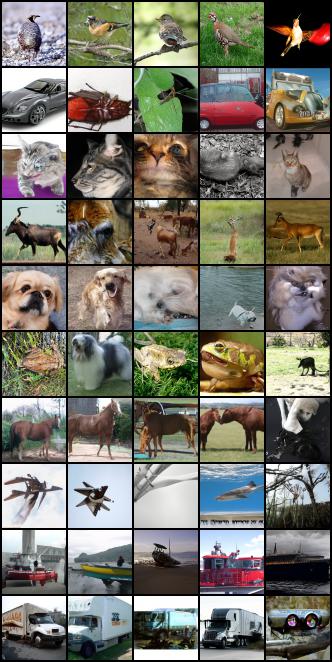}
        \includegraphics[width=0.45\linewidth]{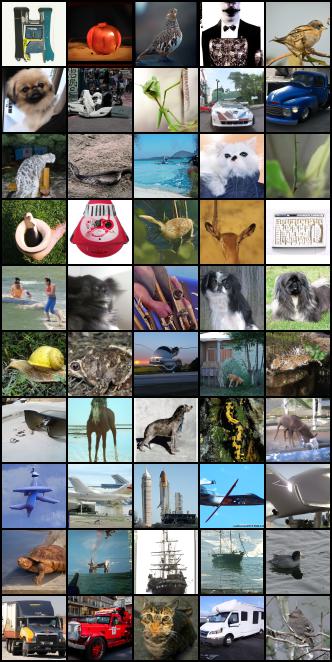}
        \caption{\revision{Iteration=1.}}
    \end{subfigure}
    \begin{subfigure}[b]{0.4\linewidth}
        \centering
        \includegraphics[width=0.45\linewidth]{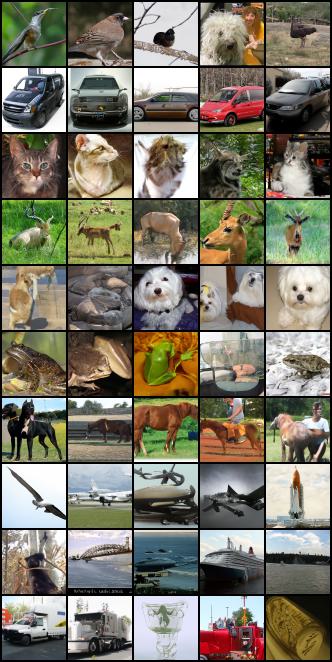}
        \includegraphics[width=0.45\linewidth]{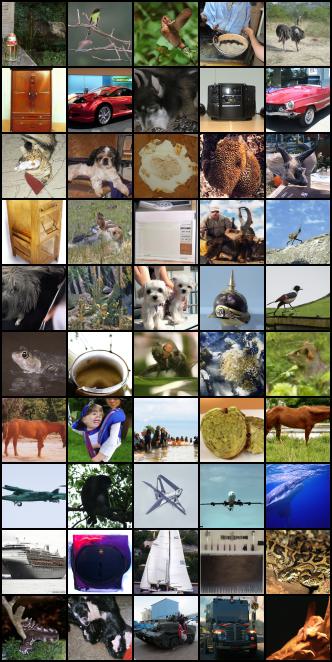}
        \caption{\revision{Iteration=2.}}
    \end{subfigure}
    \begin{subfigure}[b]{0.4\linewidth}
        \centering
        \includegraphics[width=0.45\linewidth]{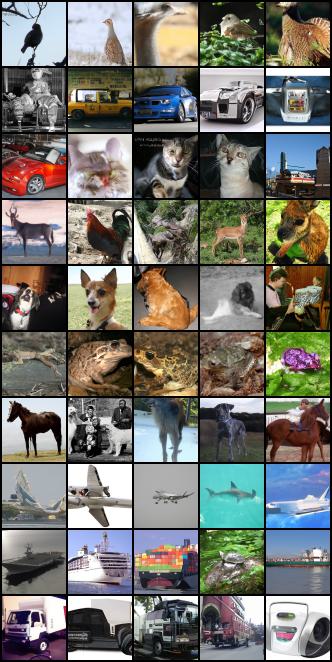}
        \includegraphics[width=0.45\linewidth]{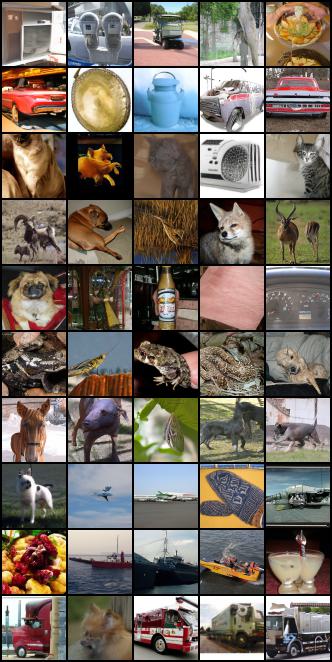}
        \caption{\revision{Iteration=3.}}
    \end{subfigure}
    \begin{subfigure}[b]{0.4\linewidth}
        \centering
        \includegraphics[width=0.45\linewidth]{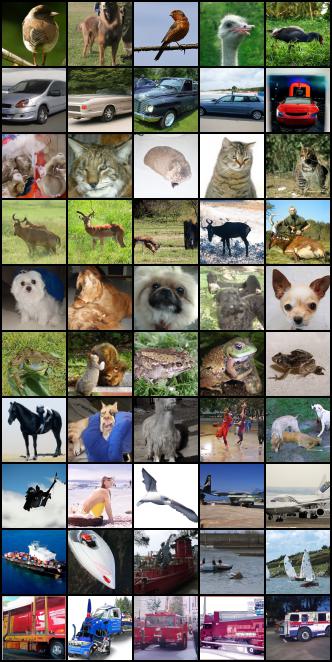}
        \includegraphics[width=0.45\linewidth]{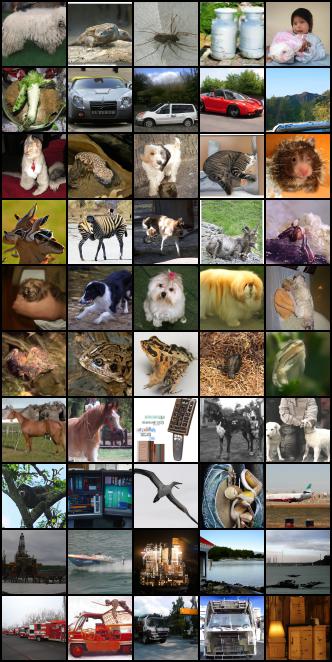}
        \caption{\revision{Iteration=4.}}
    \end{subfigure}
    \begin{subfigure}[b]{0.4\linewidth}
        \centering
        \includegraphics[width=0.45\linewidth]{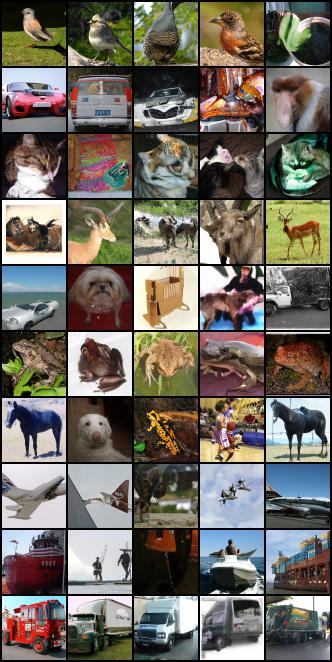}
        \includegraphics[width=0.45\linewidth]{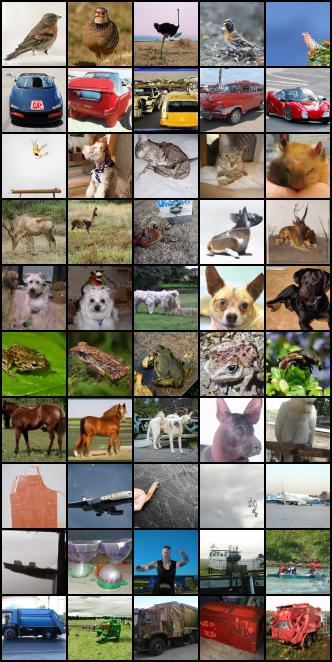}
        \caption{\revision{Iteration=5.}}
    \end{subfigure}
    \caption{\revision{Samples with the highest and the lowest counts in the \dpvotingname{} on \cifar{}. In each subfigure, each row corresponds to one class; the left subfigure shows the samples with the highest counts and the right subfigure shows the samples with the lowest counts.}}
    \label{fig:cifar10-top-bottom-count}
\end{figure}

\revision{\myparatightestn{Distribution of counts in \dpvotingname{}.}
\cref{fig:cifar_count_std} shows that the standard deviation of counts in \dpvotingname{} tends to decrease with more \algnameshort{} iterations on \cifar{}. This means that the histogram becomes ``more uniform'' with more \algnameshort{} iterations.
It is expected for the following reasons. In the beginning, only a few generated samples are close to the private data. Those generated samples get most of the votes from private samples, which results in a concentrated histogram. PE will then pick those highly voted generated samples and do more variations over them. As a result, private samples that voted for the same generated sample may now find different closest generated samples and distribute the votes, which results in a ``more uniform'' histogram. 
}

\begin{figure}
    \centering
    \includegraphics[width=0.4\linewidth]{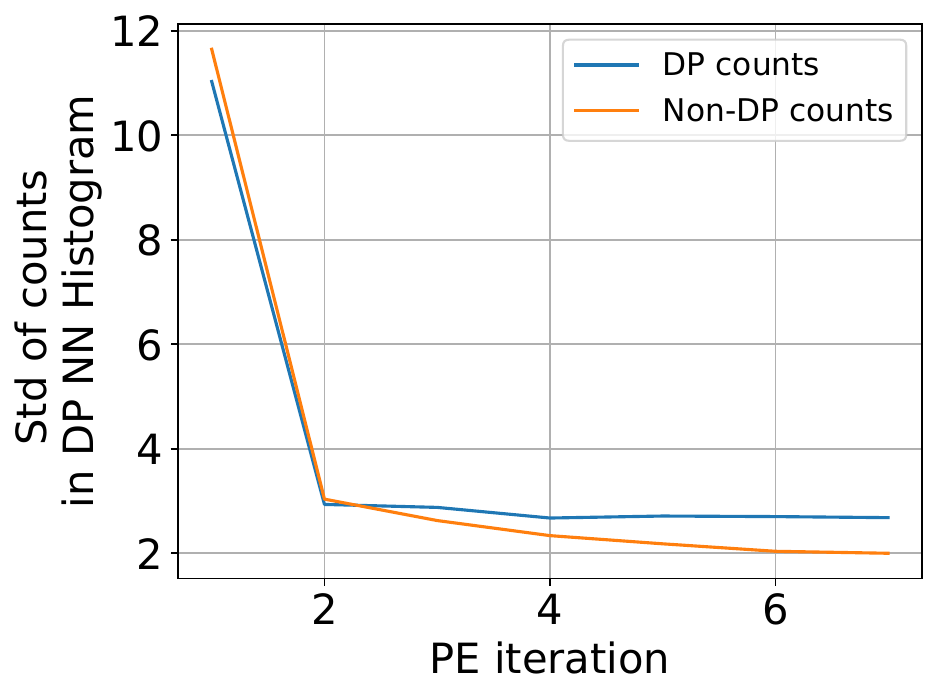}
    \caption{\revision{The standard deviation of counts in \dpvotingname{} across different \algnameshort{} iterations on \cifar{}. ``DP counts'' refers to the counts after adding Gaussian noise and thresholding. ``Non-DP counts'' refer to the counts before adding Gaussian noise and thresholding.}}
    \label{fig:cifar_count_std}
\end{figure}

%% file: tex/app_camelyon.tex
\section{More Details and Resutls on \camelyon{} Experiments}
\label{app:camelyon}

\myparatightestn{Pre-trained model.} We use the checkpoint \codeword{imagenet64_cond_270M_250K.pt} released in \cite{nichol2021improved}.\footnote{\url{https://github.com/openai/improved-diffusion}}

\myparatightestn{API implementation.} Same as \cref{app:cifar}.

\myparatightestn{Hyperparameters.} 
We set \packing{} degree $\packingdegree=8$, and number of generated samples $\numgensamples=302436$.

About the experiments in \cref{fig:camelyon_gen_samples}.
For \randomsampleapiname{} and \samplevariationapiname{}, we use DDIM sampler with 10 steps. 
For \samplevariationapiname{}, we take a 2-stage approach.  the first stage, we use DDIM sampler with 10 steps and use SDEEdit \cite{meng2021sdedit} by adding noise till $[10, 10, 10, 10,  9,  9,  9,  9,  9,  8,  8,  8,  8,  8,  7,  7, 7,  7]$ timesteps for each iteration respectively.   In the second stage, we use DDIM sampler with 40 steps and use SDEEdit \cite{meng2021sdedit} by adding noise till $[20, 19, 18, 17, 16, 15, 14, 13, 12, 11, 10]$ timesteps for each iteration respectively.
These timesteps can be regarded as the $\variationdegree$ parameter in \cref{sec:scope}.
We use noise multiplier $\noisemultiplier=2\cdot\sqrt{2}$ and threshold $\threshold=4$. 

About the experiments in \cref{sec:exp_camelyon}.
For \randomsampleapiname{} and \samplevariationapiname{}, we use DDIM sampler with 10 steps. 
For \samplevariationapiname{}, we use DDIM sampler with 10 steps and use SDEEdit \cite{meng2021sdedit} by adding noise till %
\revision{$[10, 10, 10, 10,  9,  9,  9]$}
timesteps for each iteration respectively. 
These timesteps can be regarded as the $\variationdegree$ parameter in \cref{sec:scope}.
We use noise multiplier %
\revision{$\noisemultiplier=1.381$} 
and threshold $\threshold=4$.

\myparatightestn{Generated samples.}
See \cref{fig:camelyon_gen_samples,fig:camelyon_real} for generated images and real images side-by-side. Note that the real images in \camelyon{} dataset are 96x96 images, whereas the pre-trained network is 64x64. In \cref{fig:camelyon_real}, we scale them down to 64x64 for better comparison.

\cref{fig:camelyon_gen_samples_iteration} shows the generated images in the intermediate iterations. We can see that the generated images are effectively guided towards \camelyon{} though it is very different from the pre-training dataset.

\begin{figure}[ht]
    \begin{minipage}{.45\textwidth}
	 	\centering
		\includegraphics[width=\linewidth]{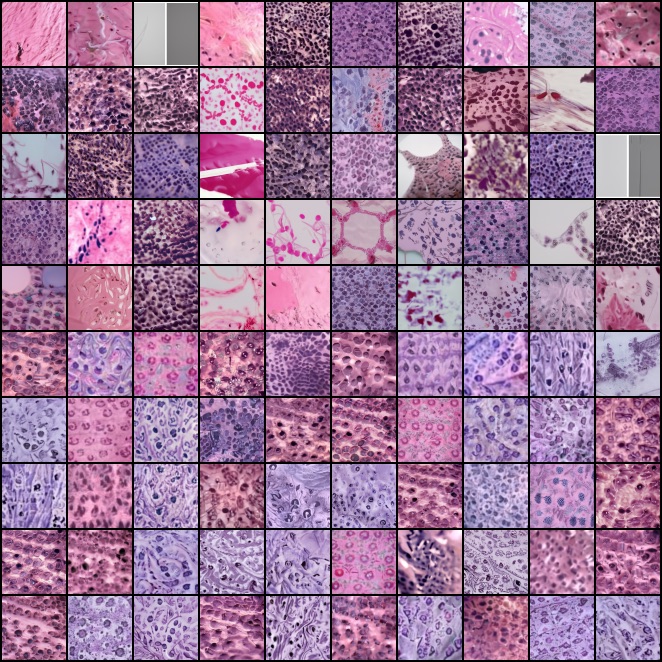}
        \caption{Generated samples on \camelyon{} with $\bra{9.92,3\cdot 10^{-6}}$-DP. The first five rows correspond to one class, and the rest correspond to the other class. FID=10.66.}
        \label{fig:camelyon_gen_samples}
    \end{minipage}%
    \hfill
    \begin{minipage}{.45\textwidth}
      	\centering
		\includegraphics[width=\linewidth]{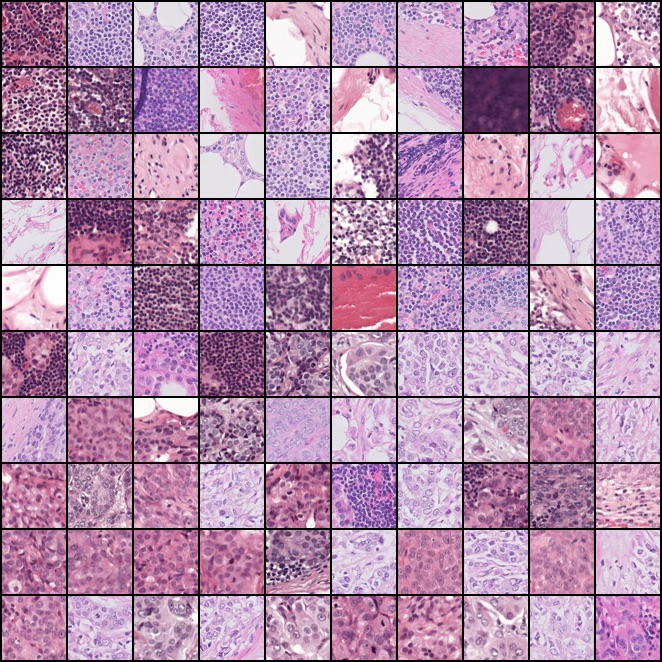}
		\caption{Real samples from \camelyon{}. The first five rows correspond to one class, and the rest correspond to the other class.\\}
		\label{fig:camelyon_real}
    \end{minipage}
\end{figure}

\begin{figure}[hbt!]
    \centering
    \begin{subfigure}[b]{0.26\linewidth}
        \centering
        \includegraphics[width=\linewidth]{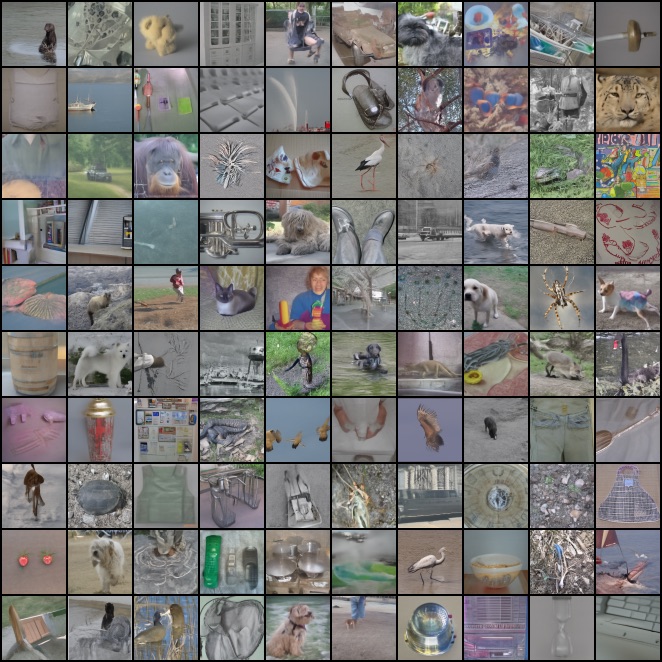}
        \caption{Iteration=0.}
    \end{subfigure}
    \begin{subfigure}[b]{0.26\linewidth}
        \centering
        \includegraphics[width=\linewidth]{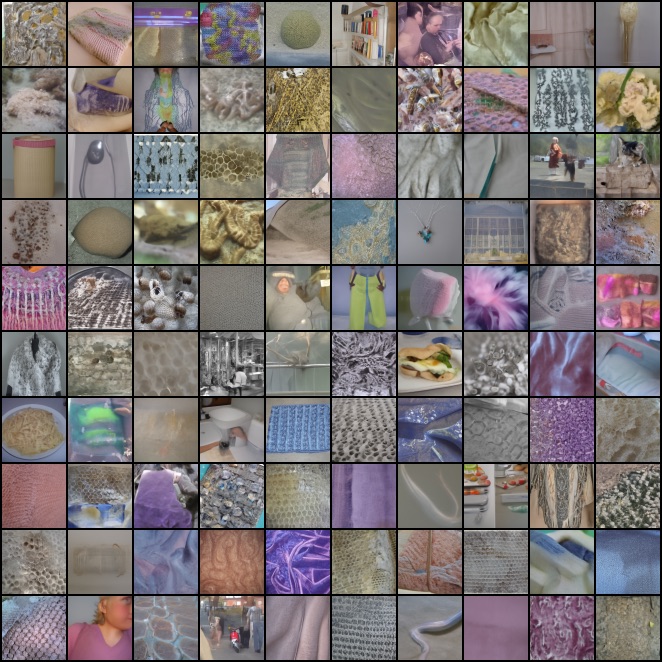}
        \caption{Iteration=1.}
    \end{subfigure}
    \begin{subfigure}[b]{0.26\linewidth}
        \centering
        \includegraphics[width=\linewidth]{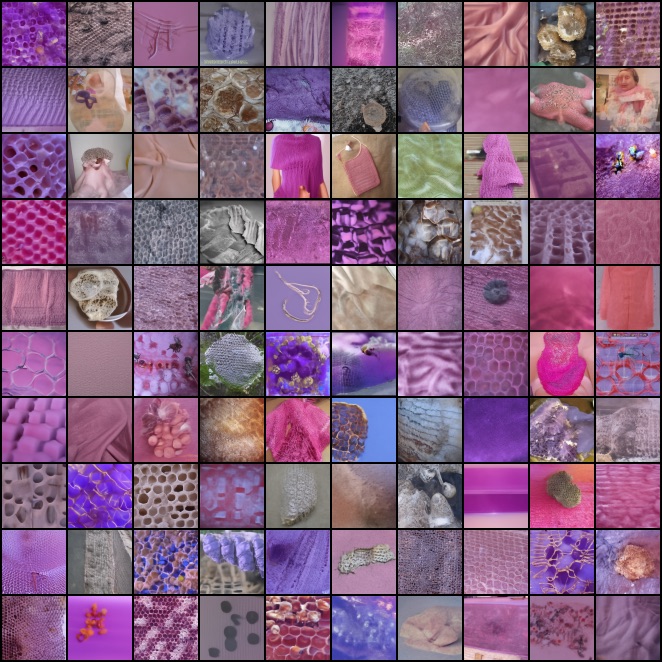}  
        \caption{Iteration=3.}
    \end{subfigure}
    \begin{subfigure}[b]{0.26\linewidth}
        \centering
        \includegraphics[width=\linewidth]{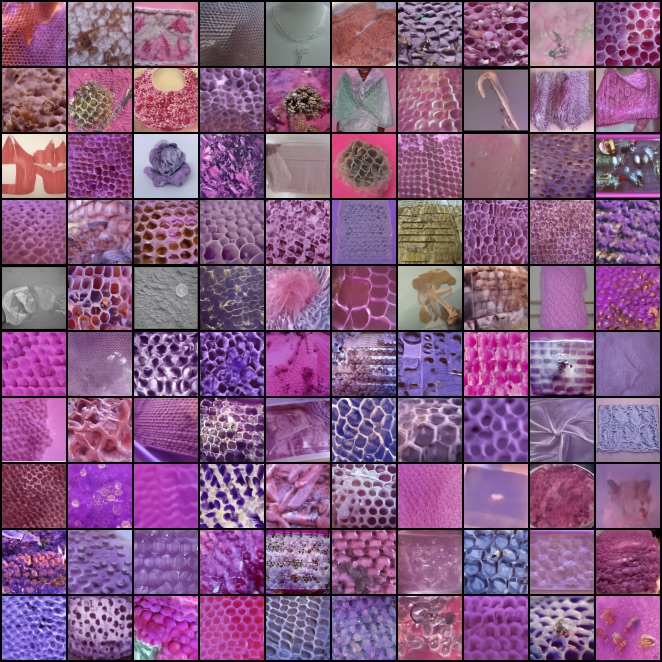}  
        \caption{Iteration=5.}
    \end{subfigure}
    \begin{subfigure}[b]{0.26\linewidth}
        \centering
        \includegraphics[width=\linewidth]{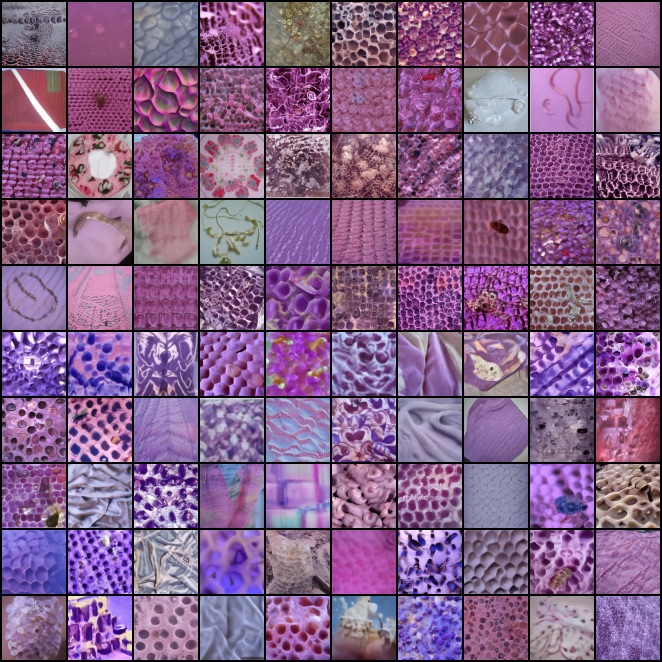}  
        \caption{Iteration=7.}
    \end{subfigure}
    \begin{subfigure}[b]{0.26\linewidth}
        \centering
        \includegraphics[width=\linewidth]{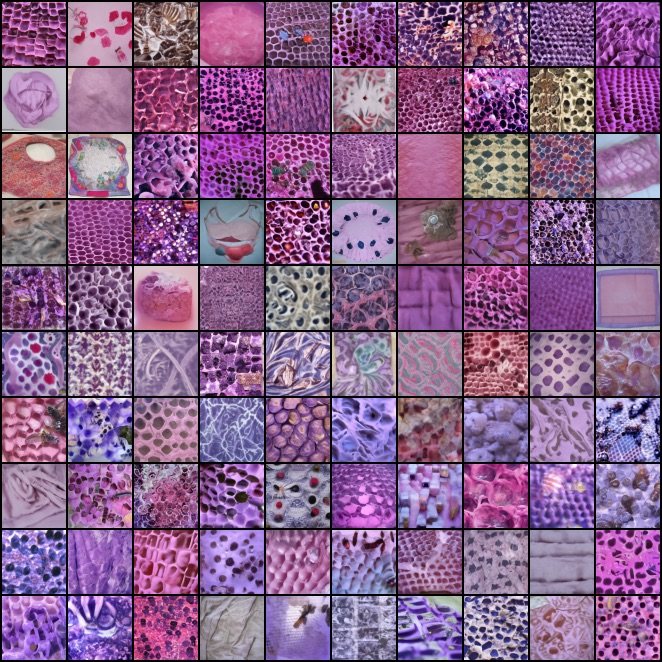}  
        \caption{Iteration=9.}
    \end{subfigure}
    \begin{subfigure}[b]{0.26\linewidth}
        \centering
        \includegraphics[width=\linewidth]{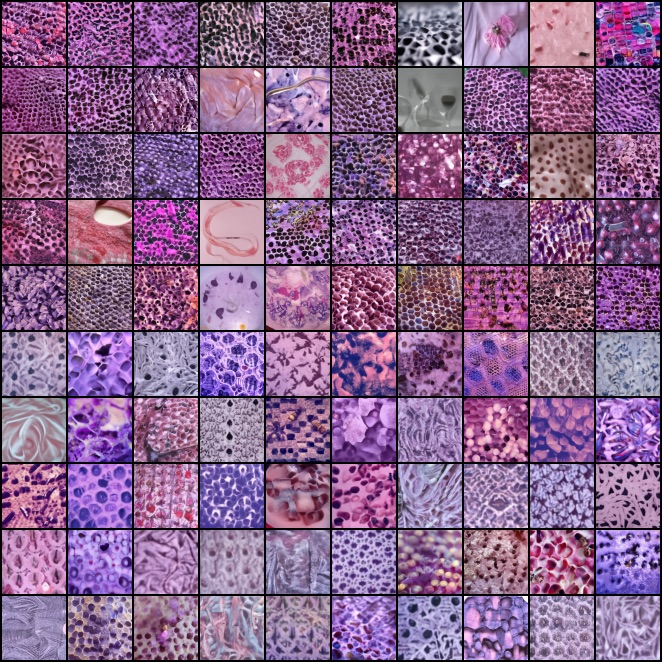}  
        \caption{Iteration=11.}
    \end{subfigure}
    \begin{subfigure}[b]{0.26\linewidth}
        \centering
        \includegraphics[width=\linewidth]{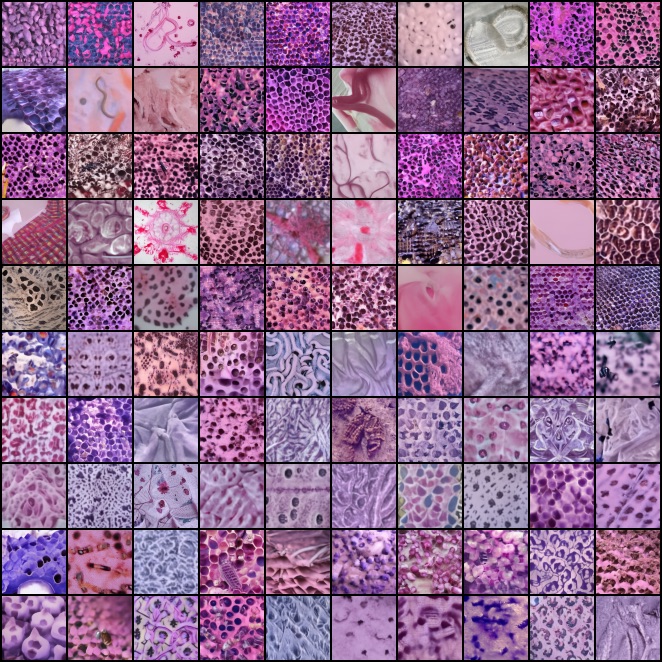}  
        \caption{Iteration=13.}
    \end{subfigure}
    \begin{subfigure}[b]{0.26\linewidth}
        \centering
        \includegraphics[width=\linewidth]{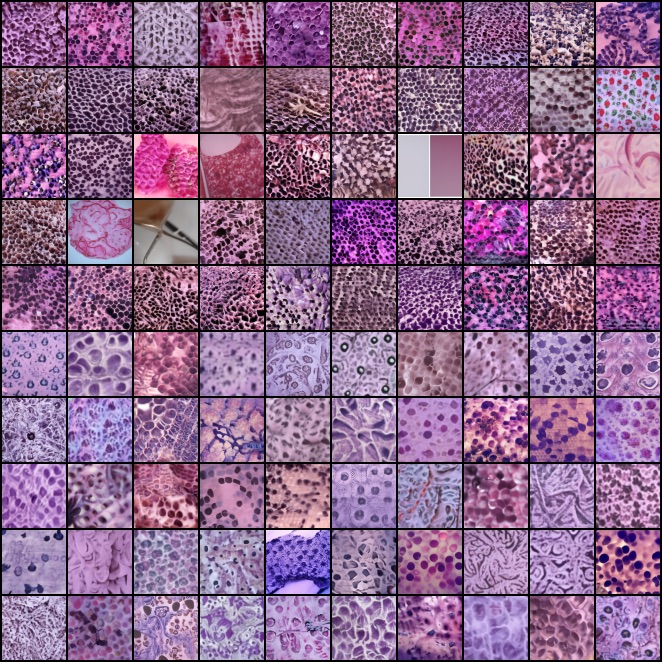}  
        \caption{Iteration=15.}
    \end{subfigure}
    \begin{subfigure}[b]{0.26\linewidth}
        \centering
        \includegraphics[width=\linewidth]{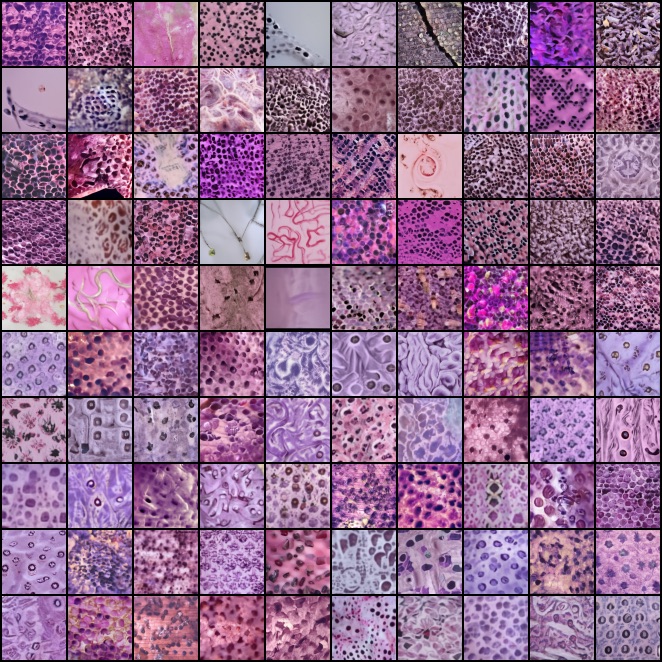}  
        \caption{Iteration=17.}
    \end{subfigure}
    \begin{subfigure}[b]{0.26\linewidth}
        \centering
        \includegraphics[width=\linewidth]{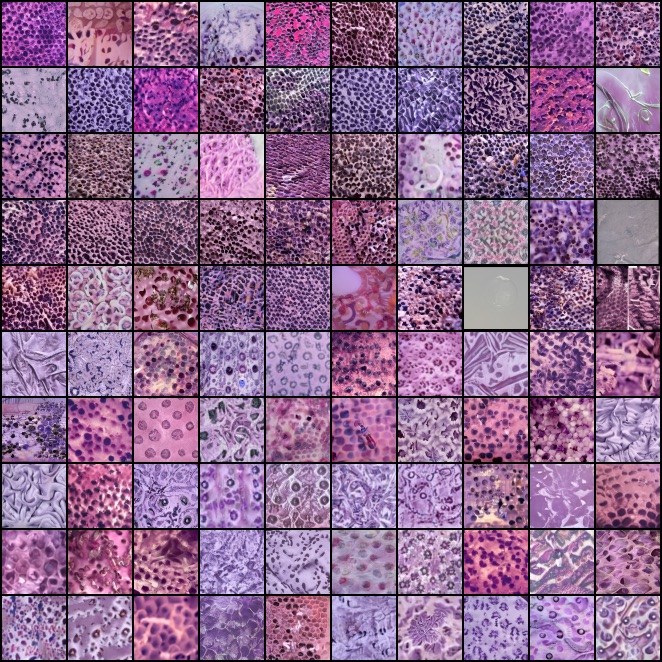}  
        \caption{Iteration=19.}
    \end{subfigure}
    \begin{subfigure}[b]{0.26\linewidth}
        \centering
        \includegraphics[width=\linewidth]{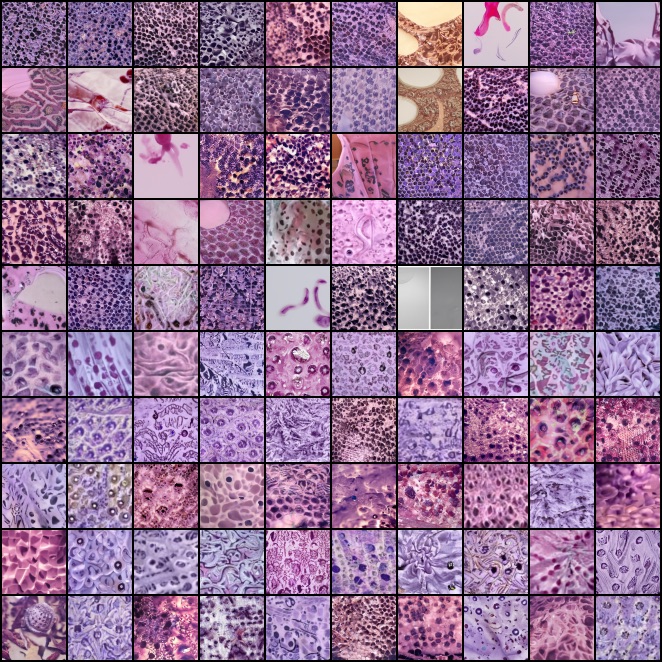}  
        \caption{Iteration=21.}
    \end{subfigure}
    \begin{subfigure}[b]{0.26\linewidth}
        \centering
        \includegraphics[width=\linewidth]{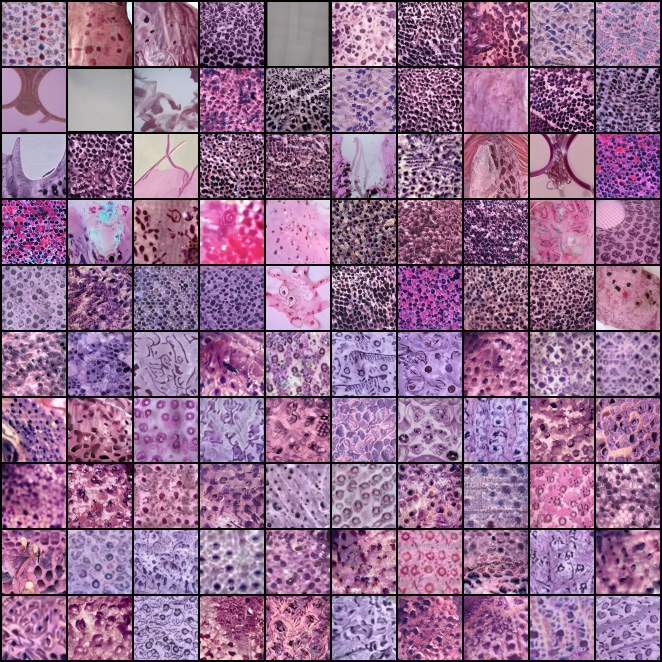}  
        \caption{Iteration=23.}
    \end{subfigure}
    \begin{subfigure}[b]{0.26\linewidth}
        \centering
        \includegraphics[width=\linewidth]{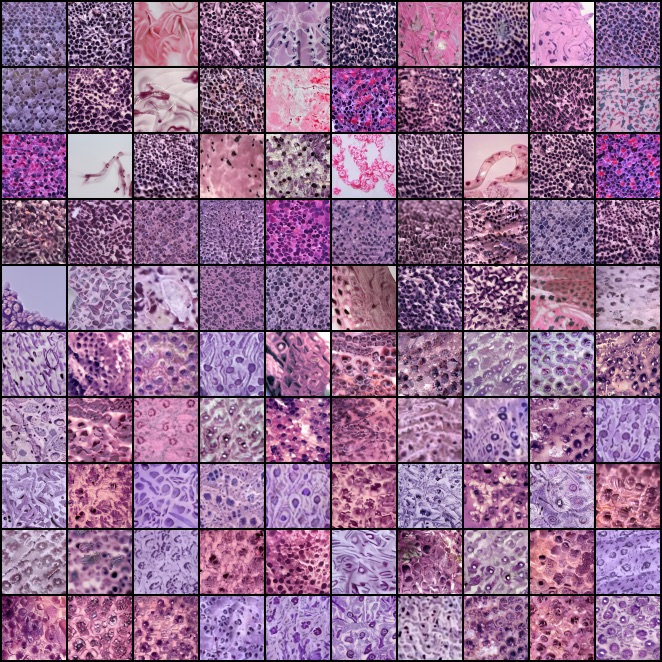}  
        \caption{Iteration=25.}
    \end{subfigure}
    \begin{subfigure}[b]{0.26\linewidth}
        \centering
        \includegraphics[width=\linewidth]{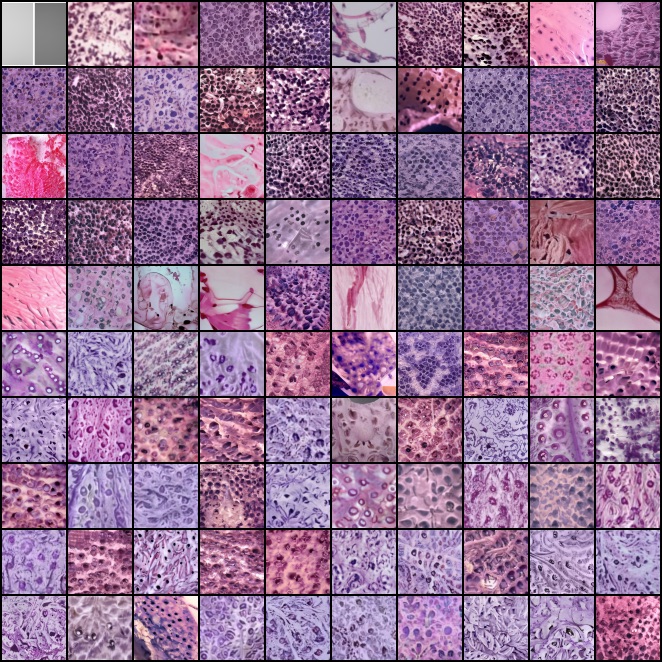}  
        \caption{Iteration=27.}
    \end{subfigure}
    \caption{Generated samples on \camelyon{} at the first few iterations. We can see that the generated images are gradually guided from \imagenet{}, the pre-training dataset, to \camelyon{}, the (very different) private dataset. ``Iteration=0'' means the initial random samples from \randomsampleapiname{}.}
    \label{fig:camelyon_gen_samples_iteration}
\end{figure}

\myparatightestn{Nearest samples in the private dataset}
\cref{fig:camelyon_nearest_neighbor_inception,fig:camelyon_nearest_neighbor_original} show generated images and their nearest neighbors in the private dataset evaluated using two distance metrics: $\ell_2$ distance in the inception embedding space and the pixel space. Similar to the results in \cifar{}, we can see that the generated images are different from private images. This is expected due to the DP guarantee.
\begin{figure}[h]
    \centering
    \includegraphics[width=0.55\linewidth]{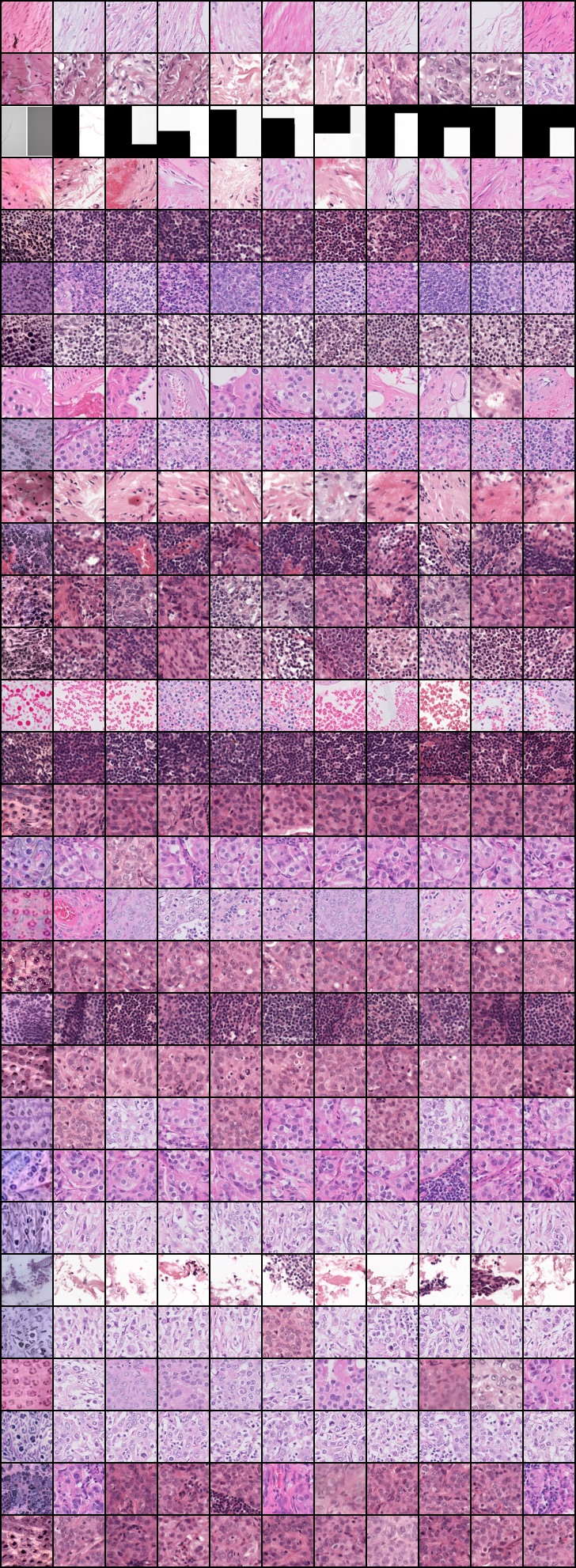}
    \caption{Nearest samples in the private dataset on \camelyon{}. In each row, the first column is a generated image (from \cref{fig:camelyon_gen_samples}), and the other columns are its nearest neighbors in the private dataset, sorted by the distance in ascending order. Every fifteen rows correspond to generated image from one class. The distance metric is $\ell_2$ in the \textbf{inception embedding space}. }
    \label{fig:camelyon_nearest_neighbor_inception}
\end{figure}

\begin{figure}[h]
    \centering
    \includegraphics[width=0.55\linewidth]{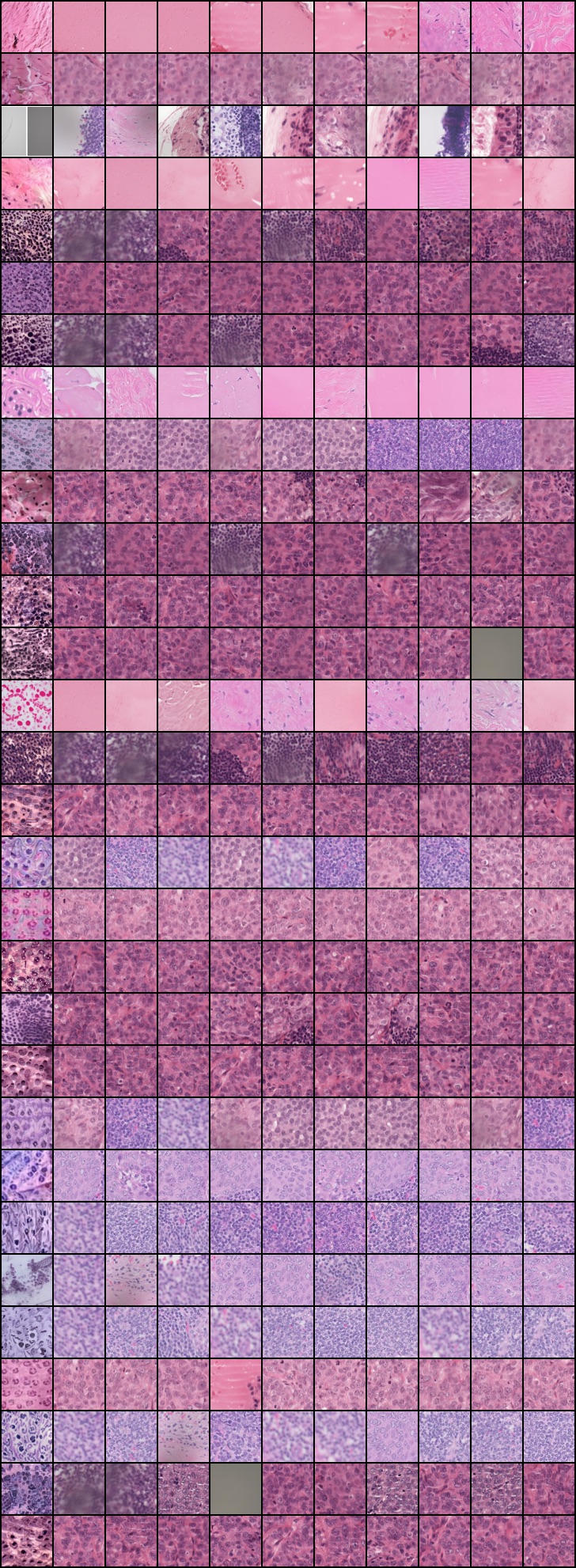}
    \caption{Nearest samples in the private dataset on \camelyon{}. In each row, the first column is a generated image (from \cref{fig:cifar_gen_samples_app}), and the other columns are its nearest neighbors in the private dataset, sorted by the distance in ascending order. Every fifteen rows correspond to generated image from one class. The distance metric is $\ell_2$ in the \textbf{pixel space}. }
    \label{fig:camelyon_nearest_neighbor_original}
\end{figure}

\revision{\myparatightestn{Samples with the highest and the lowest votes.}
\cref{fig:camelyon-top-bottom-count} shows the samples with the highest and the lowest votes in \dpvotingname{} across different \algnameshort{} iterations.
We can see that \dpvotingname{} gradually picks samples with the right patterns as the private data, and drops samples with more different patterns. As \camelyon{} is further away from the pre-training dataset \imagenet{} than \cifar{}, we can see that it converges slower than the case in \cifar{} (\cref{fig:cifar10-top-bottom-count}).
} 

\begin{figure}
    \centering
    \begin{subfigure}[b]{0.4\linewidth}
        \centering
        \includegraphics[width=0.45\linewidth]{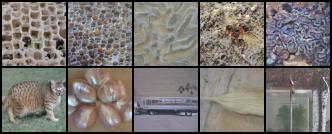}
        \includegraphics[width=0.45\linewidth]{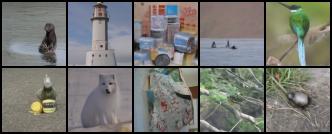}
        \caption{\revision{Iteration=0.}}
    \end{subfigure}
    \begin{subfigure}[b]{0.4\linewidth}
        \centering
        \includegraphics[width=0.45\linewidth]{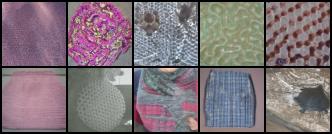}
        \includegraphics[width=0.45\linewidth]{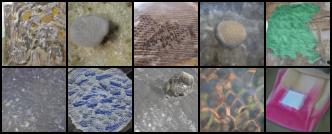}
        \caption{\revision{Iteration=1.}}
    \end{subfigure}
    \begin{subfigure}[b]{0.4\linewidth}
        \centering
        \includegraphics[width=0.45\linewidth]{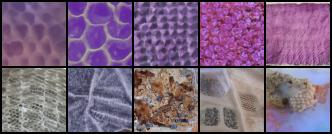}
        \includegraphics[width=0.45\linewidth]{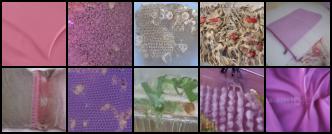}
        \caption{\revision{Iteration=2.}}
    \end{subfigure}
    \begin{subfigure}[b]{0.4\linewidth}
        \centering
        \includegraphics[width=0.45\linewidth]{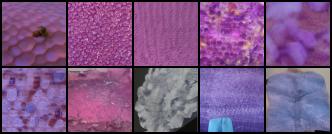}
        \includegraphics[width=0.45\linewidth]{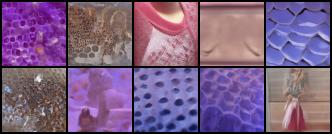}
        \caption{\revision{Iteration=3.}}
    \end{subfigure}
    \begin{subfigure}[b]{0.4\linewidth}
        \centering
        \includegraphics[width=0.45\linewidth]{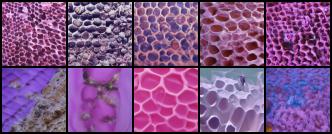}
        \includegraphics[width=0.45\linewidth]{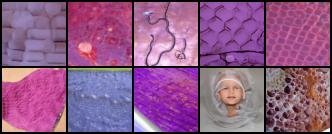}
        \caption{\revision{Iteration=4.}}
    \end{subfigure}
    \begin{subfigure}[b]{0.4\linewidth}
        \centering
        \includegraphics[width=0.45\linewidth]{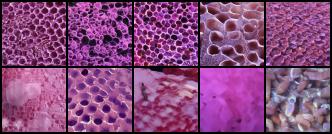}
        \includegraphics[width=0.45\linewidth]{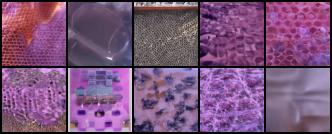}
        \caption{\revision{Iteration=5.}}
    \end{subfigure}
    \caption{\revision{Samples with the highest and the lowest counts in the \dpvotingname{} on \camelyon{}. In each subfigure, each row corresponds to one class; the left subfigure shows the samples with the highest counts and the right subfigure shows the samples with the lowest counts.}}
    \label{fig:camelyon-top-bottom-count}
\end{figure}

\revision{\myparatightestn{Distributions of the distances to nearest samples.} Continuing the above experiments, we further show the distribution of the distances between (1) generated samples and their nearest real samples and (2) real samples and their nearest generated samples in \cref{fig:nn_cdf_camelyon,fig:inverse_nn_cdf_camelyon} respectively. Similar to the \cifar{} experiments (\cref{fig:nn_cdf_cifar10,fig:inverse_nn_cdf_cifar10}), we see that: (1) During the early \algnameshort{} iterations, the \emph{inception} distances tend to decrease. This means that \algnameshort{} is effective in pushing the generated distribution to be closer to the private distribution. (2) However, as \algnameshort{} continues, the \emph{inception} distances stop decreasing. It is expected, as DP upper bounds the probability of reconstructing any sample in the private dataset.
However, one difference to \cifar{} results is that the distance between each generated sample and its nearest private samples measured in the \emph{original pixel space} (\cref{fig:nn_cdf_camelyon_pixel}) tends to increase. We hypothesize that it is be due to nature of this dataset that any shifts of the histological images are still in-distribution but can result in a high distance in the \emph{original pixel space}.
}

\begin{figure}[ht]
    \centering
    \begin{subfigure}[b]{0.4\linewidth}
        \centering
        \includegraphics[width=1\linewidth]{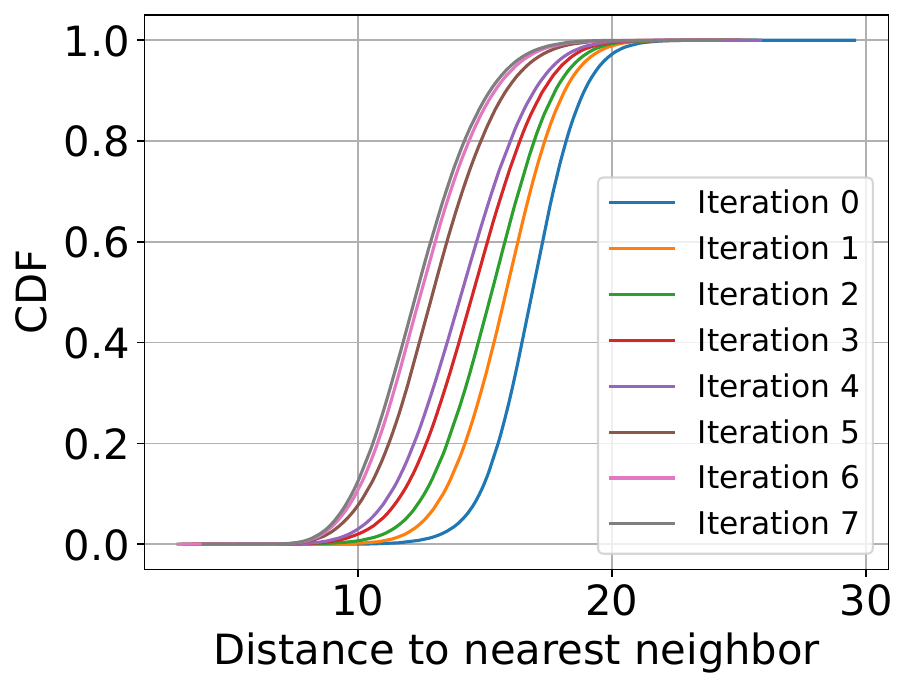}
        \caption{\revision{The distance metric is $\ell_2$ in the \textbf{inception embedding space.}}}
    \end{subfigure}
    ~~~~~
    \begin{subfigure}[b]{0.4\linewidth}
        \centering
        \includegraphics[width=1\linewidth]{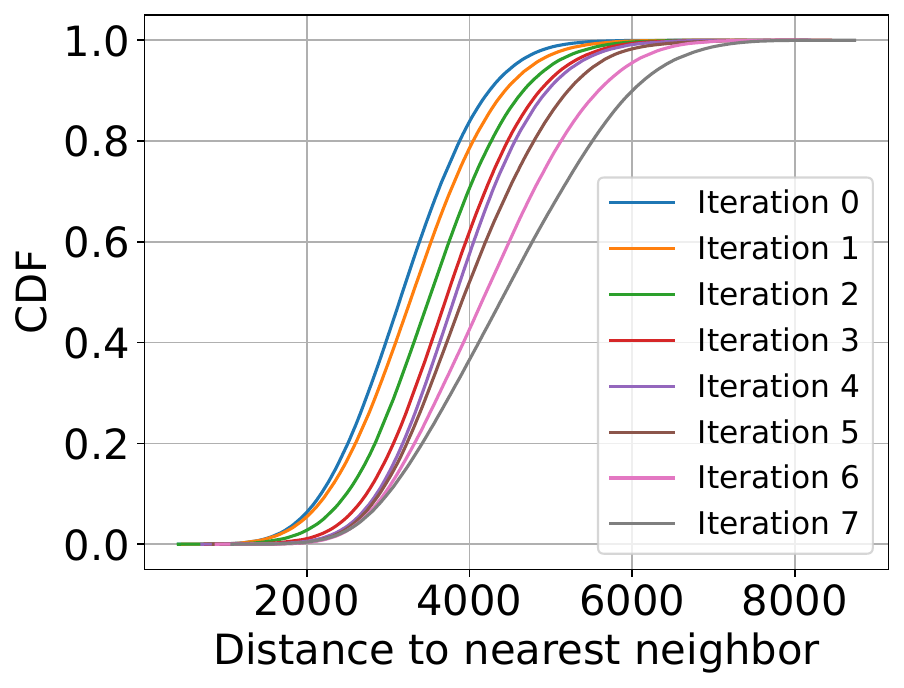}
        \caption{\revision{The distance metric is $\ell_2$ in the \textbf{pixel space.}}}
        \label{fig:nn_cdf_camelyon_pixel}
    \end{subfigure}
    \caption{\revision{CDF of the distributions \textbf{between each generated sample and its nearest private samples} on \cifar{} across different \algnameshort{} iterations. ``Iteration 0'' refers to the initial random samples from \cref{line:initial} in \cref{alg:main}.}}
    \label{fig:nn_cdf_camelyon}
\end{figure}
\begin{figure}[ht]
    \centering
    \begin{subfigure}[b]{0.4\linewidth}
        \centering
        \includegraphics[width=1\linewidth]{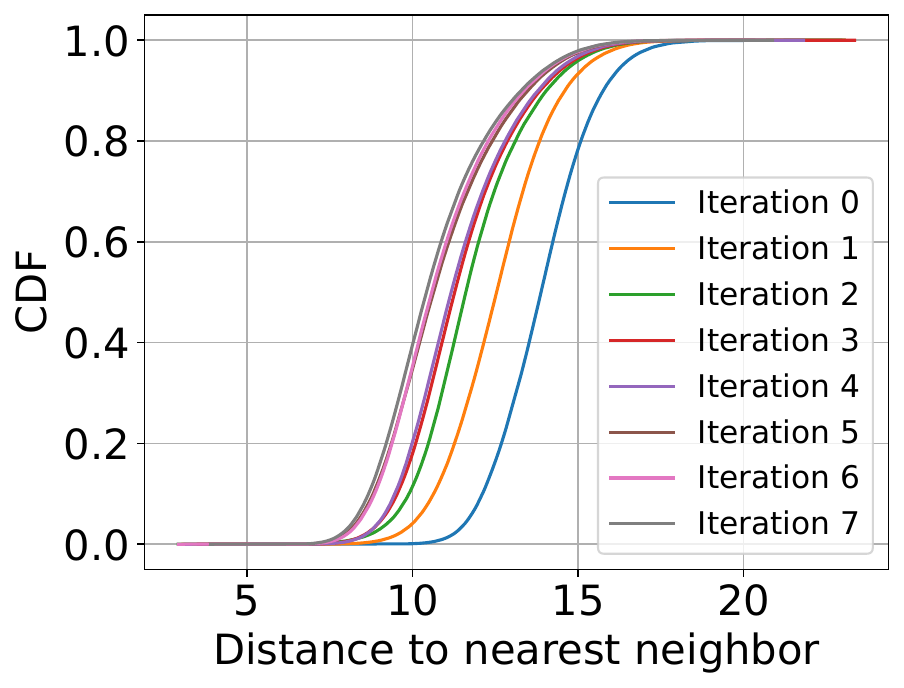}
        \caption{\revision{The distance metric is $\ell_2$ in the \textbf{inception embedding space.}}}
    \end{subfigure}
    ~~~~~
    \begin{subfigure}[b]{0.4\linewidth}
        \centering
        \includegraphics[width=1\linewidth]{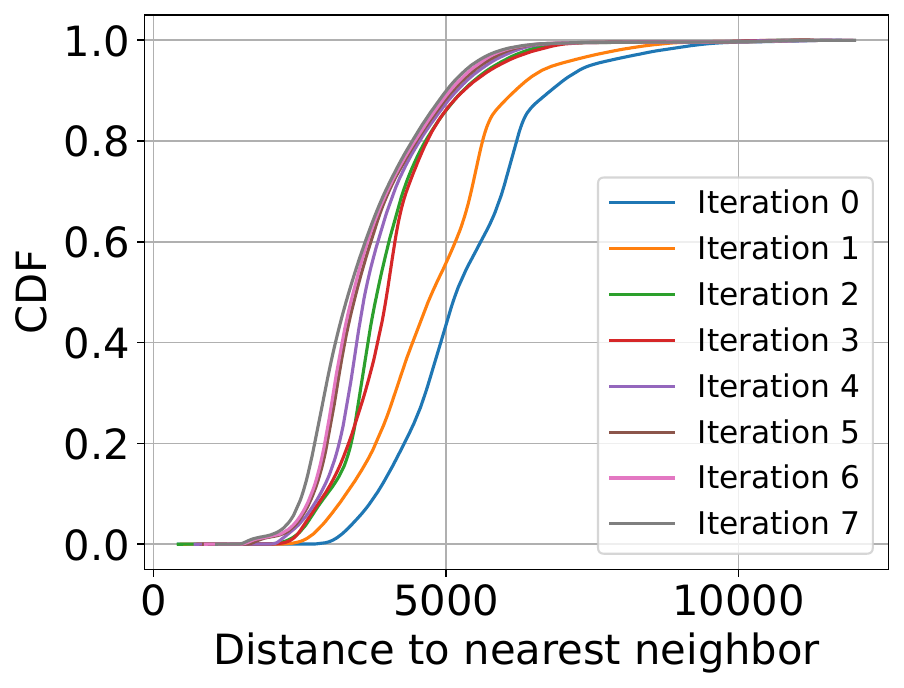}
        \caption{\revision{The distance metric is $\ell_2$ in the \textbf{pixel space.}}}
    \end{subfigure}
    \caption{\revision{CDF of the distributions \textbf{between each private sample and its nearest generated samples} on \cifar{} across different \algnameshort{} iterations. ``Iteration 0'' refers to the initial random samples from \cref{line:initial} in \cref{alg:main}.}}
    \label{fig:inverse_nn_cdf_camelyon}
\end{figure}

\revision{\myparatightestn{Distribution of counts in \dpvotingname{}.}
\cref{fig:camelyon_count_std} shows that the standard deviation of counts in \dpvotingname{} tends to decrease with more \algnameshort{} iterations on \camelyon{}. This accords with the observations and takeaways messages in \cifar{} (\cref{fig:cifar_count_std}).
}

\begin{figure}
    \centering
    \includegraphics[width=0.4\linewidth]{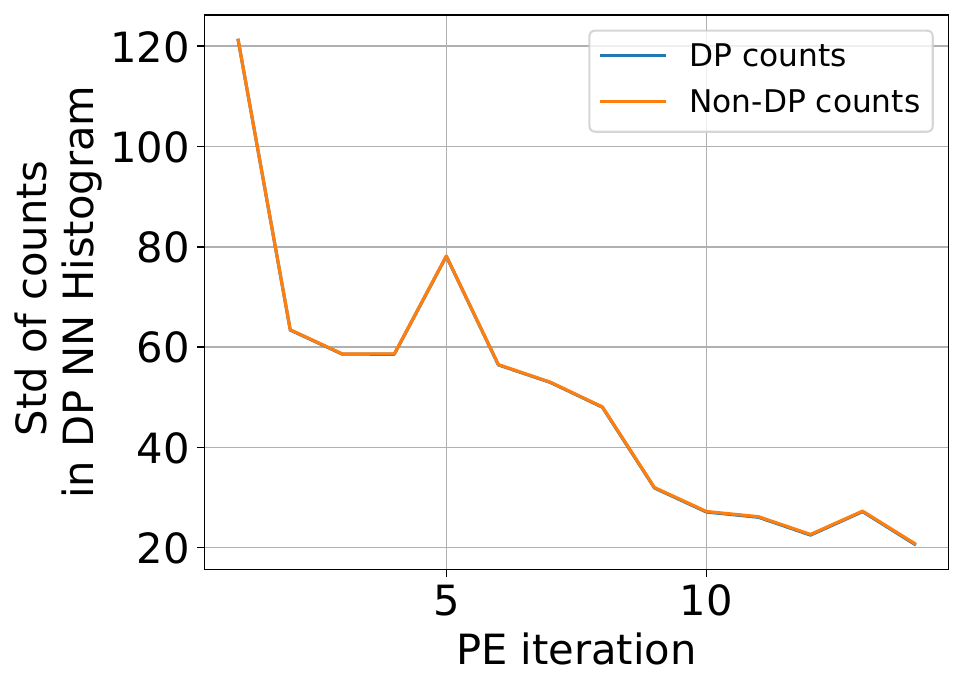}
    \caption{\revision{The standard deviation of counts in \dpvotingname{} across different \algnameshort{} iterations on \camelyon{}. ``DP counts'' refers to the counts after adding Gaussian noise and thresholding. ``Non-DP counts'' refer to the counts before adding Gaussian noise and thresholding.}}
    \label{fig:camelyon_count_std}
\end{figure}

\revision{\myparatightestn{Inspescting the ImageNet class labels of the generated samples.} As discussed in \cref{app:conditional_unconditional}, each generated samples have an associated \imagenet{} label. Here, we inspect those labels to understand how \algnameshort{} works in this dataset. The list of the labels and their associated number of generated images are: ``honeycomb'' (164647), ``velvet'' (83999), ``nematode, nematode worm, roundworm'' (35045), ``handkerchief, hankie, hanky, hankey'' (14495), ``bib'' (3102), ``head cabbage'' (934), ``bubble'' (142), ``stole'' (72). We see that 54.4\% images are with the label ``honeycomb''. Many honeycomb images in \imagenet{} have a ``net-like'' structure, which shares some similarities with \camelyon{}. However, the details, colors, and structures are still different. We compute the FID between honeycomb images and CIFAR10. The FID score is 162.86, which is much higher than the FID score of the final generated images we get (10.66 in \cref{fig:camelyon_gen_samples}). These results show that the existence of the honeycomb class in \imagenet{} helps with \algnameshort{}, but \algnameshort{} does more than just picking this class.  }

\revision{
\myparatightestn{Why \algnameshort{} works under large distribution shifts.}
The fact that \algnameshort{} is able to make \imagenet{}-pre-trained-models to generate \camelyon{} images may appear surprising to readers. Here, we provide intuitive explanations.
Even though the diffusion model is trained on natural images, the support of the generated distribution could be very large due to the formulation. More concretely, let’s look at two examples: score-based models \citep{song2019generative}, which is closely related to diffusion models, and the diffusion model we used \citep{ho2020denoising}. For score-based models, its modeling target is a distribution perturbed by Gaussian noise, and therefore, the generated distribution spans the entire sample space. For diffusion models \citep{ho2020denoising}, the latent space has the same size as the image, and the last denoising step is modeled as a distribution derived from a Gaussian distribution that spans the entire pixel space (see Section 3.3 of \citet{ho2020denoising}). Therefore, the generated distribution of diffusion models also spans the entire sample space.
}

\revision{In other words, for any (trained) score-based models or diffusion models, it is theoretically possible to generate images similar to the private dataset (or in fact, generate any images). The problem is that the probability of generating such images is small if there is a large distribution shift from the pre-training data to the private data. \algnameshort{} is effective in guiding the diffusion model to generate samples from the region that is low-density in the original pre-trained distribution but high-density in the private distribution.
}

\revision{To make it more concrete, we provide how the generated images evolve from natural images to \camelyon{} dataset in \cref{fig:camelyon_gen_samples_iteration} and the selected/filtered samples by \algnameshort{} in \cref{fig:camelyon-top-bottom-count}. At every iteration, \algnameshort{} selects the set of images that are most similar to \camelyon{} dataset (\cref{fig:camelyon-top-bottom-count}). Those images might still appear different from \camelyon{} in early iterations. However, as long as we get images that are more similar to \camelyon{} through \samplevariationapiname{} at every iteration, we will make progress, and finally, we can get images similar to Camelyon17 (\cref{fig:camelyon_gen_samples}).
}

\revision{We further experiment \algnameshort{} under different levels of distribution shifts. To do that, we take the \camelyon{} dataset and modify the \emph{saturation} of the images to create a sequence of datasets, each with a different saturation change. This way, we create a sequence of datasets with different levels of distribution shifts from ImageNet.  From \cref{fig:camelyon_fid_iteration_distribution_shift}, we can see that no matter what the level of distribution shifts is, \algnameshort{} can consistently improve the generated distribution towards the private data.}

\begin{figure}
    \centering
    \includegraphics[width=0.4\linewidth]{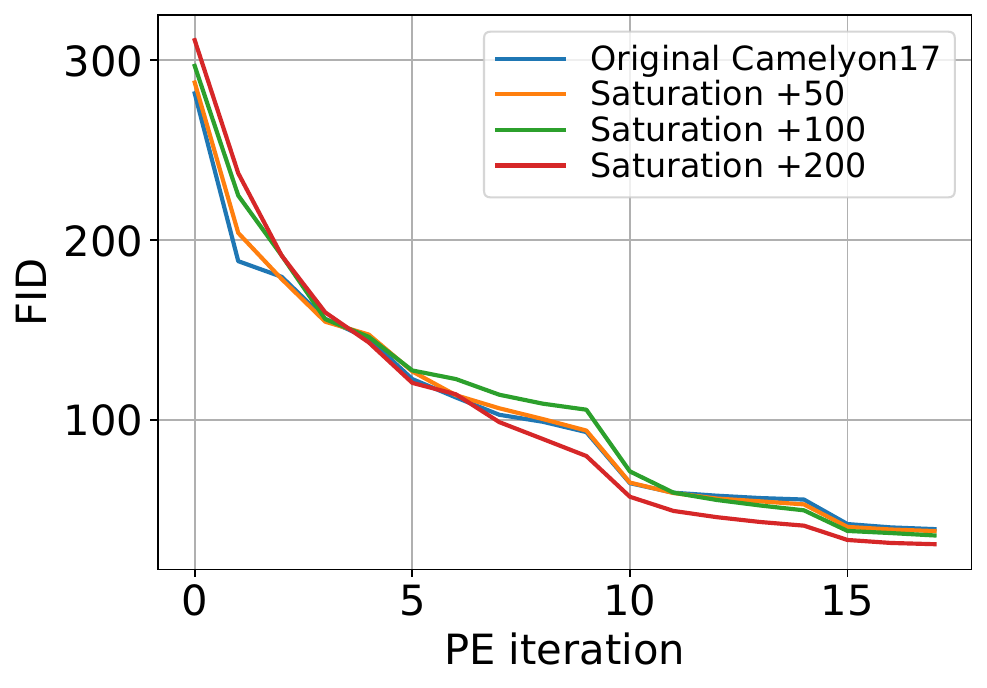}
    \caption{\revision{FID vs. \algnameshort{} iterations under \camelyon{} variants with different levels of distribution shifts.}}
    \label{fig:camelyon_fid_iteration_distribution_shift}
\end{figure}

%% file: tex/app_stable_diffusion.tex
\section{More Details and Results on Stable Diffusion Experiments}
\label{app:stable_diffusion}

\myparatightestn{Dataset construction.} We start with cat photos taken by the authors, crop the region around cat faces with a resolution larger than 512x512 manually, and resize the images to 512x512. We construct two datasets, each for one cat with 100 images. See \cref{fig:cookie_real,fig:ruby_real} for all images.
\revision{Note that having multiple samples from the same identity is not meaningful from the practical perspective of DP. Instead of regarding these datasets as real-world use cases, they should be treated as ``toy datasets'' for experimenting DP generative models with a small number of high-resolution images.}

\begin{figure}[ht]
    \centering
    \includegraphics[width=0.8\linewidth]{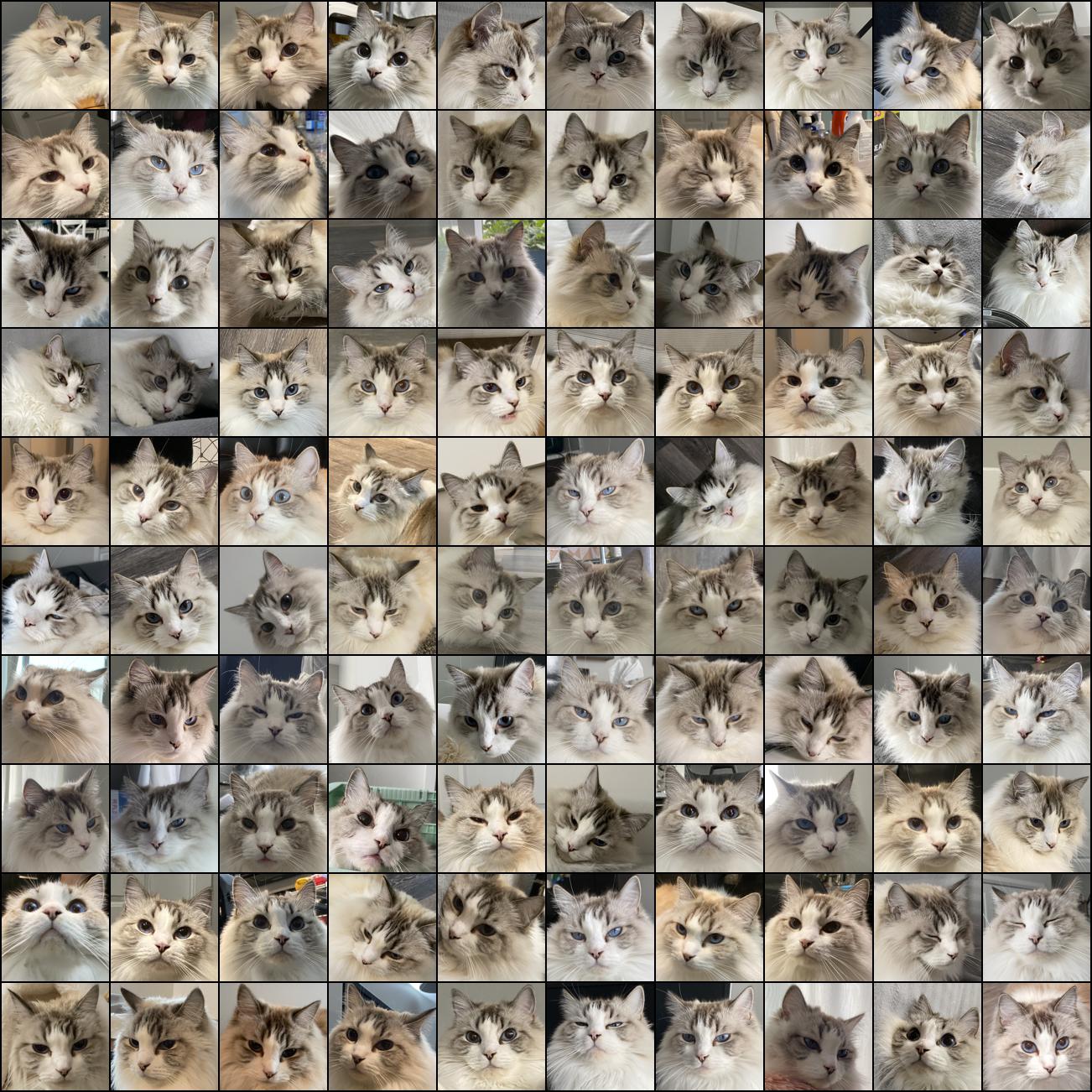}
    \caption{All images from \catcookie{} dataset. The original resolution is 512x512; we resize them to 128x128 here for reducing the file size of the paper.}
    \label{fig:cookie_real}
\end{figure}

\begin{figure}[ht]
    \centering
    \includegraphics[width=0.8\linewidth]{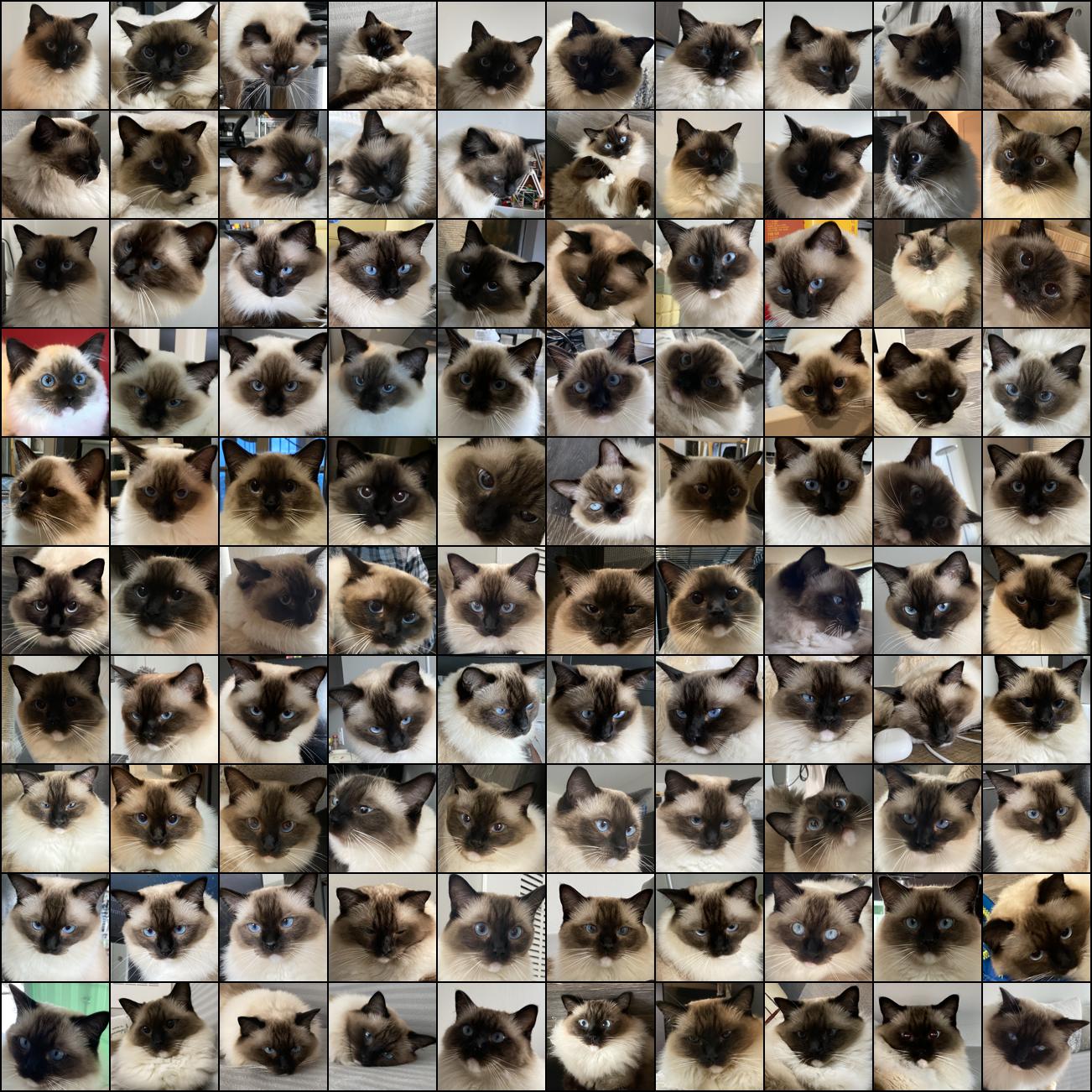}
    \caption{All images from \catruby{} dataset. The original resolution is 512x512; we resize them to 128x128 here for reducing the file size of the paper.}
    \label{fig:ruby_real}
\end{figure}

\myparatightestn{API implementation.}
We use off-the-shelf open-sourced APIs of Stable Diffusion. For \randomsampleapiname{}, we use the text-to-image generation API\footnote{\url{https://huggingface.co/docs/diffusers/api/pipelines/stable_diffusion/text2img}}, which is implemented by the standard diffusion models' guided sampling process.
For \samplevariationapiname{}, we use the image-to-image generation API\footnote{\url{https://huggingface.co/docs/diffusers/api/pipelines/stable_diffusion/img2img}}, which allows us to control the degree of variation. Its internal implementation is SDEdit \cite{meng2021sdedit}, which adds noise to the input images and runs diffusion models' denoising process.

\myparatightestn{Hyperparameters} 
We set \packing{} degree $\packingdegree=8$, and number of generated samples $\numgensamples=100$.
For \randomsampleapiname{} and \samplevariationapiname{}, we use the default DDIM sampler with 50 steps. 
For \randomsampleapiname{}, we use the prompt "A photo of ragdoll cat". This gives reaonable cat images but still far away from the private data (\cref{fig:ruby_gen_0}).
For \samplevariationapiname{}, we use variation degrees $[0.98,0.96,0.94,0.92,0.90,0.88,0.84,0.8,0.76,0.72,0.68,0.64,0.6]$ for each iteration respectively with the same prompt.
We use noise multiplier $\noisemultiplier=2$ and threshold $\threshold=2$. We use inception embedding for \cref{eq:distance_no_packing}.

\myparatightestn{Generated images.}
We use the same hyper-parameters to run \algnameshort{} on two datasets separately. This can also be regarded as running the conditional version of \algnameshort{} (\cref{alg:main_full}) on the whole dataset (with labels \catcookie{} or \catruby{}) together.
All generated images are in \cref{fig:cookie_gen,fig:ruby_gen}. While the two experiments use completely the same hyper-parameters, and the initial random images are very different from the cats (\cref{fig:ruby_gen_0}), our \algnameshort{} can guide the generated distribution in the right direction and the final generated images do capture the key color and characteristics of each of the cats. This demonstrates the effectiveness of \algnameshort{} with large foundation models such as Stable Diffusion.

We also observe that the diversity of generated images (e.g., poses, face directions) is limited compared to the real data. However, given the small number of samples and the tight privacy budget, this is an expected behavior: capturing more fine-grained features of each image would likely violate DP. 

\begin{figure}[ht]
    \centering
    \includegraphics[width=0.8\linewidth]{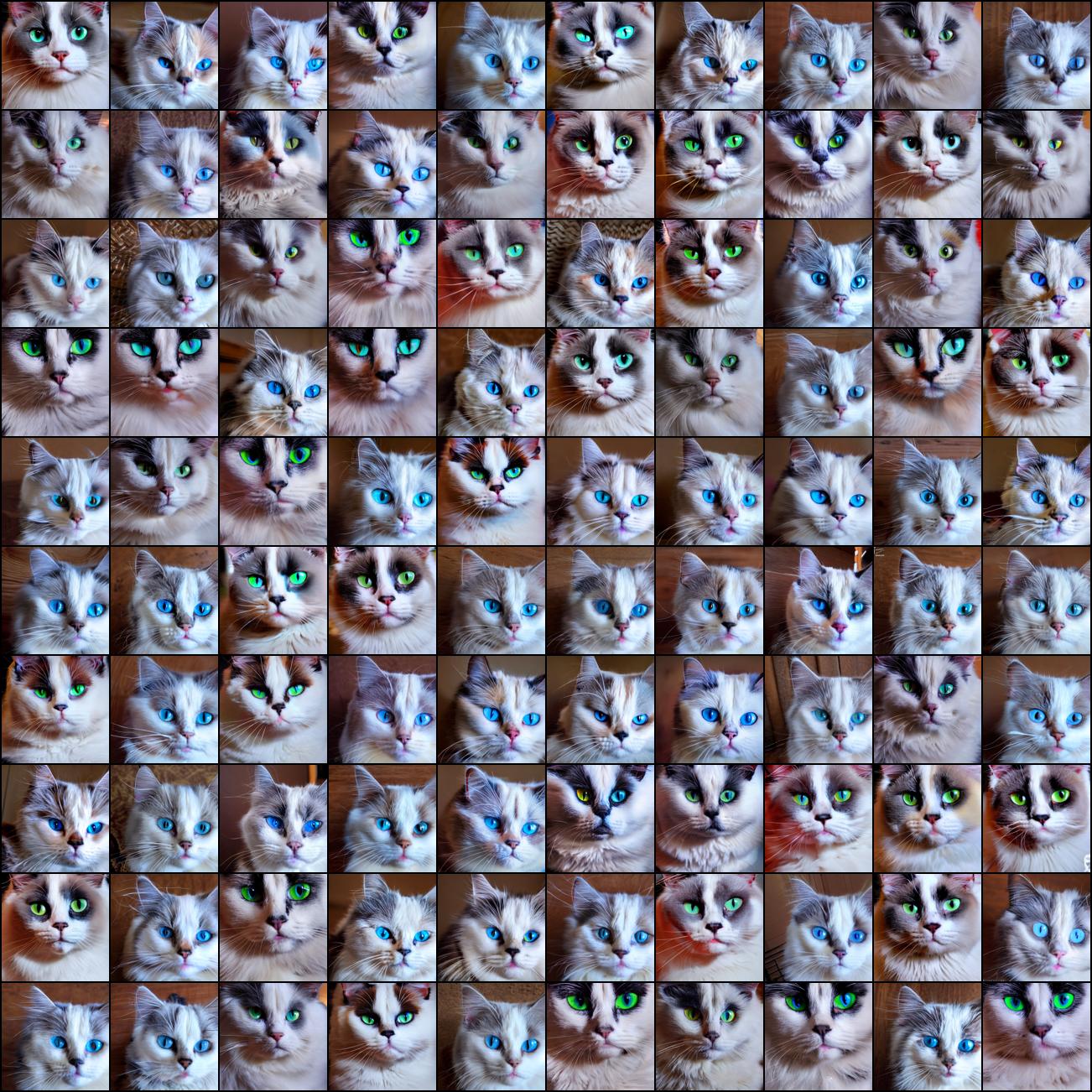}
    \caption{All generated images from \catcookie{} dataset with $(6.62,10^{-3})$-DP. The original resolution is 512x512; we resize them to 128x128 here for reducing the file size of the paper.}
    \label{fig:cookie_gen}
\end{figure}

\begin{figure}[ht]
    \centering
    \includegraphics[width=0.8\linewidth]{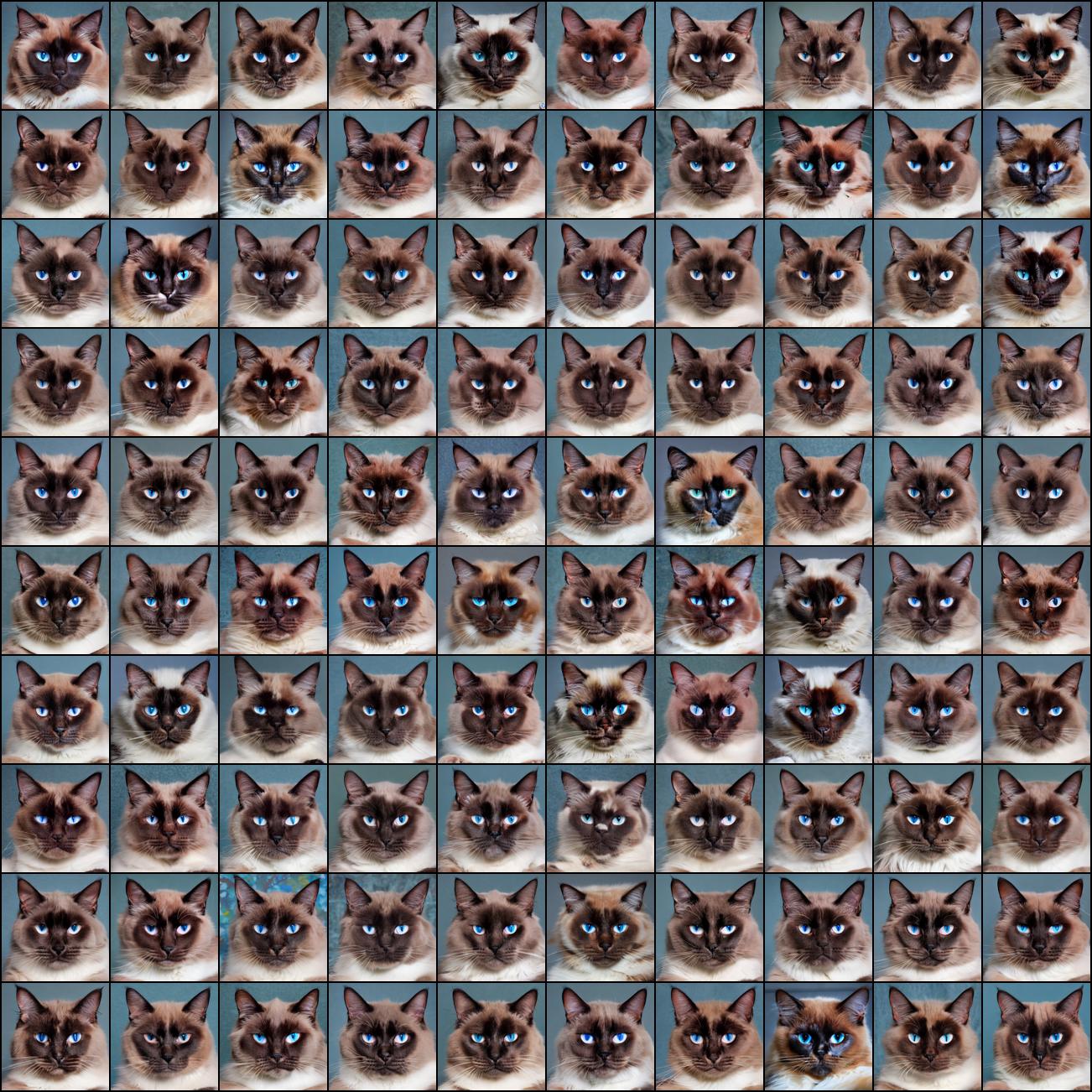}
    \caption{All generated images from \catruby{} dataset with $(6.62,10^{-3})$-DP. The original resolution is 512x512; we resize them to 128x128 here for reducing the file size of the paper.}
    \label{fig:ruby_gen}
\end{figure}

\begin{figure}[ht]
    \centering
    \includegraphics[width=0.8\linewidth]{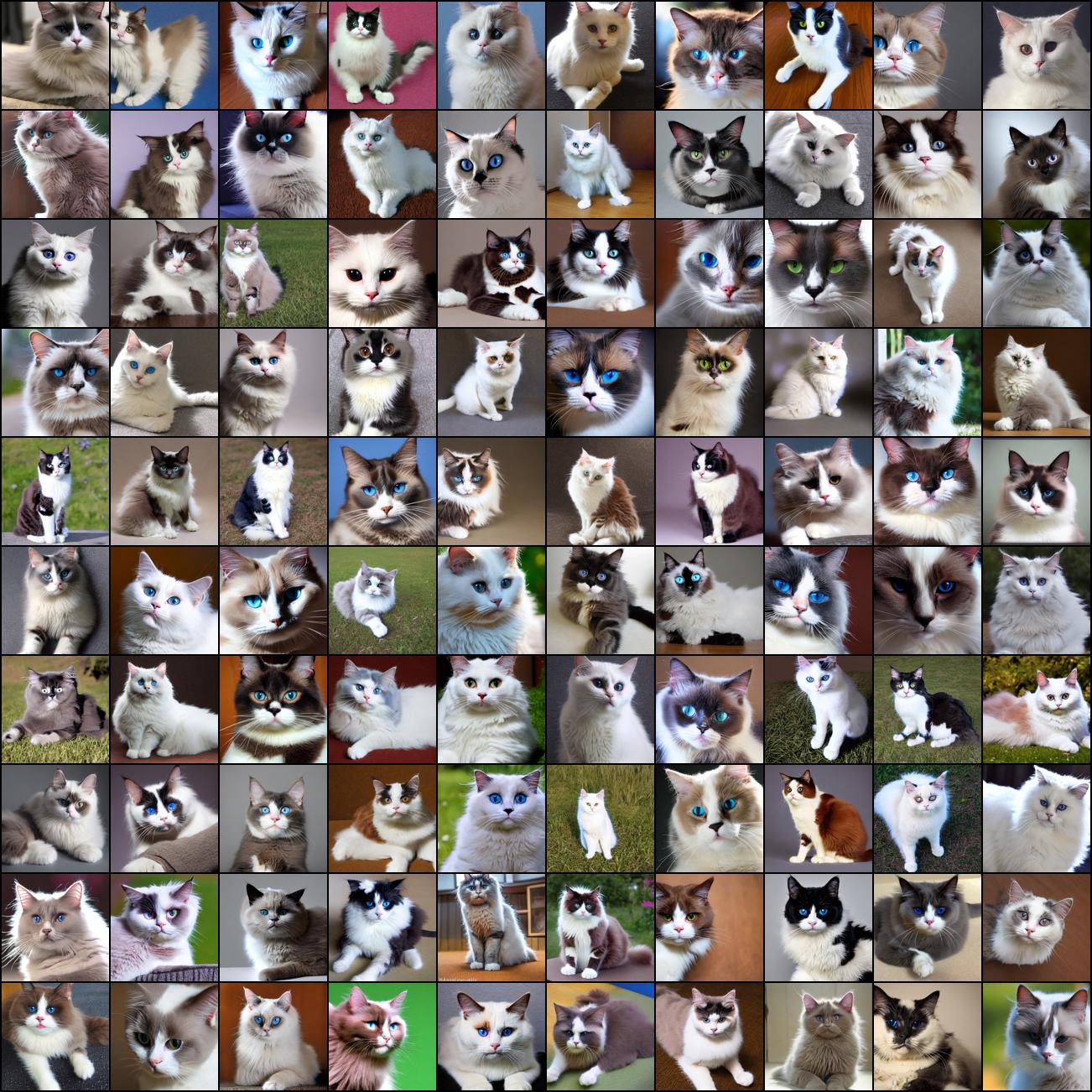}
    \caption{The initial random images from Stable Diffusion. The original resolution is 512x512; we resize them to 128x128 here for reducing the file size of the paper.}
    \label{fig:ruby_gen_0}
\end{figure}

\myparatightestn{Generated images with more diversity.} 
To make the generated images more diverse, we can utilize the approach in \cref{sec:algorithm_unlimited}, which passes the generated images through \samplevariationapiname{}. We have demonstrated in \cref{sec:exp_unlimited_number_of_samples} that this approach can generate more samples that are useful for downstream classification tasks. Here, we use it for a different purpose: enriching the diversity of generated samples.

Figure \cref{fig:cookie_gen_diversity,fig:ruby_gen_diversity} show the results. We can see that this simple approach is able to generate cats with a more diverse appearance. This is possible because the foundation model (Stable Diffusion) has good prior knowledge about cats learned from massive pre-training, and \algnameshort{} is able to utilize that effectively.

\begin{figure}[ht]
    \centering
    \includegraphics[width=0.8\linewidth]{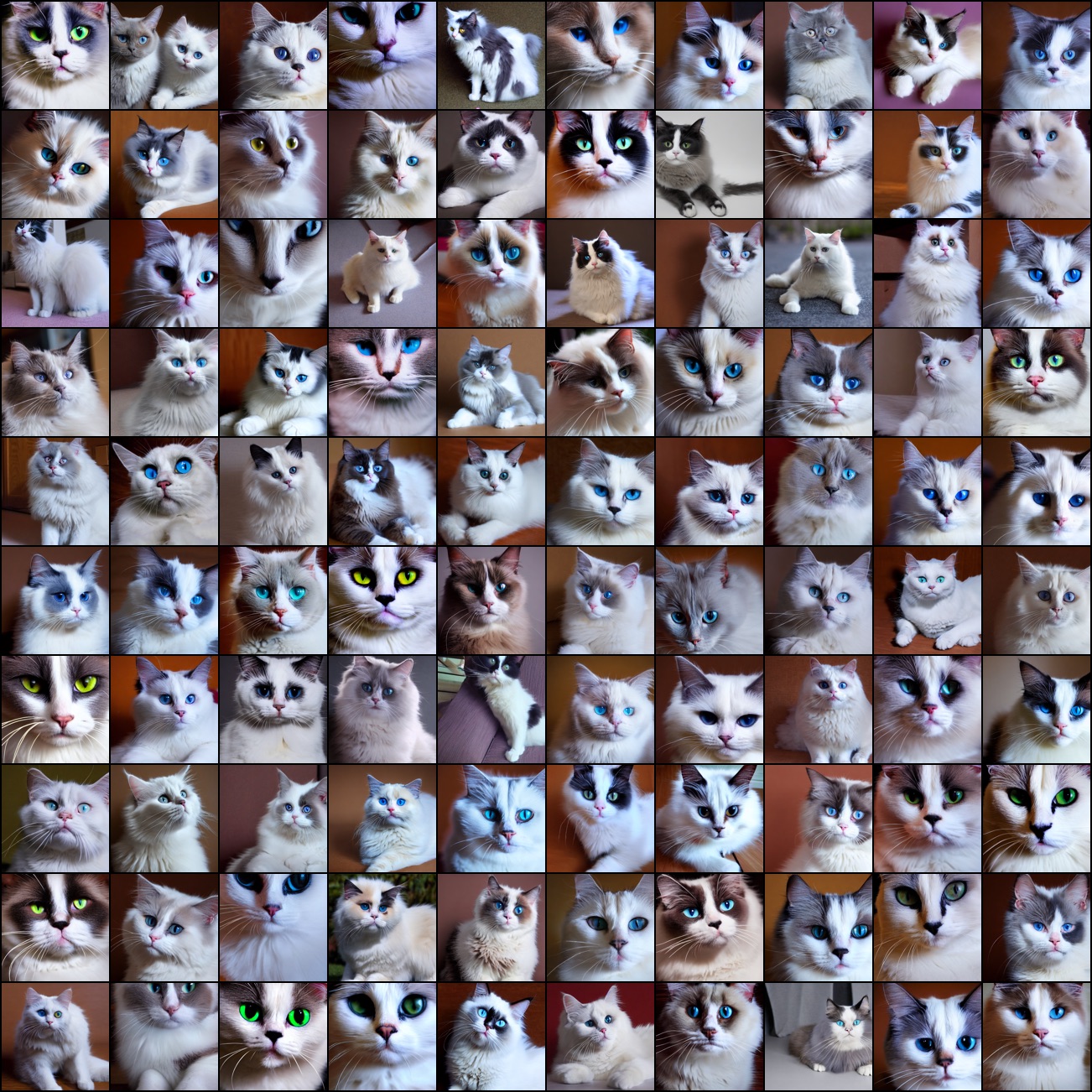}
    \caption{Generated images with enhanced diversity using the approach in \cref{sec:algorithm_unlimited} on \catcookie{}. The original resolution is 512x512; we resize them to 128x128 here for reducing the file size of the paper.}
    \label{fig:cookie_gen_diversity}
\end{figure}

\begin{figure}[ht]
    \centering
    \includegraphics[width=0.8\linewidth]{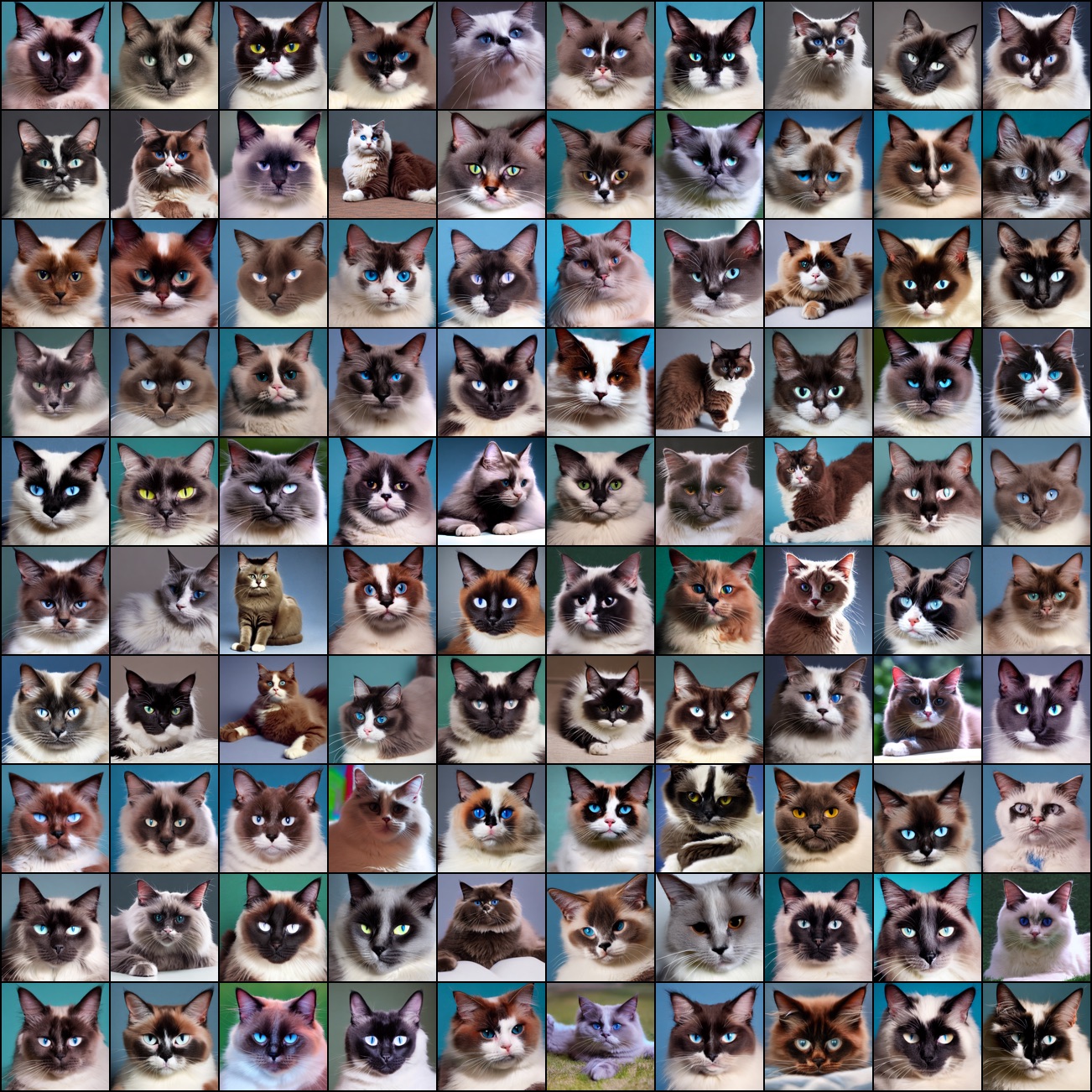}
    \caption{Generated images with enhanced diversity using the approach in \cref{sec:algorithm_unlimited} on \catruby{}. The original resolution is 512x512; we resize them to 128x128 here for reducing the file size of the paper.}
    \label{fig:ruby_gen_diversity}
\end{figure}

%% file: tex/app_cinic.tex
\section{CINIC Experiments}
\label{app:cinic}
Given that \algnameshort{} is based on sampling from pre-trained foundation models, a natural baseline is to conditionally sample from the classes closest to the private data. For example, in the experiments of \cref{sec:exp_cifar} (\cifar{} as private data, \imagenet{} as the public data), it would be interesting to consider baselines that conditionally sample from the same classes in \cifar{} (with privacy) from a foundation model trained on \imagenet{}. 
However, there could be many ways to implement such an algorithm, and the model and hyper-parameter choices could impact the result. 
To eliminate the influence of these arbitrariness choices, we conduct the following experiment which \emph{gives an estimate on the best FID such an approach could achieve}.

\myparatightestn{Experiment setup.} We use CINIC10 dataset \cite{darlow2018cinic}, which contains \cifar{} images and the filtered images from \imagenet{} that belong to \cifar{} categories. We compute the FID between (a) the subset of CINIC10 images that come from \imagenet{} (i.e., excluding \cifar{} images from CINIC10 dataset), and (b) \cifar{} dataset. This is to simulate the best case when the foundation model (1) learns the \imagenet{} dataset perfectly and (2) is able to draw only the samples from \cifar{} classes. 
In practice, the above two assumptions will certainly not hold; especially, achieving (2) would necessarily incur privacy costs as the knowledge of which samples belong to \cifar{} is assumed to be private. 
Therefore, this process gives us a lower bound of the FID (i.e., \emph{the best possible FID}) we can get by such a baseline that samples the same classes as \cifar{} from an \imagenet{} model using an arbitrary privacy cost.

\myparatightestn{Results.} The FID score from the above is 12.21. We can see from \cref{fig:cifar_fid_epsilon} that \algnameshort{} is able to achieve a smaller FID with $\epsilon$ as small as 0.38. This result suggests that PE does a non-trivial job: it can achieve a lower FID score than simply sampling the same classes as CIFAR10 from an ImageNet model.

%% file: tex/app_ablation.tex
\section{More Ablation Studies}
\label{app:ablation}

All ablation studies are conducted by taking the default parameters in unconditional \cifar{} experiments and modifying one hyperparameter at a time. The default noise multiplier $\noisemultiplier=5\cdot\sqrt{2}$ and the default threshold $\threshold=10$.

\myparatightestn{\Packing{} degree.} 
\cref{fig:cifar_ablation_packing} shows how the \packing{}  degree $\packingdegree$ %
(\cref{sec:alg}) 
impacts the results. We can see that higher \packing{} degrees monotonically improve the FID score. 
\revision{However, the marginal benefits diminish as the lookahead degree goes beyond 8, and a higher lookahead degree increases the required of API calls.}
Throughout all experiments, we used $\packingdegree=8$. This experiment suggests that better results can be obtained with a higher $\packingdegree$.
\revision{ In practice, we suggest users set the lookahead degree as the highest value within their computation or API budget.}
\begin{figure}[th]
    \centering
    \includegraphics[width=0.33\linewidth]{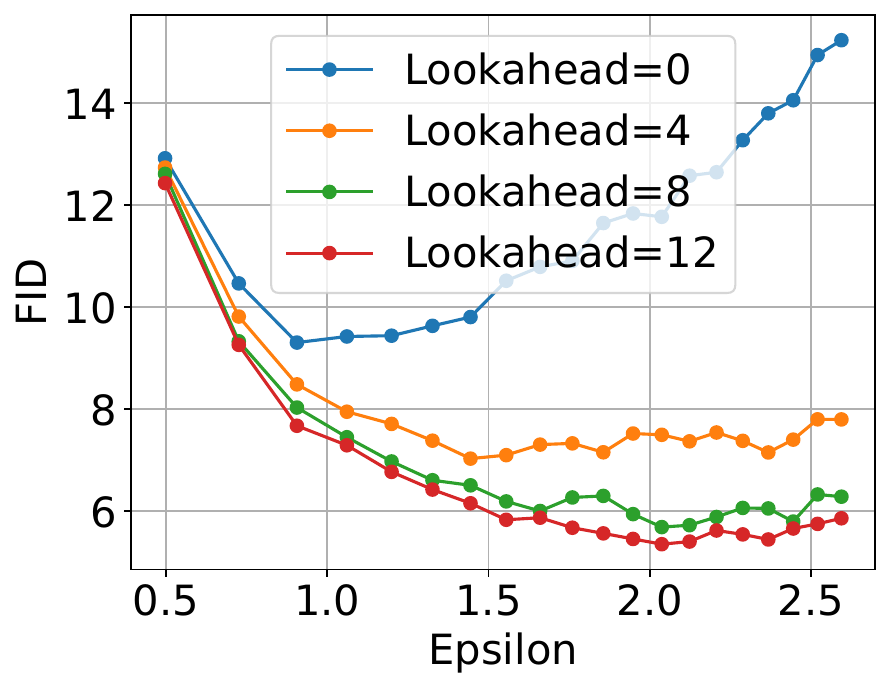}
    \caption{Ablation studies on the \packing{} degree $\packingdegree$ for \dpvotingname{}. \Packing{}=0 means \packing{} is not used (\cref{eq:distance_no_packing}). }
    \label{fig:cifar_ablation_packing}
\end{figure}

\myparatightestn{Population size.} \cref{fig:cifar_ablation_num_samples} shows how the number of generated samples $\numgensamples$ impacts the results in \emph{non-DP} setting. We can see that, as the number of samples increases,  FID score monotonically gets better. This is expected because with more generated samples, there are higher chances to get samples similar to the private data. However, we want to point out that in the DP case, it may not be true, as a large number of samples would flatten the \dpvotingname{} and thus decrease the signal noise ratio.

\begin{figure}[th]
    \centering
    \includegraphics[width=0.33\linewidth]{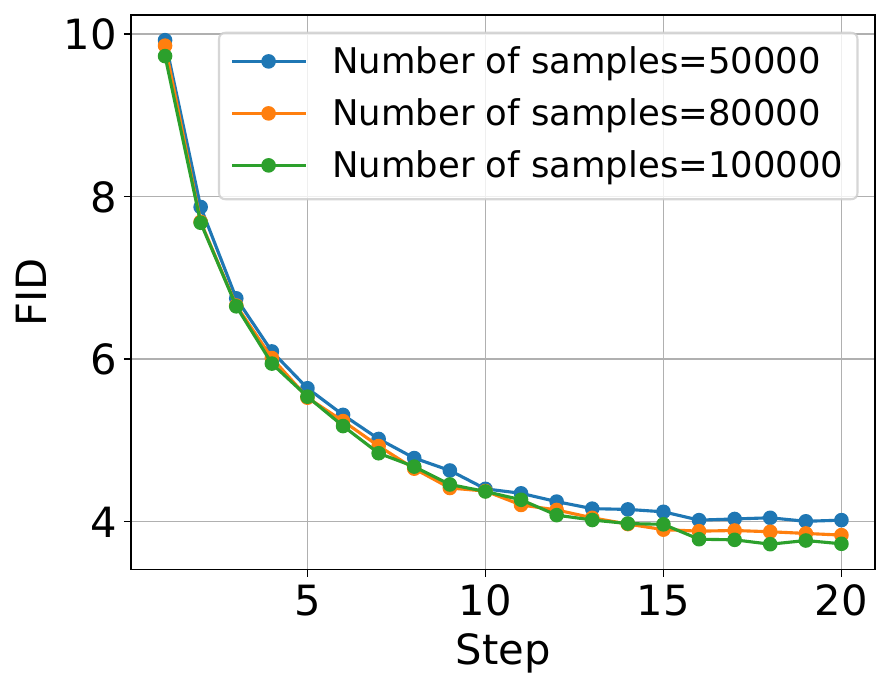}
    \caption{Ablation studies on the number of generated samples $\numgensamples$ in non-DP setting.}
    \label{fig:cifar_ablation_num_samples}
\end{figure}

\myparatightestn{Histogram threshold.}
\cref{fig:cifar_ablation_th} shows how the threshold $\threshold$ in \dpvotingname{} impacts the results. We can see that a large threshold results in a faster convergence speed at the beginning. This is because, in the early iterations, many samples are far away from the private data. A larger threshold can effectively remove those bad samples that have a non-zero histogram count due to the added DP noise.
However, at a later iteration, the distribution of generated samples is already close to the private data. A large threshold may potentially remove useful samples (e.g., the samples at low-density regions such as  classifier boundaries). This may hurt the generated data, as shown in the increasing FID scores at threshold=15. In this paper, we used a fixed threshold across all iterations. These results suggest that an adaptive threshold that gradually decreases might work better.
\begin{figure}[th]
    \centering
    \includegraphics[width=0.33\linewidth]{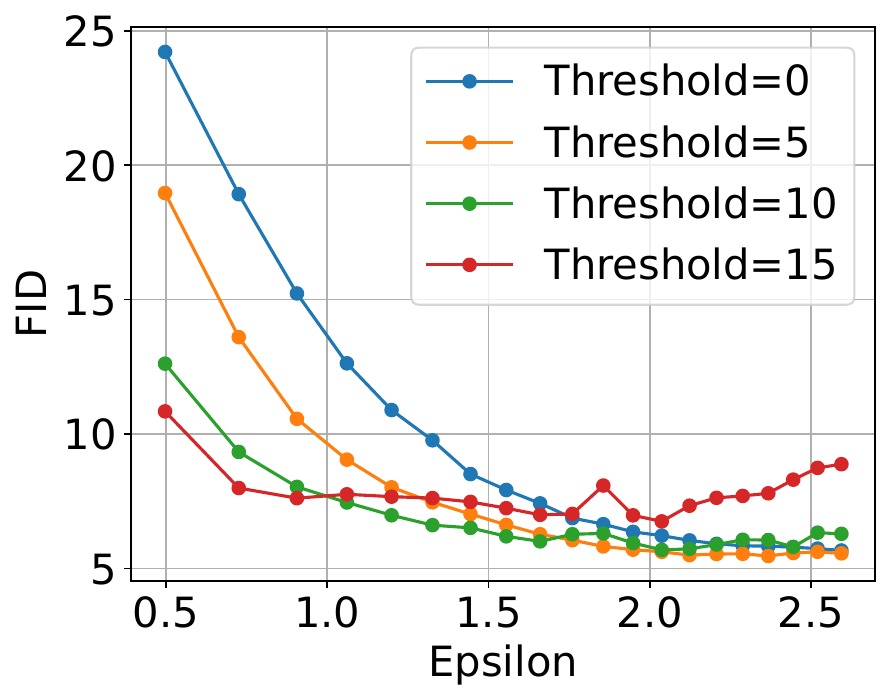}
    \caption{Ablation studies on the threshold $\threshold$ for \dpvotingname{}.}
    \label{fig:cifar_ablation_th}
\end{figure}

\myparatightestn{Embedding.}
\cref{fig:cifar_ablation_clip} compares the results with inception embedding or CLIP embedding  in \cref{eq:distance_no_packing}. The results show that both embedding networks work well, suggesting that \algnameshort{} is not too sensitive to the embedding network. Inception embedding works slightly better. One reason is that the inception network is trained on \imagenet{}, which is similar to a private dataset (\cifar{}). Therefore, it might be better at capturing the properties of images. 
Another possible reason is that FID score is calculated using inception embeddings, which might lead to some bias that favors inception embedding. %

\begin{figure}
    \centering
    \includegraphics[width=0.33\linewidth]{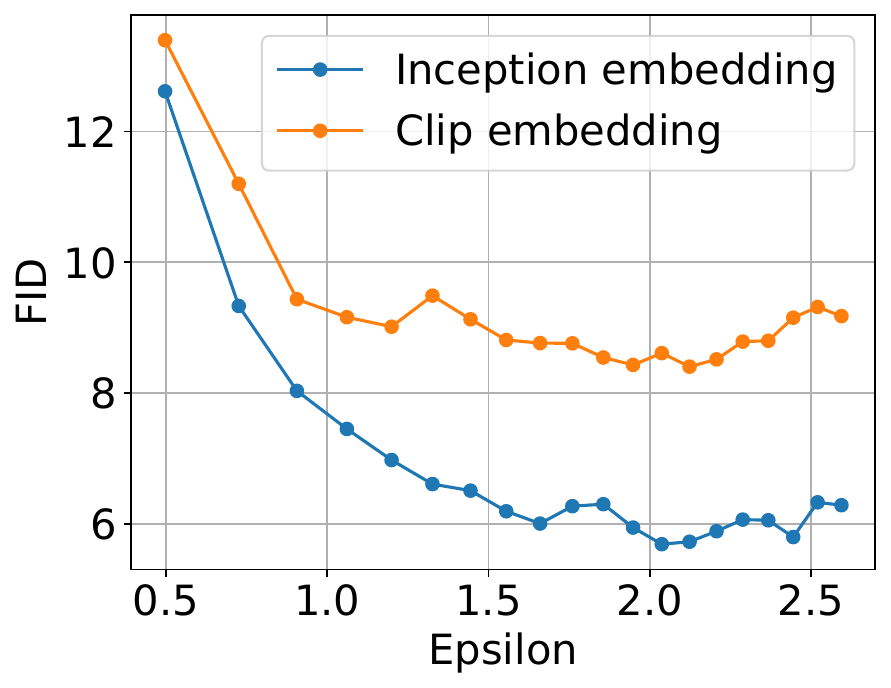}
    \caption{Ablation studies on the embedding network in \cref{eq:distance_no_packing} to use.}
    \label{fig:cifar_ablation_clip}
\end{figure}

\revision{\myparatightestn{The number of private samples.} In this experiment, we show how the number of private samples $\numprisamples$ impacts the performance of \algnameshort{}.
Specifically, on \cifar{}, we vary the number of samples $\numprisamples$ by sub-sampling. For each $\numprisamples\in\brc{50000, 20000, 10000, 5000}$:
\begin{packeditemize}
\item We fix $\numgensamples/\numprisamples=1$ when running \algnameshort{}.
\item After \algnameshort{} is done, we use the approach in \cref{sec:algorithm_unlimited} to augment the number of generated samples back to 50000, and compute the FID between the generated samples and \cifar{}.
\item All other hyper-parameters are set the same as the \cifar{} experiments in \cref{fig:cifar_gen_samples}.
\end{packeditemize}
The results are shown in \cref{tab:cifar_num_priv_sample}, which suggests that larger $\numprisamples$ helps the performance of \algnameshort{}, similar to the observation in DP-SGD \cite{anil2021large}.
We hypothesize that the reason is as follows. When we set  $\numgensamples/\numprisamples=1$, although the signal-noise ratio in the \dpvotingname{} remains constant, larger $N_{syn}$ does allow the generated samples to explore larger space, and therefore it is more likely to get a sample closer to the private data, which helps the convergence of \algnameshort{}. Our theorem in \cref{sec:theory} also shows that increasing the number of variations (which is controlled by $N_{syn}$) speeds up the convergence.
}

\begin{table}[ht]
    \centering
    \begin{tabular}{c|c}
    \toprule
        Number of samples & FID \\\midrule
       50000 (original \cifar{})  &  7.87\\
       20000 (2.5x smaller) & 10.47\\
       10000 (5x smaller) & 12.51\\
       5000 (10x smaller) & 18.60\\\bottomrule
    \end{tabular}
    \caption{\revision{FID vs. the size of the private dataset in \cifar{}.}}
    \label{tab:cifar_num_priv_sample}
\end{table}

%% file: tex/app_more_samples.tex
\section{Generating More Samples}
\label{app:more_samples}

We discussed two ways of generating more samples after \algnameshort{} is done: (1) taking the final generated samples from \algnameshort{} and passing them through \samplevariationapiname{} to get more samples (\cref{sec:algorithm_unlimited}), and (2) using a larger $\numgensamples$ when running \algnameshort{} (\cref{fig:cifar_ablation_num_samples}). While we see that both approaches are promising, they have a key difference: the first approach can be applied post hoc without the need to rerun \algnameshort{}, whereas the second approach requires rerunning \algnameshort{}. This difference would result in a very different number of API calls. More concretely, let’s say we want to generate $N$ more samples after \algnameshort{} is done. In the former approach, we simply take the final generated samples from \algnameshort{} and call \samplevariationapiname{} $N$ times. In contrast, the latter approach requires rerunning the entire \algnameshort{} process and would need to increase the number of API calls by $N\cdot (\packingdegree + 1)$, as we will need to generate more samples at every iteration.

%% file: tex/app_computation_resource.tex
\section{Computational Cost Evaluation}
\label{app:computation_resource}
\label{sec:exp_computation}

We compare the GPU hours of the SOTA DP fine-tuning method \cite{ghalebikesabi2023differentially} %
and \algnameshort{} 
for generating 50k samples in \cref{sec:exp_cifar}. %
Note that \algnameshort{} is designed to use existing pre-trained models and we do so in all experiments. 
In contrast,
DP fine-tuning methods usually require a careful selection of pre-training datasets and architectures (e.g.,  \cite{ghalebikesabi2023differentially} pre-trained their own diffusion models), which could be costly.
Even if we ignore this and only consider the computational cost after the pre-training, %
\emph{the total computational cost of \algnameshort{} is only 37\% of \cite{ghalebikesabi2023differentially}} while having better sample quality and downstream classification accuracy (\cref{sec:exp_cifar}). See \cref{fig:computation_resource} for a detailed breakdown of the computational cost.
Note that \algnameshort{}'s computational cost is mostly on the APIs. %

The key takeaway is that even if practitioners want to run the APIs locally (i.e., downloading the foundation models and running the APIs locally without using public API providers), there are still benefits of using \algnameshort{}: (1) The computational cost of \algnameshort{} can be smaller than training-based methods. (2) Implementing and deploying \algnameshort{} are easier because \algnameshort{} only requires blackbox APIs of the models and does not require code modifications inside the models.

\myparatightestn{Experimental details.}
\revision{In this experiment, we follow the setting in \cref{sec:exp_cifar}, where the user has a private dataset of size $\numprisamples=50000$ and a pre-trained model (hosted locally, running on GPUs), and wants to generate a total number of $\numgensamples=50000$ synthetic data samples. For the DP fine-tuning baseline, the procedure includes (1) DP fine-tuning the pre-trained model, and then (2) using the fine-tuned model to generate $50000$ synthetic samples. Therefore, we break the time into these two parts. For \algnameshort{}, the procedure includes running PE with $\numgensamples=50000$. PE includes the steps of \randomsampleapiname{}, \samplevariationapiname{}, nearest neighbor search, and feature extraction. Therefore, we break down the time of PE into these four parts.}

\revision{In the above evaluation, we did not consider the time of generating more samples beyond $\numgensamples=50000$, and we discuss its effect here. To generate more samples beyond $\numgensamples=50000$ samples, the DP fine-tuning method can directly use the fine-tuned model to generate more samples without more fine-tuning. \algnameshort{} can use the approach in \cref{sec:algorithm_unlimited}, where we pass the generated $\numgensamples$ samples through \samplevariationapiname{} to generate more samples (what we did in \cref{sec:exp_unlimited_number_of_samples}; see results thereof). Note that, for the same size of the model, the same number of generated samples, the same diffusion sampler, and the same total number of denoising steps, \emph{the running time of PE to generate more samples is smaller than the DP fine-tuning baseline.} The reason is as follows. Let’s say we use a diffusion sampler with 100 steps in total. In the DP fine-tuning baseline, generated samples must take all 100 steps (i.e., starting from Gaussian noise, passing the image through the diffusion model 100 times iteratively to get the final generated samples). For \algnameshort{}, \samplevariationapiname{} is implemented by SDEdit \cite{meng2021sdedit} (see \cref{app:cifar}), where we add noise to input images and let the diffusion model denoise them \emph{starting from the middle of the diffusion process} (e.g., adding Gaussian noise to the image, then passing the image through the diffusion model starting from 20th denoising step for all the rest 80 steps iteratively to get the final generated samples). In other words, each generated does not need to go through all 100 steps, but only a fraction of the steps (80 steps in the above example).}

To ensure a fair comparison, we estimate the runtime of both algorithms using 1 NVIDIA V100 32GB GPU. 

To evaluate the computational cost of \cite{ghalebikesabi2023differentially}, we take the open-source diffusion model implementation from \cite{dhariwal2021diffusion}\footnote{\url{https://github.com/openai/guided-diffusion}} and modify the hyper-parameters according to \cite{ghalebikesabi2023differentially}. We obtain a model with 79.9M parameters, slightly smaller than the one reported in \cite{ghalebikesabi2023differentially} (80.4M). This difference might be due to other implementation details that are not mentioned in \cite{ghalebikesabi2023differentially}. To implement DP training, we utilize Opacus library \cite{opacus}. To evaluate the fine-tuning cost, we use \texttt{torch.cuda.Event} instrumented before and after the core logic of forward and backward pass, ignoring other factors such as data loading time. We estimate the total runtime based on the mean runtime of 10 batches after 10 batches of warmup. We do not implement augmentation multiplicity with data and timestep \cite{ghalebikesabi2023differentially}; instead, we use multiplicity=1 (i.e., a vanilla diffusion model), and multiply the estimated runtime by 128, the multiplicity used in \cite{ghalebikesabi2023differentially}. To evaluate the generation cost, we use \texttt{torch.cuda.Event} instrumented before and after the core logic of sampling. We estimate the total runtime based on the mean runtime of 10 batches after 1 batch of warmup.

To evaluate the computational cost of our \algnameshort{}, we use a similar method: we use \texttt{torch.cuda.Event} instrumented before and after the core logic of each component of our algorithm that involves GPU computation. \randomsampleapiname{} and \samplevariationapiname{} are estimated based on the mean runtime of 10 batches after 1 batch of warmup. Feature extraction is estimated based on the mean runtime of 90 batches after 10 batch of warmup. The nearest neighbor search is estimated based on 1 run of the full search. We use faiss library\footnote{\url{https://github.com/facebookresearch/faiss}} for nearest neighbor search. Its implementation is very efficient so its computation time is negligible compared with the total time.

\begin{figure}[ht]
    \centering
    \includegraphics[width=0.5\linewidth]{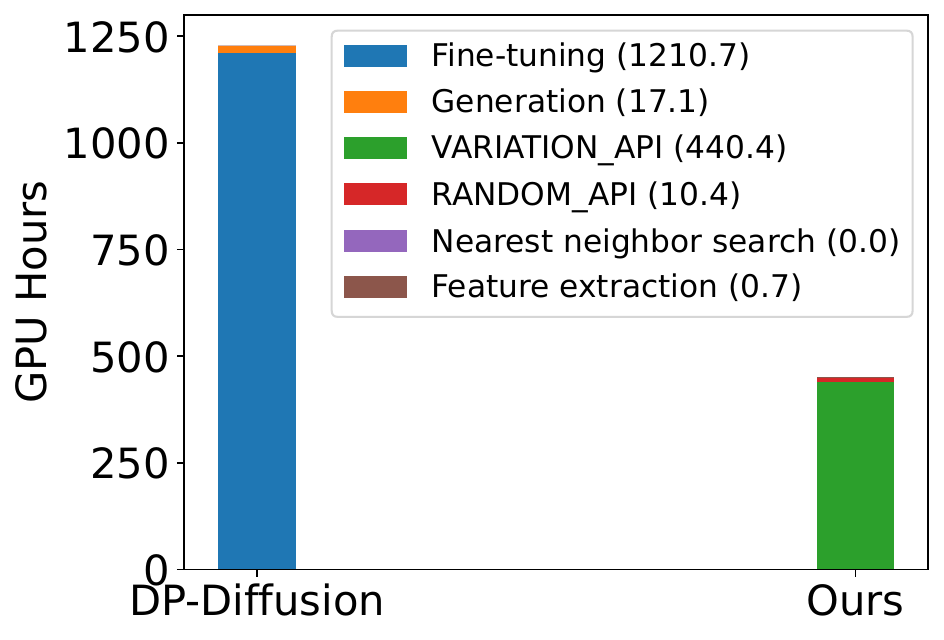}
    \caption{GPU hours (on 1 NVIDIA V100 32GB) required to obtain the samples for \cref{sec:exp_cifar} with DP-Diffusion \cite{ghalebikesabi2023differentially} and ours. The legend denotes the steps and the GPU hours; \cite{ghalebikesabi2023differentially} contains the fine-tuning and generation steps, whereas ours contains the other steps.}
    \label{fig:computation_resource}
\end{figure}

%% file: main_arxiv_V4.bbl
\begin{thebibliography}{79}
\providecommand{\natexlab}[1]{#1}
\providecommand{\url}[1]{\texttt{#1}}
\expandafter\ifx\csname urlstyle\endcsname\relax
  \providecommand{\doi}[1]{doi: #1}\else
  \providecommand{\doi}{doi: \begingroup \urlstyle{rm}\Url}\fi

\bibitem[Abadi et~al.(2016)Abadi, Chu, Goodfellow, McMahan, Mironov, Talwar,
  and Zhang]{abadi2016deep}
Martin Abadi, Andy Chu, Ian Goodfellow, H~Brendan McMahan, Ilya Mironov, Kunal
  Talwar, and Li~Zhang.
\newblock Deep learning with differential privacy.
\newblock In \emph{Proceedings of the 2016 ACM SIGSAC conference on computer
  and communications security}, pp.\  308--318, 2016.

\bibitem[Alemohammad et~al.(2023)Alemohammad, Casco-Rodriguez, Luzi, Humayun,
  Babaei, LeJeune, Siahkoohi, and Baraniuk]{alemohammad2023self}
Sina Alemohammad, Josue Casco-Rodriguez, Lorenzo Luzi, Ahmed~Imtiaz Humayun,
  Hossein Babaei, Daniel LeJeune, Ali Siahkoohi, and Richard~G Baraniuk.
\newblock Self-consuming generative models go mad.
\newblock \emph{arXiv preprint arXiv:2307.01850}, 2023.

\bibitem[Anil et~al.(2021)Anil, Ghazi, Gupta, Kumar, and
  Manurangsi]{anil2021large}
Rohan Anil, Badih Ghazi, Vineet Gupta, Ravi Kumar, and Pasin Manurangsi.
\newblock Large-scale differentially private bert.
\newblock \emph{arXiv preprint arXiv:2108.01624}, 2021.

\bibitem[Arjovsky et~al.(2017)Arjovsky, Chintala, and
  Bottou]{arjovsky2017wasserstein}
Martin Arjovsky, Soumith Chintala, and L{\'e}on Bottou.
\newblock Wasserstein generative adversarial networks.
\newblock In \emph{International conference on machine learning}, pp.\
  214--223. PMLR, 2017.

\bibitem[Balcan et~al.(2017)Balcan, Dick, Liang, Mou, and
  Zhang]{balcan2017differentially}
Maria-Florina Balcan, Travis Dick, Yingyu Liang, Wenlong Mou, and Hongyang
  Zhang.
\newblock Differentially private clustering in high-dimensional euclidean
  spaces.
\newblock In \emph{International Conference on Machine Learning}, pp.\
  322--331. PMLR, 2017.

\bibitem[Balle \& Wang(2018)Balle and Wang]{balle2018improving}
Borja Balle and Yu-Xiang Wang.
\newblock Improving the gaussian mechanism for differential privacy: Analytical
  calibration and optimal denoising.
\newblock In \emph{International Conference on Machine Learning}, pp.\
  394--403. PMLR, 2018.

\bibitem[Bandi et~al.(2018)Bandi, Geessink, Manson, Van~Dijk, Balkenhol,
  Hermsen, Bejnordi, Lee, Paeng, Zhong, et~al.]{bandi2018detection}
Peter Bandi, Oscar Geessink, Quirine Manson, Marcory Van~Dijk, Maschenka
  Balkenhol, Meyke Hermsen, Babak~Ehteshami Bejnordi, Byungjae Lee, Kyunghyun
  Paeng, Aoxiao Zhong, et~al.
\newblock From detection of individual metastases to classification of lymph
  node status at the patient level: the camelyon17 challenge.
\newblock \emph{IEEE transactions on medical imaging}, 38\penalty0
  (2):\penalty0 550--560, 2018.

\bibitem[Beaulieu-Jones et~al.(2019)Beaulieu-Jones, Wu, Williams, Lee,
  Bhavnani, Byrd, and Greene]{beaulieu2019privacy}
Brett~K Beaulieu-Jones, Zhiwei~Steven Wu, Chris Williams, Ran Lee, Sanjeev~P
  Bhavnani, James~Brian Byrd, and Casey~S Greene.
\newblock Privacy-preserving generative deep neural networks support clinical
  data sharing.
\newblock \emph{Circulation: Cardiovascular Quality and Outcomes}, 12\penalty0
  (7):\penalty0 e005122, 2019.

\bibitem[Brown et~al.(2020)Brown, Mann, Ryder, Subbiah, Kaplan, Dhariwal,
  Neelakantan, Shyam, Sastry, Askell, et~al.]{brown2020language}
Tom Brown, Benjamin Mann, Nick Ryder, Melanie Subbiah, Jared~D Kaplan, Prafulla
  Dhariwal, Arvind Neelakantan, Pranav Shyam, Girish Sastry, Amanda Askell,
  et~al.
\newblock Language models are few-shot learners.
\newblock \emph{Advances in neural information processing systems},
  33:\penalty0 1877--1901, 2020.

\bibitem[Cao et~al.(2021)Cao, Bie, Vahdat, Fidler, and Kreis]{cao2021don}
Tianshi Cao, Alex Bie, Arash Vahdat, Sanja Fidler, and Karsten Kreis.
\newblock Don’t generate me: Training differentially private generative
  models with sinkhorn divergence.
\newblock \emph{Advances in Neural Information Processing Systems},
  34:\penalty0 12480--12492, 2021.

\bibitem[Carlini et~al.(2019)Carlini, Liu, Erlingsson, Kos, and
  Song]{CarliniLEKS19}
Nicholas Carlini, Chang Liu, {\'U}lfar Erlingsson, Jernej Kos, and Dawn Song.
\newblock The secret sharer: Evaluating and testing unintended memorization in
  neural networks.
\newblock In \emph{28th USENIX Security Symposium}, 2019.

\bibitem[Carlini et~al.(2021{\natexlab{a}})Carlini, Chien, Nasr, Song, Terzis,
  and Tramer]{carlini2021membership}
Nicholas Carlini, Steve Chien, Milad Nasr, Shuang Song, Andreas Terzis, and
  Florian Tramer.
\newblock Membership inference attacks from first principles.
\newblock \emph{arXiv preprint arXiv:2112.03570}, 2021{\natexlab{a}}.

\bibitem[Carlini et~al.(2021{\natexlab{b}})Carlini, Tram\`er, Wallace,
  Jagielski, Herbert-Voss, Lee, Roberts, Brown, Song, Erlingsson, Oprea, and
  Raffel]{CarliniTWJHLRBSEOR21}
Nicholas Carlini, Florian Tram\`er, Eric Wallace, Matthew Jagielski, Ariel
  Herbert-Voss, Katherine Lee, Adam Roberts, Tom Brown, Dawn Song, Ulfar
  Erlingsson, Alina Oprea, and Colin Raffel.
\newblock Extracting training data from large language models.
\newblock In \emph{30th USENIX Security Symposium}, USENIX Security '21, pp.\
  2633--2650. USENIX Association, 2021{\natexlab{b}}.

\bibitem[Carlini et~al.(2023{\natexlab{a}})Carlini, Hayes, Nasr, Jagielski,
  Sehwag, Tramer, Balle, Ippolito, and Wallace]{carlini2023extracting}
Nicholas Carlini, Jamie Hayes, Milad Nasr, Matthew Jagielski, Vikash Sehwag,
  Florian Tramer, Borja Balle, Daphne Ippolito, and Eric Wallace.
\newblock Extracting training data from diffusion models.
\newblock \emph{arXiv preprint arXiv:2301.13188}, 2023{\natexlab{a}}.

\bibitem[Carlini et~al.(2023{\natexlab{b}})Carlini, Ippolito, Jagielski, Lee,
  Tramer, and Zhang]{carlini2023quantifying}
Nicholas Carlini, Daphne Ippolito, Matthew Jagielski, Katherine Lee, Florian
  Tramer, and Chiyuan Zhang.
\newblock Quantifying memorization across neural language models.
\newblock \emph{International Conference on Learning Representations},
  2023{\natexlab{b}}.

\bibitem[Chang \& Kamath(2023)Chang and Kamath]{chang2023differentially}
Alisa Chang and Pritish Kamath.
\newblock Differentially private clustering in google’s differential privacy
  library.
\newblock Google AI Blog, 2023.
\newblock URL
  \url{https://ai.googleblog.com/2021/10/practical-differentially-private.html}.
\newblock
  \url{https://ai.googleblog.com/2021/10/practical-differentially-private.html}.

\bibitem[Choquette-Choo et~al.(2021)Choquette-Choo, Tramer, Carlini, and
  Papernot]{ChoquetteChooTCP21}
Christopher~A Choquette-Choo, Florian Tramer, Nicholas Carlini, and Nicolas
  Papernot.
\newblock Label-only membership inference attacks.
\newblock In \emph{Proceedings of the 38th International Conference on Machine
  Learning}, ICML '21, pp.\  1964--1974. JMLR, Inc., 2021.

\bibitem[Commission()]{dataarticle}
European Commission.
\newblock Article 29 data protection working party.

\bibitem[Darlow et~al.(2018)Darlow, Crowley, Antoniou, and
  Storkey]{darlow2018cinic}
Luke~N Darlow, Elliot~J Crowley, Antreas Antoniou, and Amos~J Storkey.
\newblock Cinic-10 is not imagenet or cifar-10.
\newblock \emph{arXiv preprint arXiv:1810.03505}, 2018.

\bibitem[Davis(1987)]{davis1987genetic}
Lawrence Davis.
\newblock Genetic algorithms and simulated annealing.
\newblock 1987.

\bibitem[De et~al.(2022)De, Berrada, Hayes, Smith, and Balle]{de2022unlocking}
Soham De, Leonard Berrada, Jamie Hayes, Samuel~L Smith, and Borja Balle.
\newblock Unlocking high-accuracy differentially private image classification
  through scale.
\newblock \emph{arXiv preprint arXiv:2204.13650}, 2022.

\bibitem[Deng et~al.(2009)Deng, Dong, Socher, Li, Li, and
  Fei-Fei]{deng2009imagenet}
Jia Deng, Wei Dong, Richard Socher, Li-Jia Li, Kai Li, and Li~Fei-Fei.
\newblock Imagenet: A large-scale hierarchical image database.
\newblock In \emph{2009 IEEE conference on computer vision and pattern
  recognition}, pp.\  248--255. Ieee, 2009.

\bibitem[Dhariwal \& Nichol(2021)Dhariwal and Nichol]{dhariwal2021diffusion}
Prafulla Dhariwal and Alexander Nichol.
\newblock Diffusion models beat gans on image synthesis.
\newblock \emph{Advances in Neural Information Processing Systems},
  34:\penalty0 8780--8794, 2021.

\bibitem[Dockhorn et~al.(2022)Dockhorn, Cao, Vahdat, and
  Kreis]{dockhorn2022differentially}
Tim Dockhorn, Tianshi Cao, Arash Vahdat, and Karsten Kreis.
\newblock Differentially private diffusion models.
\newblock \emph{arXiv preprint arXiv:2210.09929}, 2022.

\bibitem[Dong et~al.(2022)Dong, Roth, and Su]{dong2022gaussian}
Jinshuo Dong, Aaron Roth, and Weijie~J Su.
\newblock Gaussian differential privacy.
\newblock \emph{Journal of the Royal Statistical Society Series B: Statistical
  Methodology}, 84\penalty0 (1):\penalty0 3--37, 2022.

\bibitem[Dwork et~al.(2006)Dwork, McSherry, Nissim, and
  Smith]{dwork2006calibrating}
Cynthia Dwork, Frank McSherry, Kobbi Nissim, and Adam Smith.
\newblock Calibrating noise to sensitivity in private data analysis.
\newblock In \emph{Theory of Cryptography: Third Theory of Cryptography
  Conference, TCC 2006, New York, NY, USA, March 4-7, 2006. Proceedings 3},
  pp.\  265--284. Springer, 2006.

\bibitem[Dwork et~al.(2014)Dwork, Roth, et~al.]{dwork2014algorithmic}
Cynthia Dwork, Aaron Roth, et~al.
\newblock The algorithmic foundations of differential privacy.
\newblock \emph{Foundations and Trends{\textregistered} in Theoretical Computer
  Science}, 9\penalty0 (3--4):\penalty0 211--407, 2014.

\bibitem[Fredrikson et~al.(2015)Fredrikson, Jha, and
  Ristenpart]{fredrikson2015model}
Matt Fredrikson, Somesh Jha, and Thomas Ristenpart.
\newblock Model inversion attacks that exploit confidence information and basic
  countermeasures.
\newblock In \emph{Proceedings of the 22nd ACM SIGSAC conference on computer
  and communications security}, pp.\  1322--1333, 2015.

\bibitem[Ganesh et~al.(2023)Ganesh, Haghifam, Nasr, Oh, Steinke, Thakkar,
  Thakurta, and Wang]{ganesh2023public}
Arun Ganesh, Mahdi Haghifam, Milad Nasr, Sewoong Oh, Thomas Steinke,
  Om~Thakkar, Abhradeep Thakurta, and Lun Wang.
\newblock Why is public pretraining necessary for private model training?
\newblock \emph{arXiv preprint arXiv:2302.09483}, 2023.

\bibitem[Ghalebikesabi et~al.(2023)Ghalebikesabi, Berrada, Gowal, Ktena,
  Stanforth, Hayes, De, Smith, Wiles, and
  Balle]{ghalebikesabi2023differentially}
Sahra Ghalebikesabi, Leonard Berrada, Sven Gowal, Ira Ktena, Robert Stanforth,
  Jamie Hayes, Soham De, Samuel~L Smith, Olivia Wiles, and Borja Balle.
\newblock Differentially private diffusion models generate useful synthetic
  images.
\newblock \emph{arXiv preprint arXiv:2302.13861}, 2023.

\bibitem[Ghazi et~al.(2020)Ghazi, Kumar, and
  Manurangsi]{ghazi2020differentially}
Badih Ghazi, Ravi Kumar, and Pasin Manurangsi.
\newblock Differentially private clustering: Tight approximation ratios.
\newblock \emph{Advances in Neural Information Processing Systems},
  33:\penalty0 4040--4054, 2020.

\bibitem[Ghazi et~al.(2022)Ghazi, He, Kohlhoff, Kumar, Manurangsi, Navalpakkam,
  and Valliappan]{ghazi2022differentially}
Badih Ghazi, Junfeng He, Kai Kohlhoff, Ravi Kumar, Pasin Manurangsi, Vidhya
  Navalpakkam, and Nachiappan Valliappan.
\newblock Differentially private heatmaps.
\newblock \emph{arXiv preprint arXiv:2211.13454}, 2022.

\bibitem[Goodfellow et~al.(2020)Goodfellow, Pouget-Abadie, Mirza, Xu,
  Warde-Farley, Ozair, Courville, and Bengio]{goodfellow2020generative}
Ian Goodfellow, Jean Pouget-Abadie, Mehdi Mirza, Bing Xu, David Warde-Farley,
  Sherjil Ozair, Aaron Courville, and Yoshua Bengio.
\newblock Generative adversarial networks.
\newblock \emph{Communications of the ACM}, 63\penalty0 (11):\penalty0
  139--144, 2020.

\bibitem[Gopi et~al.(2020)Gopi, Gulhane, Kulkarni, Shen, Shokouhi, and
  Yekhanin]{gopi2020differentially}
Sivakanth Gopi, Pankaj Gulhane, Janardhan Kulkarni, Judy~Hanwen Shen, Milad
  Shokouhi, and Sergey Yekhanin.
\newblock Differentially private set union.
\newblock In \emph{International Conference on Machine Learning}, pp.\
  3627--3636. PMLR, 2020.

\bibitem[Gopi et~al.(2021)Gopi, Lee, and Wutschitz]{gopi2021numerical}
Sivakanth Gopi, Yin~Tat Lee, and Lukas Wutschitz.
\newblock Numerical composition of differential privacy.
\newblock \emph{Advances in Neural Information Processing Systems},
  34:\penalty0 11631--11642, 2021.

\bibitem[Haim et~al.(2022)Haim, Vardi, Yehudai, Shamir, and
  Irani]{haim2022reconstructing}
Niv Haim, Gal Vardi, Gilad Yehudai, Ohad Shamir, and Michal Irani.
\newblock Reconstructing training data from trained neural networks.
\newblock \emph{arXiv preprint arXiv:2206.07758}, 2022.

\bibitem[Harder et~al.(2021)Harder, Adamczewski, and Park]{harder2021dp}
Frederik Harder, Kamil Adamczewski, and Mijung Park.
\newblock Dp-merf: Differentially private mean embeddings with randomfeatures
  for practical privacy-preserving data generation.
\newblock In \emph{International conference on artificial intelligence and
  statistics}, pp.\  1819--1827. PMLR, 2021.

\bibitem[Harder et~al.(2023)Harder, Jalali, Sutherland, and
  Park]{harder2022differentially}
Frederik Harder, Milad Jalali, Danica~J Sutherland, and Mijung Park.
\newblock Pre-trained perceptual features improve differentially private image
  generation.
\newblock \emph{Transactions on Machine Learning Research}, 2023.

\bibitem[He et~al.(2022)He, Li, Yu, Zhang, Kulkarni, Lee, Backurs, Yu, and
  Bian]{he2022exploring}
Jiyan He, Xuechen Li, Da~Yu, Huishuai Zhang, Janardhan Kulkarni, Yin~Tat Lee,
  Arturs Backurs, Nenghai Yu, and Jiang Bian.
\newblock Exploring the limits of differentially private deep learning with
  group-wise clipping.
\newblock \emph{arXiv preprint arXiv:2212.01539}, 2022.

\bibitem[Heusel et~al.(2017)Heusel, Ramsauer, Unterthiner, Nessler, and
  Hochreiter]{heusel2017gans}
Martin Heusel, Hubert Ramsauer, Thomas Unterthiner, Bernhard Nessler, and Sepp
  Hochreiter.
\newblock Gans trained by a two time-scale update rule converge to a local nash
  equilibrium.
\newblock \emph{Advances in neural information processing systems}, 30, 2017.

\bibitem[Ho et~al.(2020)Ho, Jain, and Abbeel]{ho2020denoising}
Jonathan Ho, Ajay Jain, and Pieter Abbeel.
\newblock Denoising diffusion probabilistic models.
\newblock \emph{Advances in Neural Information Processing Systems},
  33:\penalty0 6840--6851, 2020.

\bibitem[Hong et~al.(2022)Hong, Lyu, Zhou, and Spranger]{hong2022outsourcing}
Junyuan Hong, Lingjuan Lyu, Jiayu Zhou, and Michael Spranger.
\newblock Outsourcing training without uploading data via efficient
  collaborative open-source sampling.
\newblock \emph{Advances in neural information processing systems},
  35:\penalty0 20133--20146, 2022.

\bibitem[Hou et~al.(2023)Hou, Zhan, Shrivastava, Wang, Livshits, Fanti, and
  Lazar]{hou2023privately}
Charlie Hou, Hongyuan Zhan, Akshat Shrivastava, Sid Wang, Aleksandr Livshits,
  Giulia Fanti, and Daniel Lazar.
\newblock Privately customizing prefinetuning to better match user data in
  federated learning.
\newblock \emph{arXiv preprint arXiv:2302.09042}, 2023.

\bibitem[Hu et~al.(2023)Hu, Wu, Li, Long, Garrido, Ge, Ding, Forsyth, Li, and
  Song]{hu2023sok}
Yuzheng Hu, Fan Wu, Qinbin Li, Yunhui Long, Gonzalo~Munilla Garrido, Chang Ge,
  Bolin Ding, David Forsyth, Bo~Li, and Dawn Song.
\newblock Sok: Privacy-preserving data synthesis.
\newblock \emph{arXiv preprint arXiv:2307.02106}, 2023.

\bibitem[Issa et~al.(2019)Issa, Wagner, and Kamath]{issa2019operational}
Ibrahim Issa, Aaron~B Wagner, and Sudeep Kamath.
\newblock An operational approach to information leakage.
\newblock \emph{IEEE Transactions on Information Theory}, 66\penalty0
  (3):\penalty0 1625--1657, 2019.

\bibitem[Jordon et~al.(2019)Jordon, Yoon, and Van Der~Schaar]{jordon2019pate}
James Jordon, Jinsung Yoon, and Mihaela Van Der~Schaar.
\newblock {PATE-GAN}: Generating synthetic data with differential privacy
  guarantees.
\newblock In \emph{International conference on learning representations}, 2019.

\bibitem[Koh et~al.(2021)Koh, Sagawa, Marklund, Xie, Zhang, Balsubramani, Hu,
  Yasunaga, Phillips, Gao, et~al.]{koh2021wilds}
Pang~Wei Koh, Shiori Sagawa, Henrik Marklund, Sang~Michael Xie, Marvin Zhang,
  Akshay Balsubramani, Weihua Hu, Michihiro Yasunaga, Richard~Lanas Phillips,
  Irena Gao, et~al.
\newblock Wilds: A benchmark of in-the-wild distribution shifts.
\newblock In \emph{International Conference on Machine Learning}, pp.\
  5637--5664. PMLR, 2021.

\bibitem[Krizhevsky et~al.(2009)Krizhevsky, Hinton,
  et~al.]{krizhevsky2009learning}
Alex Krizhevsky, Geoffrey Hinton, et~al.
\newblock Learning multiple layers of features from tiny images.
\newblock 2009.

\bibitem[Li et~al.(2021)Li, Tramer, Liang, and Hashimoto]{li2021large}
Xuechen Li, Florian Tramer, Percy Liang, and Tatsunori Hashimoto.
\newblock Large language models can be strong differentially private learners.
\newblock \emph{arXiv preprint arXiv:2110.05679}, 2021.

\bibitem[Li et~al.(2022)Li, Liu, Hashimoto, Inan, Kulkarni, Lee, and
  Guha~Thakurta]{li2022does}
Xuechen Li, Daogao Liu, Tatsunori~B Hashimoto, Huseyin~A Inan, Janardhan
  Kulkarni, Yin-Tat Lee, and Abhradeep Guha~Thakurta.
\newblock When does differentially private learning not suffer in high
  dimensions?
\newblock \emph{Advances in Neural Information Processing Systems},
  35:\penalty0 28616--28630, 2022.

\bibitem[Lin(2022)]{lin2022data}
Zinan Lin.
\newblock \emph{Data Sharing with Generative Adversarial Networks: From Theory
  to Practice}.
\newblock PhD thesis, Carnegie Mellon University, 2022.

\bibitem[Lin et~al.(2020)Lin, Jain, Wang, Fanti, and Sekar]{lin2020using}
Zinan Lin, Alankar Jain, Chen Wang, Giulia Fanti, and Vyas Sekar.
\newblock Using gans for sharing networked time series data: Challenges,
  initial promise, and open questions.
\newblock In \emph{Proceedings of the ACM Internet Measurement Conference},
  pp.\  464--483, 2020.

\bibitem[Lin et~al.(2021)Lin, Sekar, and Fanti]{lin2021privacy}
Zinan Lin, Vyas Sekar, and Giulia Fanti.
\newblock On the privacy properties of gan-generated samples.
\newblock In \emph{International Conference on Artificial Intelligence and
  Statistics}, pp.\  1522--1530. PMLR, 2021.

\bibitem[Lin et~al.(2022)Lin, Wang, Sekar, and Fanti]{lin2022distributional}
Zinan Lin, Shuaiqi Wang, Vyas Sekar, and Giulia Fanti.
\newblock Distributional privacy for data sharing.
\newblock In \emph{NeurIPS 2022 Workshop on Synthetic Data for Empowering ML
  Research}, 2022.

\bibitem[Lin et~al.(2023{\natexlab{a}})Lin, Gopi, Kulkarni, Nori, and
  Yekhanin]{lin2023differentially}
Zinan Lin, Sivakanth Gopi, Janardhan Kulkarni, Harsha Nori, and Sergey
  Yekhanin.
\newblock Differentially private synthetic data via foundation model {API}s 1:
  Images.
\newblock In \emph{NeurIPS 2023 Workshop on Synthetic Data Generation with
  Generative AI}, 2023{\natexlab{a}}.
\newblock URL \url{https://openreview.net/forum?id=7GbfIEvoS8}.

\bibitem[Lin et~al.(2023{\natexlab{b}})Lin, Wang, Sekar, and
  Fanti]{lin2023summary}
Zinan Lin, Shuaiqi Wang, Vyas Sekar, and Giulia Fanti.
\newblock Summary statistic privacy in data sharing.
\newblock \emph{arXiv preprint arXiv:2303.02014}, 2023{\natexlab{b}}.

\bibitem[Meng et~al.(2021)Meng, Song, Song, Wu, Zhu, and Ermon]{meng2021sdedit}
Chenlin Meng, Yang Song, Jiaming Song, Jiajun Wu, Jun-Yan Zhu, and Stefano
  Ermon.
\newblock Sdedit: Image synthesis and editing with stochastic differential
  equations.
\newblock \emph{arXiv preprint arXiv:2108.01073}, 2021.

\bibitem[Mironov(2017)]{mironov2017renyi}
Ilya Mironov.
\newblock R{\'e}nyi differential privacy.
\newblock In \emph{2017 IEEE 30th computer security foundations symposium
  (CSF)}, pp.\  263--275. IEEE, 2017.

\bibitem[Nichol \& Dhariwal(2021)Nichol and Dhariwal]{nichol2021improved}
Alexander~Quinn Nichol and Prafulla Dhariwal.
\newblock Improved denoising diffusion probabilistic models.
\newblock In \emph{International Conference on Machine Learning}, pp.\
  8162--8171. PMLR, 2021.

\bibitem[OpenAI(2023)]{openai2023gpt4}
OpenAI.
\newblock Gpt-4 technical report, 2023.

\bibitem[Radford et~al.(2021)Radford, Kim, Hallacy, Ramesh, Goh, Agarwal,
  Sastry, Askell, Mishkin, Clark, et~al.]{radford2021learning}
Alec Radford, Jong~Wook Kim, Chris Hallacy, Aditya Ramesh, Gabriel Goh,
  Sandhini Agarwal, Girish Sastry, Amanda Askell, Pamela Mishkin, Jack Clark,
  et~al.
\newblock Learning transferable visual models from natural language
  supervision.
\newblock In \emph{International conference on machine learning}, pp.\
  8748--8763. PMLR, 2021.

\bibitem[Ramesh et~al.(2022)Ramesh, Dhariwal, Nichol, Chu, and
  Chen]{ramesh2022hierarchical}
Aditya Ramesh, Prafulla Dhariwal, Alex Nichol, Casey Chu, and Mark Chen.
\newblock Hierarchical text-conditional image generation with clip latents.
\newblock \emph{arXiv preprint arXiv:2204.06125}, 2022.

\bibitem[Rombach et~al.(2022)Rombach, Blattmann, Lorenz, Esser, and
  Ommer]{rombach2022high}
Robin Rombach, Andreas Blattmann, Dominik Lorenz, Patrick Esser, and Bj{\"o}rn
  Ommer.
\newblock High-resolution image synthesis with latent diffusion models.
\newblock In \emph{Proceedings of the IEEE/CVF Conference on Computer Vision
  and Pattern Recognition}, pp.\  10684--10695, 2022.

\bibitem[Savage(2023)]{Savage2023}
Neil Savage.
\newblock Synthetic data could be better than real data.
\newblock \emph{Nature}, April 2023.
\newblock \doi{10.1038/d41586-023-01445-8}.
\newblock URL \url{https://doi.org/10.1038/d41586-023-01445-8}.

\bibitem[Sohl-Dickstein et~al.(2015)Sohl-Dickstein, Weiss, Maheswaranathan, and
  Ganguli]{sohl2015deep}
Jascha Sohl-Dickstein, Eric Weiss, Niru Maheswaranathan, and Surya Ganguli.
\newblock Deep unsupervised learning using nonequilibrium thermodynamics.
\newblock In \emph{International Conference on Machine Learning}, pp.\
  2256--2265. PMLR, 2015.

\bibitem[Song et~al.(2020)Song, Meng, and Ermon]{song2020denoising}
Jiaming Song, Chenlin Meng, and Stefano Ermon.
\newblock Denoising diffusion implicit models.
\newblock \emph{arXiv preprint arXiv:2010.02502}, 2020.

\bibitem[Song \& Ermon(2019)Song and Ermon]{song2019generative}
Yang Song and Stefano Ermon.
\newblock Generative modeling by estimating gradients of the data distribution.
\newblock \emph{Advances in neural information processing systems}, 32, 2019.

\bibitem[Su et~al.(2016)Su, Cao, Li, Bertino, and Jin]{su2016differentially}
Dong Su, Jianneng Cao, Ninghui Li, Elisa Bertino, and Hongxia Jin.
\newblock Differentially private k-means clustering.
\newblock In \emph{Proceedings of the sixth ACM conference on data and
  application security and privacy}, pp.\  26--37, 2016.

\bibitem[Szegedy et~al.(2016)Szegedy, Vanhoucke, Ioffe, Shlens, and
  Wojna]{szegedy2016rethinking}
Christian Szegedy, Vincent Vanhoucke, Sergey Ioffe, Jon Shlens, and Zbigniew
  Wojna.
\newblock Rethinking the inception architecture for computer vision.
\newblock In \emph{Proceedings of the IEEE conference on computer vision and
  pattern recognition}, pp.\  2818--2826, 2016.

\bibitem[Tang et~al.(2023)Tang, Shin, Inan, Manoel, Mireshghallah, Lin, Gopi,
  Kulkarni, and Sim]{tang2023privacy}
Xinyu Tang, Richard Shin, Huseyin~A Inan, Andre Manoel, Fatemehsadat
  Mireshghallah, Zinan Lin, Sivakanth Gopi, Janardhan Kulkarni, and Robert Sim.
\newblock Privacy-preserving in-context learning with differentially private
  few-shot generation.
\newblock \emph{arXiv preprint arXiv:2309.11765}, 2023.

\bibitem[Tram{\`e}r et~al.(2022)Tram{\`e}r, Shokri, Joaquin, Le, Jagielski,
  Hong, and Carlini]{TramerSJLJHC22}
Florian Tram{\`e}r, Reza Shokri, Ayrton~San Joaquin, Hoang Le, Matthew
  Jagielski, Sanghyun Hong, and Nicholas Carlini.
\newblock Truth serum: Poisoning machine learning models to reveal their
  secrets.
\newblock \emph{arXiv preprint arXiv:2204.00032}, 2022.

\bibitem[Vinaroz et~al.(2022)Vinaroz, Charusaie, Harder, Adamczewski, and
  Park]{vinaroz2022hermite}
Margarita Vinaroz, Mohammad-Amin Charusaie, Frederik Harder, Kamil Adamczewski,
  and Mi~Jung Park.
\newblock Hermite polynomial features for private data generation.
\newblock In \emph{International Conference on Machine Learning}, pp.\
  22300--22324. PMLR, 2022.

\bibitem[Wang et~al.(2023)Wang, Chen, Pei, Xie, Kang, Zhang, Xu, Xiong, Dutta,
  Schaeffer, et~al.]{wang2023decodingtrust}
Boxin Wang, Weixin Chen, Hengzhi Pei, Chulin Xie, Mintong Kang, Chenhui Zhang,
  Chejian Xu, Zidi Xiong, Ritik Dutta, Rylan Schaeffer, et~al.
\newblock Decodingtrust: A comprehensive assessment of trustworthiness in gpt
  models.
\newblock \emph{arXiv preprint arXiv:2306.11698}, 2023.

\bibitem[Yin et~al.(2022)Yin, Lin, Jin, Fanti, and Sekar]{yin2022practical}
Yucheng Yin, Zinan Lin, Minhao Jin, Giulia Fanti, and Vyas Sekar.
\newblock Practical gan-based synthetic ip header trace generation using
  netshare.
\newblock In \emph{Proceedings of the ACM SIGCOMM 2022 Conference}, pp.\
  458--472, 2022.

\bibitem[Yousefpour et~al.(2021)Yousefpour, Shilov, Sablayrolles, Testuggine,
  Prasad, Malek, Nguyen, Ghosh, Bharadwaj, Zhao, Cormode, and Mironov]{opacus}
Ashkan Yousefpour, Igor Shilov, Alexandre Sablayrolles, Davide Testuggine,
  Karthik Prasad, Mani Malek, John Nguyen, Sayan Ghosh, Akash Bharadwaj,
  Jessica Zhao, Graham Cormode, and Ilya Mironov.
\newblock Opacus: {U}ser-friendly differential privacy library in {PyTorch}.
\newblock \emph{arXiv preprint arXiv:2109.12298}, 2021.

\bibitem[Yu et~al.(2021)Yu, Naik, Backurs, Gopi, Inan, Kamath, Kulkarni, Lee,
  Manoel, Wutschitz, et~al.]{yu2021differentially}
Da~Yu, Saurabh Naik, Arturs Backurs, Sivakanth Gopi, Huseyin~A Inan, Gautam
  Kamath, Janardhan Kulkarni, Yin~Tat Lee, Andre Manoel, Lukas Wutschitz,
  et~al.
\newblock Differentially private fine-tuning of language models.
\newblock \emph{arXiv preprint arXiv:2110.06500}, 2021.

\bibitem[Yu et~al.(2023)Yu, Gopi, Kulkarni, Lin, Naik, Religa, Yin, and
  Zhang]{yu2023selective}
Da~Yu, Sivakanth Gopi, Janardhan Kulkarni, Zinan Lin, Saurabh Naik,
  Tomasz~Lukasz Religa, Jian Yin, and Huishuai Zhang.
\newblock Selective pre-training for private fine-tuning.
\newblock \emph{arXiv preprint arXiv:2305.13865}, 2023.

\bibitem[Yue et~al.(2022)Yue, Inan, Li, Kumar, McAnallen, Sun, Levitan, and
  Sim]{yue2022synthetic}
Xiang Yue, Huseyin~A Inan, Xuechen Li, Girish Kumar, Julia McAnallen, Huan Sun,
  David Levitan, and Robert Sim.
\newblock Synthetic text generation with differential privacy: A simple and
  practical recipe.
\newblock \emph{arXiv preprint arXiv:2210.14348}, 2022.

\bibitem[Zagoruyko \& Komodakis(2016)Zagoruyko and
  Komodakis]{zagoruyko2016wide}
Sergey Zagoruyko and Nikos Komodakis.
\newblock Wide residual networks.
\newblock \emph{arXiv preprint arXiv:1605.07146}, 2016.

\end{thebibliography}
